\documentclass{article}

\usepackage{microtype}
\usepackage{graphicx}
\usepackage{subcaption}
\usepackage{booktabs}

\usepackage[accepted]{icml2026}
\usepackage{hyperref}
\usepackage{amsmath,amsfonts,amssymb}
\usepackage{mathtools}
\usepackage{amsthm}
\usepackage{bm}

\usepackage[capitalize,noabbrev]{cleveref}
\usepackage{algorithm}
\usepackage{algorithmic}
\usepackage{wrapfig}
\usepackage{tikz}
\usetikzlibrary{positioning,calc}
\usepackage{multirow}
\usepackage{xcolor}
\usepackage{adjustbox}
\usepackage{makecell} 
\usepackage{enumitem}
\usepackage{pifont}

\usepackage{amsmath,amsfonts,bm}

\def\eqref#1{equation~\ref{#1}}
\def\Eqref#1{Equation~\ref{#1}}

\def\1{\bm{1}}

\DeclareMathAlphabet{\mathsfit}{\encodingdefault}{\sfdefault}{m}{sl}
\SetMathAlphabet{\mathsfit}{bold}{\encodingdefault}{\sfdefault}{bx}{n}

\theoremstyle{plain}
\newtheorem{theorem}{Theorem}[section]
\newtheorem{proposition}[theorem]{Proposition}

\theoremstyle{definition}

\theoremstyle{remark}

\newcommand{\method}{scCBGM}
\newcommand{\cellflow}{Vanilla-FM}
\def\Eqref#1{Equation~\ref{#1}}
\newcommand{\tabwidth}{0.9\textwidth} 

\definecolor{mypink}{rgb}{0.858, 0.188, 0.478}
\definecolor{mygreen}{RGB}{0, 128, 0}
\definecolor{myblue}{RGB}{0, 0, 255}
\definecolor{myred}{RGB}{255, 0, 0}
\definecolor{mybrown}{RGB}{150, 75, 0}
\definecolor{myorange}{RGB}{255,69,0}
\definecolor{mypurple}{RGB}{128,0,128}
\definecolor{cbpink}{RGB}{220,38,127}
\definecolor{cborange}{RGB}{254,97,0}
\definecolor{cbpurple}{RGB}{120,94,240}

\newcommand{\bestres}[1]{\textcolor{cbpurple}{#1}}
\newcommand{\secondbest}[1]{\textbf{\textcolor{cborange}{#1}}}
\newcommand{\thirdbest}[1]{\textbf{#1}}

\newcommand{\cmark}{\textcolor{mygreen}{\ding{51}}}%
\newcommand{\xmark}{\textcolor{myred}{\ding{55}}}

\icmltitlerunning{scCBGM: Interpretable Single-Cell Counterfactual Editing}

\begin{document}

\twocolumn[

\icmltitle{Single-Cell Concept Bottleneck Generative Models for Interpretable and Controllable Cellular Editing}
\icmlsetsymbol{equal}{*}

\begin{icmlauthorlist}
\icmlauthor{Alma Andersson}{equal,gne}
\icmlauthor{Aya Abdelsalam Ismail}{equal,gl}
\icmlauthor{Edward De Brouwer}{equal,gne}
\icmlauthor{Doron Haviv}{equal,gne}
\icmlauthor{Tommaso Biancalani}{gne}
\icmlauthor{Kyunghyun Cho}{gne,nyucs,nyucds}
\icmlauthor{Gabriele Scalia}{gne}
\icmlauthor{Aicha BenTaieb}{gne}
\icmlauthor{Hector Corrada Bravo}{gne}
\end{icmlauthorlist}

\icmlaffiliation{gne}{Genentech}
\icmlaffiliation{gl}{Guide Labs}
\icmlaffiliation{nyucs}{Department of Computer Science, New York University}
\icmlaffiliation{nyucds}{Center for Data Science, New York University}


\icmlcorrespondingauthor{Alma Andersson}{andersson.alma@gene.com. Code available at \url{https://github.com/almaan/scCBGM}}

\icmlkeywords{Machine Learning, Single-Cell Genomics, Counterfactual Inference, Concept Bottleneck Models, Generative Models, ICML}

\vskip 0.4in
]

\printAffiliationsAndNotice{\icmlEqualContribution}

\begin{abstract}
Understanding cellular phenotypes and how they respond to perturbations is critical for disease biology and therapeutic design. Single-cell RNA sequencing enables characterization at cellular resolution, yet the combinatorial space of conditions makes exhaustive experimental mapping infeasible. We introduce single-cell Concept Bottleneck Generative Models (\method), a framework for interpretable and precise counterfactual editing of individual cells. \method~adapts concept bottleneck architectures for single-cell data through decoder skip connections and a cross-covariance penalty that promotes disentanglement without dimensional constraints. We extend the framework to flow matching models, enabling concept-guided editing in both encoding-decoding and generation regimes. To enable rigorous evaluation, we develop a synthetic benchmark with ground-truth counterfactuals. Across multiple real datasets, \method~demonstrates superior performance in combinatorial generalization and counterfactual prediction, supported by cell-level validation on synthetic data and population-level benchmarks on real datasets.

\end{abstract}

\section{Introduction}
\label{introduction}

Understanding cellular phenotypes and how they mediate response to exposures, e.g., drugs, cytokines, chemokines, is critical to disease biology and translational clinical research. Single-cell RNA sequencing (scRNA-seq) enables the characterization of cellular phenotypes and measuring responses at cellular resolution, revealing cell states, trajectories, and disease mechanisms \citep{scvelo,Aevermann2018,kang,Wu2021}. Yet the combinatorial space of cellular populations and conditions (treatments, exposures, doses) makes exhaustive experimental mapping infeasible \citep{scVIDR}. 
Computational models that can help fill this map by predicting cellular responses under unseen conditions are therefore essential.


A critical capability required of these models is the ability to perform precise \textbf{cellular editing}: starting with an observed cell and systematically modifying specific biological properties while preserving others. Editing differs from conditional generation in that the latter considers \textit{any} cell under given conditions, while the former considers a \textit{specific} cell under altered conditions. For example, a researcher might ask \textit{``What would \underline{this} T-cell look like after receiving anti-PD-1 or anti-TNF treatment, compared to its observed untreated state? What if its NF-kB pathway is turned off?''}. Such editing capabilities enable cell-level counterfactual reasoning, which is critical for causal discovery, therapeutic design, and precision medicine.


A cell-editing model must satisfy two requirements. First, it should generate cell-specific counterfactuals, predicting how an observed cell would respond to a specified intervention, rather than only population-level averages. Second, it should provide interpretable control, allowing interventions on biologically meaningful concepts such as gene programs or cell types, rather than opaque latent variables. Earlier methods modeled conditional distributions of cell states across contexts, capturing population-level effects but not counterfactuals for individual cells \citep{scgen2019, scVIDR, state}. More recent work has moved toward cell-level counterfactual prediction \citep{zhang2024scdisinfact, biolord}, but explicit, interpretable control remains an open and important area for further development. Meanwhile, existing interpretability methods are descriptive but not actionable: they explain correlations but cannot simulate or edit cellular responses \citep{scETM,Chen2024}. Enabling precise counterfactual editing at the level of individual cells with biologically interpretable control remains an open challenge.



 \looseness=-1 In this work, we introduce single-cell Concept Bottleneck Generative Models (scCBGM), a generative framework that enables interpretable and controllable cellular editing. Our approach builds on Concept Bottleneck Generative Models (CBGMs)~\citep{cbgm}, which extend concept bottleneck models~\citep{koh2020concept} to the generative setting, adapting them to the unique challenges of single-cell data: high heterogeneity, complex biological processes, substantial technical noise, and unreliable concept annotations. Our main contributions are:

\begin{itemize}[itemsep=2pt, parsep=0pt, topsep=0pt, partopsep=0pt]
\item We introduce architectural modifications to CBGMs: a computationally efficient cross-covariance penalty that promotes decoupled embeddings without imposing dimensionality constraints on model components and decoder concept skip connections that enable controllable generation under noisy biological annotations.
\item We extend scCBGM beyond VAEs to flow matching models, enabling concept-guided editing in both decoding-only and encoding–decoding regimes. This demonstrates the flexibility of our framework across generative architectures.
\item We develop a synthetic data generation process that separates exogenous noise from conditions, providing access to true counterfactuals. This enables systematic evaluation of concept-based editing under realistic single-cell settings (noise, missing labels, heterogeneous populations).
\item We demonstrate strong cell-level editing accuracy on our synthetic benchmark, where ground-truth counterfactuals are available, and competitive-to-strong performance on population-level proxy metrics across three real-world datasets, outperforming several state-of-the-art methods in combinatorial and zero-shot generalization. We show a use case where in-silico interventions with \method~on both cellular phenotype, i.e., pathway activity \textit{and} treatment elucidate mechanistic hypotheses of treatment response.
\end{itemize}

\section{scCBGM   Single-Cell Concept Bottleneck Generative Model}
\label{section:sccbgm}

%


This section formalizes the problem of single-cell counterfactual editing, describes the architecture and training objectives of \method, and details a procedure for augmenting flow matching generative models with \method.

\subsection{Problem setup} 

We assume $N$ i.i.d.\ single-cell RNA-seq profiles $\mathbf{x}\in\mathbb{R}^d$, each with $d$ genes, and an associated vector of $K$ biological concepts $\mathbf{c}\in\mathbb{R}^K$. Each concept may be binary (e.g., cell type or stimulation status, encoded as $0/1$) or continuous/soft (e.g., a drug dosage or pathway-activity score). Our goal is to generate \emph{counterfactual gene expression}: given a factual cell $(\mathbf{x},\mathbf{c})$, predict what the \emph{same} cell would look like if its concepts had been set to $\mathbf{c}'$ instead.

\textbf{Notation.} We use capital letters for random variables ($U, U_C, C, X$), lowercase bold for observed data ($\mathbf{x}, \mathbf{c}, \mathbf{c}'$), and lowercase with hats for model predictions ($\hat{\mathbf{c}}, \hat{\mathbf{u}}$).

\textbf{Data-generating process.}
We adopt Pearl's structural causal model (SCM) formalism \citep{pearl2009causality}. A gene expression $X$ is influenced by two sources of variation as illustrated in Figure ~\ref{fig:dag_ucx}: the observed concepts $C$ and unobserved residual factors $U$. Concepts themselves are driven by their own exogenous variables $U_C$:
\begin{equation}
\begin{split}
  &U\!\sim\!P(U),\quad U_C\!\sim\!P(U_C), \\
  &C \leftarrow f_C(U_C),\quad X \leftarrow f_X(C,U).
\end{split}
\label{eq:struct}
\end{equation}
where $P(U)$ and $P(U_C)$ are probability distributions over the unobserved factors, and $f_C$ and $f_X$ are functions that determine how these factors generate concepts and gene expression, respectively. We assume $U \perp U_C$; therefore $U \perp C$ under $C\!\leftarrow\!f_C(U_C)$.

\textbf{Counterfactuals.} The key insight is to ask \textit{“what if”} questions about cells. Given observed gene expression $\mathbf{x}$ and associated concepts $\mathbf{c}$, we ask: what would the \emph{same} cell look like if its concepts were instead $\mathbf{c}'$, everything else being equal? We refer to the observed gene expression as the \emph{factual} outcome $X$, and the hypothetical alternative as the \emph{counterfactual} outcome $X'$. Formally, for a given unobserved residual factor $\mathbf{u}$, if the concepts were $\mathbf{c}'$ rather than $\mathbf{c}$, the counterfactual gene expression would be:
\begin{equation}
X'(\mathbf{u}) = f_X(\mathbf{c}', \mathbf{u}),
\label{eq:cf-value}
\end{equation}
where $\mathbf{u}$ represents the same unobserved factors that generated the original cell.  Given a factual observation $\mathbf{x}$, we define the \emph{counterfactual edit} of that cell to concepts $\mathbf{c}'$ as the expected counterfactual outcome:
\begin{equation}
\mu_{\mathbf{x},\mathbf{c}'} := \mathbb{E}_U\!\left[f_X(\mathbf{c}', U)\,\middle|\, X = \mathbf{x}\right].
\label{eq:target}
\end{equation}
This represents what we expect that specific cell would look like on average if it had concepts $\mathbf{c}'$ instead, where the expectation is over the posterior of $U$ given $X=\mathbf{x}$ (abduction).





\begin{figure}[htpb!]
  \centering
  \begin{tikzpicture}[->, >=stealth, thick, scale=1,
    latent/.style={draw, circle, dashed, minimum size=10mm, inner sep=0pt},
    obs/.style={draw, circle, minimum size=10mm, inner sep=0pt}]
    \node[latent] (U) {$U$};
    \node[latent] (UC) at ($(U)+(-3.0cm,0)$) {$U_C$};
    \node[obs] (C) at ($(UC)+(0,-2cm)$) {$C$};
    \node[obs] (X) at ($(U)+(0,-2cm)$) {$X$};
    \draw (UC) -- (C);
    \draw (U)  -- (X);
    \draw (C)  -- (X);
  \end{tikzpicture}
  \caption{Directed Acyclic Graph (DAG) of the data-generating process. Two unobserved variables \(U\) and \(U_C\) drive \(X\) and \(C\), respectively, with \(C\) also influencing \(X\), consistent with \(C \leftarrow f_C(U_C),\; X \leftarrow f_X(C,U)\).}
  \label{fig:dag_ucx}

\end{figure}
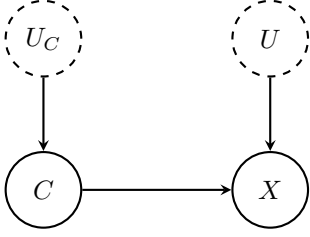

\textbf{Cell editing as counterfactual prediction.}
We model single-cell editing as learning a predictor of counterfactual edit (\Eqref{eq:target}). Given an observed cell $\mathbf{x}$ with concepts $\mathbf{c}$, the goal is to learn a function $f_{\theta}$ that generates the cell under alternative concepts $\mathbf{c}'$:
\begin{equation}
\hat{\mathbf{x}}' = f_{\theta}(\mathbf{x}, \mathbf{c}') \approx \mu_{\mathbf{x},\mathbf{c}'}.
\label{eq:learning-goal}
\end{equation}
In other words, $f_{\theta}$ estimates how the same cell $\mathbf{x}$ would appear if its concepts were $\mathbf{c}'$ rather than $\mathbf{c}$; this is equivalent to the cell editing task.

\subsection{Architecture of \method}
\label{subsection:architecture}
The overall architecture of our model is shown in Figure~\ref{fig:scCBGM_arch}. It follows an encoder–decoder design with a concept bottleneck. The encoder $E(\cdot)$ maps single-cell expression $\mathbf{x}$ to a latent $\mathbf{z}$, which is decomposed into known factors $\hat{\mathbf{c}}$ and unknown factors $\mathbf{u}$ reflecting the assumed data-generating process. These are concatenated and passed to the decoder $D(\cdot)$ to reconstruct $\mathbf{x}$. Details of the encoder, decoder, and bottleneck are given below, with losses and training described in Section~\ref{sub_section:losses}.



\begin{figure*}[t]
\centering
\includegraphics[width=0.58\linewidth]{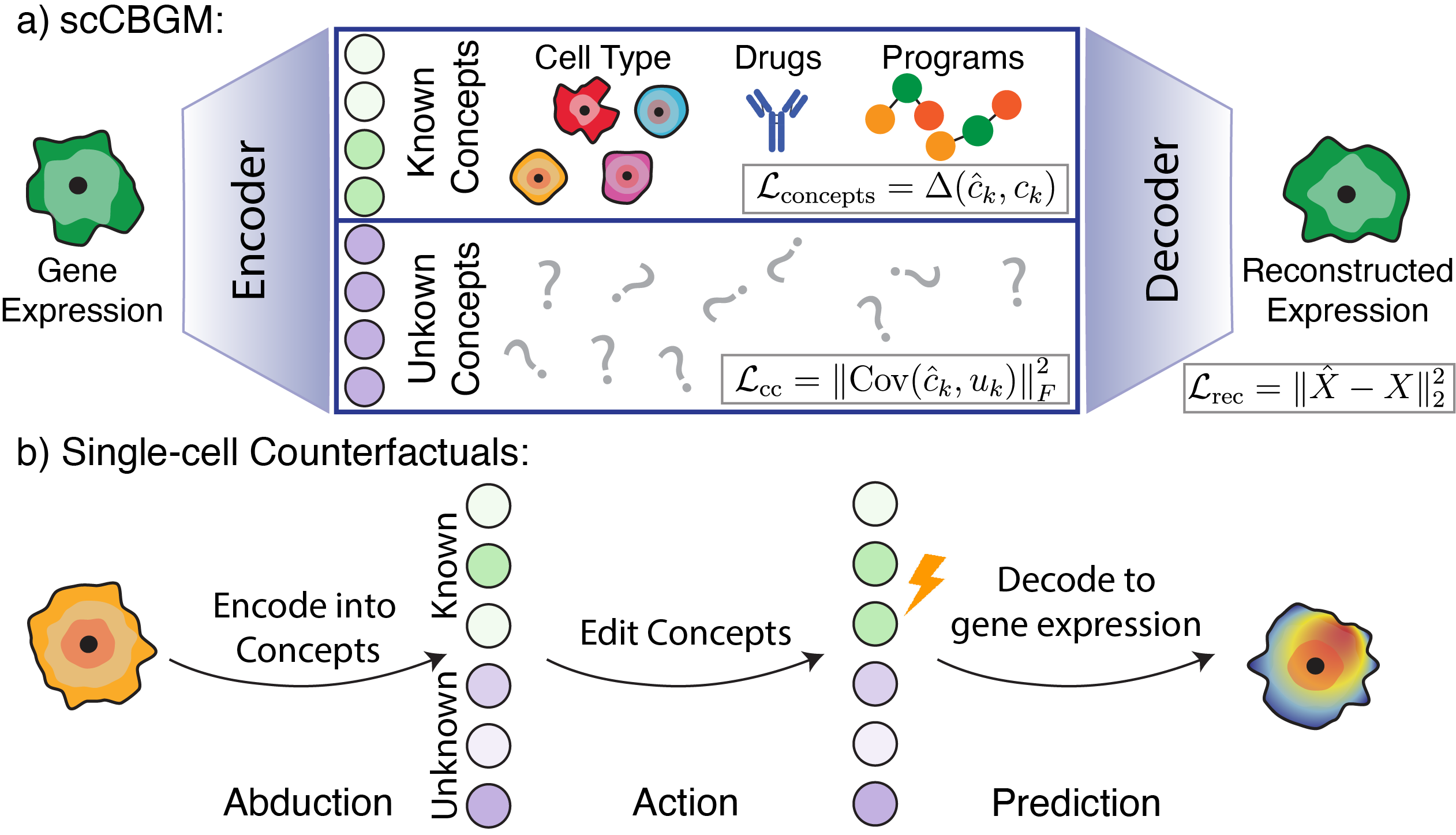}
\caption{Single-cell Concept Bottleneck Generative Model Overview. \textbf{Top:} A concept bottleneck VAE encodes gene expression into known concepts and unknown concepts, which are subsequently decoded to reconstruct the expression profile. The model is trained to match the encoded concepts ($\hat{c}_{k}$) to the ground truth ($c_{k}$) while keeping the unknown concepts ($u_{k}$) independent via the cross-covariance loss. \textbf{Bottom:} Counter-factuals are constructed by: \textbf{1.} encoding each cells into its known and unknown concept. \textbf{2.} editing the desired concepts while keeping the rest stationary and \textbf{3.} decoding with the edited embedding.}
\label{fig:scCBGM_arch}
\end{figure*}


\textbf{Encoder}
The encoder $E(\cdot)$ is a neural network that maps an input $\mathbf{x}\in\mathbb{R}^d$ to the parameters of a multivariate Gaussian, defining the posterior distribution $q(\mathbf{z}|\mathbf{x}) = \mathcal{N}(\mu_z,\Sigma_z)$. A latent vector $\mathbf{z}$ is then sampled from this distribution. 

\textbf{Concept bottleneck module} The concept bottleneck module is composed of (1) a concept network~\cite{koh2020concept} that parametrizes a function $f_c$ mapping the latent representation $\mathbf{z}$ into concept predictions $\hat{\mathbf{c}}=f_c(\mathbf{z})\in\mathbb{R}^{K}$, and (2) a network $f_u$ that produces representation $\mathbf{u}=f_u(\mathbf{z})\in\mathbb{R}^{d_u}$ capturing variation not explained by $\hat{\mathbf{c}}$. We refer to $\hat{\mathbf{c}}$ and $\mathbf{u}$ as the known and unknown factors, respectively.

\textbf{Decoder with concept skip connections} The decoder is a network $D(\cdot)$ that takes the concatenation of the known and unknown factors as input ($\mathbf{h}_0=[\mathbf{u},\hat{\mathbf{c}}]$). At each layer $\ell > 1$, we feed the previous hidden state concatenated with the known factors:
\begin{equation}
\begin{split}
    &\mathbf{h}_1 = D_1(\mathbf{h}_0), \quad \ell=2,\dots,L, \\
    &\mathbf{h}_\ell = D_\ell\big([\mathbf{h}_{\ell-1},\,\hat{\mathbf{c}}]\big), \quad \hat{\mathbf{x}} = D_{\text{final}}(\mathbf{h}_L).
\end{split}
\label{eq:layer}
\end{equation}
Including $\hat{\mathbf{c}}$ at every layer enforces a stronger and more systematic conditioning on the concepts. This leads to better performance than only adding $\hat{\mathbf{c}}$ in the input $\mathbf{h}_0$, especially when dealing with noisy concept annotations, as demonstrated empirically in Section~\ref{sec:scCBGM_vs_CBGM}.

\subsection{Loss functions and training}
\label{sub_section:losses}

We train the model by minimizing a loss derived from the $\beta$-VAE evidence lower bound (ELBO), \emph{i.e.,} a reconstruction term, a concept supervision term, and a $\beta$-scaled KL regularization:
\begin{equation}
\begin{split}
\mathcal{L}_{\text{VAE}}
&= -\mathbb{E}_{q(\mathbf{z}|\mathbf{x})}\!\left[\log p(\mathbf{x}|\mathbf{z})\right]
+ \lambda_c\,\mathcal{L}_{\text{concept}} \\
&\quad + \beta\,\mathrm{KL}\!\left(q(\mathbf{z}\mid \mathbf{x})\,\|\,p(\mathbf{z})\right).
\end{split}
\label{eq:loss}
\end{equation}
where $\mathcal{L}_{\text{concept}}$ is a surrogate for $-\mathbb{E}_{q(\mathbf{z}\mid\mathbf{x})}[\log p(\mathbf{c}\mid\mathbf{z})]$, as described next.

\textbf{Concept loss $\mathcal{L}_{\text{concept}}$}
Biological data often has missing concept annotations. To accommodate this, we define a binary mask $\mathbf{m}\in\{0,1\}^{K}$ indicating which of the $K$ total concepts are present for a given sample. We partition the concepts into binary and continuous sets, indexed by $K_{\text{bin}}$ and $K_{\text{cont}}$ respectively. The total concept loss is the normalized sum of the \textbf{Binary Cross Entropy (BCE)} loss for binary concepts and the \textbf{Mean Squared Error (MSE)} loss for continuous concepts:
\begin{equation}
\begin{split}
\mathcal{L}_{\text{concept}}
= \frac{1}{\sum_k m_k} \Bigg( &\sum_{k \in K_{\text{bin}}} m_k \, \mathcal{L}_{\text{BCE}}(\hat{c}_k, c_k) \\
&+ \sum_{k \in K_{\text{cont}}} m_k \, \mathcal{L}_{\text{MSE}}(\hat{c}_k, c_k) \Bigg).
\end{split}
\label{eq:concept_loss}
\end{equation}

where $\hat{c}_k$ is the prediction for the ground-truth concept $c_k$.

\textbf{Cross-covariance penalty}
To enforce independence between latent factors $U$ and concepts $C$, we penalize their cross-covariance. Unlike cosine-similarity losses, this approach accommodates arbitrary embedding dimensions ($K \neq d_u$). For a minibatch of size $B$, the penalty is the squared Frobenius norm of the empirical cross-covariance between predicted concepts $\hat{C} \in \mathbb{R}^{B \times K}$ and unknown factors $U \in \mathbb{R}^{B \times d_u}$:
\begin{equation}
\mathcal{L}_{\text{cc}} = \left\| \frac{1}{B-1}(\hat{C} - \mathbf{1}\boldsymbol{\mu}_{\hat{c}}^{\top})^{\top}(U - \mathbf{1}\boldsymbol{\mu}_u^{\top}) \right\|_F^2
\end{equation}
where $\boldsymbol{\mu}_{\hat{c}}$ and $\boldsymbol{\mu}_u$ are the column mean vectors of $\hat{C}$ and $U$, respectively. Scaling is enforced implicitly for $\hat{C}$ via the \textbf{Concept Loss}, and for $U$ via a ReLU non-linearity. This prevents degenerate solutions where the unknown factors take on near-zero values. This pushes the model to encode non-concept specific sources of variation (e.g., batch, cell type) in $\mathbf{u}$, as theoretically analyzed in Appendix~\ref{app:decoupling_proof} and empirically demonstrated in Appendix \ref{app:concept_leakage}.

\textbf{Final loss}

The complete training objective is obtained by combining $\mathcal{L}_\text{VAE}$ and $\mathcal{L}_{\text{cc}}$:
\begin{equation}
\label{eq:full_loss}
    \mathcal{L} = \mathcal{L}_\text{VAE} +\lambda_{\text{cc}} \mathcal{L}_{\text{cc}}
\end{equation}
To further motivate our choice of model architecture and loss function, we provide proofs for the consistency of the counterfactual estimator learnt by minimizing $\mathcal{L}$ in Appendix~\ref{app:consistency}.


\textbf{Differences with standard concept bottleneck generative models.}
Our approach differs from standard CBGMs~\citep{cbgm} in three important ways: (i) we use a standard concept bottleneck model \citep{koh2020concept}, where the bottleneck maps inputs to scalar concept predictions that directly correspond to interpretable concepts, rather than a concept embedding model \citep{espinosa2022concept} where the bottleneck maps each concept to high-dimensional representations;(ii) we add skip connections to the decoder to maintain persistent concept conditioning; and (iii) we use a cross-covariance loss \cite{hsic_2009,hsic_2020} instead of the cosine similarity loss for orthogonality. Changes (i) and (ii) improve robustness to noisy concept annotations, while change (iii) removes dimensional constraints on embeddings that are enforced by CBGMs. We show empirical performance gains compared to existing CBGMs in Section~\ref{sec:scCBGM_vs_CBGM}.

\subsection{Counterfactual prediction}
Our model enables counterfactual predictions by implementing the standard abduction-action-prediction framework \citep{pearl2009causality} with our generative architecture. (\textbf{1-Abduction Step}) Given an observed cell $\mathbf{x}$, we first encode it with our encoder to obtain $\mathbf{z} = E_{\mu}(\mathbf{x})$, and decompose it into concepts $\hat{\mathbf{c}}$ and unknown factors $\mathbf{u}$. (\textbf{2-Action Step}) We edit the desired concept values, to obtain a modified concept vector $\hat{\mathbf{c}}'$. We edit the vector $\hat{\mathbf{c}}$ by assigning the dimensions $k$ for which the $\hat{\mathbf{c}}$ and $\mathbf{c}'$ differ, and leave the other dimensions untouched. That is, $ \hat{\mathbf{c}}' \equiv \hat{\mathbf{c}}_k\leftarrow \mathbf{c}'_k, \forall k: \mathbf{c}_k' \neq \mathbf{c}_k$. (\textbf{3-Prediction Step}) Lastly, we decode the modified representation $[\mathbf{u}, \hat{\mathbf{c}}']$ to produce the counterfactual prediction $\hat{\mathbf{x}}' = D([\mathbf{u},\hat{\mathbf{c}}'])$. Figure~\ref{fig:scCBGM_arch} shows an illustration of the abduction-action-prediction process.

\subsection{Enabling fine-grained control of generative models with \method~conditioning}
\label{sub_section:scCBMFM}

While \method~is a standalone generative model, it can also enhance other generative models' architectures to enable more fine-grained control. By conditioning the generative process on the \method~embeddings, we can combine the generation quality of state-of-the-art generative models with the interpretability and controllability of our method. In this section, we describe the procedure for flow matching (FM) models~\citep{lipman2022flow,liu2022flow,lipman2024flow}; the extension to diffusion models is analogous. 



\textbf{\method-guided FM}
Given a trained \method-VAE, we train a flow matching model by directly conditioning on the \method~embeddings. For a cell $\mathbf{x}$, we extract its concept embeddings $\hat{\mathbf{c}} = f_c(\mathbf{z})$ and unknown representations $\mathbf{u} = f_u(\mathbf{z})$ where $\mathbf{z} = E_{\mu}(\mathbf{x})$. We use these vectors to learn a conditional vector field $v_\theta(\mathbf{x}_t, t; [\mathbf{u}, \hat{\mathbf{c}}])$ by minimizing the conditional flow matching loss:
\begin{equation*}
\mathcal{L}_{\text{FM}} = \displaystyle \mathbb{E}_{t,\mathbf{\zeta}\sim q,\mathbf{x}_t \sim p_t(\cdot |\mathbf{z})} \Big[\|v_\theta(\mathbf{x}_t, t; [\mathbf{u}, \hat{\mathbf{c}}]) \\
- v^*(\mathbf{x}_t, t\mid\mathbf{\zeta})\|_2^2\Big]
\end{equation*}
Here, $v^*$ is the target conditional velocity field that generates the conditional probability paths $p_t(\cdot\mid\mathbf{\zeta})$ for $t \in [0,1]$, with $p_0(\cdot|\zeta)$ and $p_1(\cdot|\zeta)$ are boundary conditional probabilities whose marginal match with a noise distribution and the data distribution at $t=0$ and $t=1$ respectively~\citep{lipman2022flow}. The learnt conditional vector field $v_\theta(\mathbf{x}_t, t; [\mathbf{u}, \hat{\mathbf{c}}])$ generates a flow $\varphi_t(\mathbf{x}, [\mathbf{u}, \hat{\mathbf{c}}])$ defined by $\frac{d}{dt}\varphi_t(\mathbf{x}, [\mathbf{u}, \hat{\mathbf{c}}]) = v_\theta(\varphi_t(\mathbf{x}_t,[\mathbf{u}, \hat{\mathbf{c}}]), t; [\mathbf{u}, \hat{\mathbf{c}}])$ and $\varphi_0(\mathbf{x}, [\mathbf{u}, \hat{\mathbf{c}}]) = \mathbf{x}$. Using the flow, one can generate counterfactuals in two ways: decoding-only and encoding-decoding.

\textbf{Decoding-only counterfactual prediction}
Given $\mathbf{x}$, we obtain $\hat{\mathbf{c}}'$ via the abduction and action steps on \method. We then predict the counterfactuals by sampling 
$\mathbf{x_0} \sim \mathcal{N}(\mathbf{0}, \mathbf{I})$ and passing it to the conditional flow $\hat{\mathbf{x}}' = \varphi_1(\mathbf{x}_0, [\mathbf{u}, \hat{\mathbf{c}}'])$. This approach is useful for generating diverse examples of cells with a specific profile. We refer to this approach as scCBGM-FM (decode).

\textbf{Encoding-decoding counterfactual prediction}
A more accurate option for counterfactual prediction relies on mapping $\mathbf{x}$ back to the noise distribution (encoding step) and then using the conditional flow from this initial condition to generate the modified cell (decoding step). Importantly, the encoding step uses a conditioning on $\hat{\mathbf{c}}$ while, the decoding uses the edited concepts $\hat{\mathbf{c}}'$. Given a starting cell $\mathbf{x}$, we \textit{encode} it using $\mathbf{x}_0 = \varphi_1^{-1}(\mathbf{x},[\mathbf{u}, \hat{\mathbf{c}}])$, then \textit{decode} it using  $\hat{\mathbf{x}}' = \varphi_1(\mathbf{x}_0, [\mathbf{u}, \hat{\mathbf{c}}'])$.  This implements Pearl's abduction-action-prediction framework~\citep{pearl2009causality, xia2025decoupled, rout2024semantic}: mapping $\mathbf{x}$ to noise $\mathbf{x}_0$ constitutes the \textit{abduction} of exogenous factors (cell identity), while the forward pass under intervention $\hat{\mathbf{c}}'$ performs the \textit{action} and \textit{prediction}~\citep{sanchez2022diffusion, wang2024taming}. We refer to this approach as scCBGM-FM (edit).


\section{Do we need a single-cell specific CBGM?}
\label{sec:scCBGM_vs_CBGM}

While concept bottleneck models have been explored in prior work, existing architectures 
do not address the unique challenges of scRNA-seq data, a modality that presents 
fundamentally different problems than well-curated domains such as images. In scRNA-seq, 
some concepts are well-defined by the experimental design (e.g., drug treatments, dosages), 
whereas others are inherently noisier (e.g., cell state, pathway activation). Noisy labels 
may arise from challenging annotation tasks, limited signal, incomplete coverage, or 
redundancy across concepts. To demonstrate that our architectural modifications address 
these challenges, we systematically compare scCBGM against standard CBGMs using synthetic 
data where ground-truth counterfactuals are available for direct evaluation.


\textbf{Synthetic Data for Evaluation} To compare model performance, we evaluate (1) cell-level counterfactual generation and (2) robustness to noisy concepts. Because real data do not permit controlled evaluation, we propose a synthetic scRNA-seq generation process based on a hierarchical overdispersed Poisson model \cite{Pan2023, arxiv_pln, Subedi2025}. In this process, gene expression is decomposed into contributions from batches, tissues, cell types, and concepts. By isolating exogenous noise, we can generate true counterfactuals for evaluating model predictions. The process also supports the injection of four types of noisy concept annotations: incorrectly annotated, irrelevant, missing, and duplicated concepts. Appendix~\ref{app:synth:process}-\ref{app:synth:example} describes the data generating process , and Appendix~\ref{app:synth:noisy_concepts} details the noisy concept types.

\begin{figure}[t]
    \centering
    \includegraphics[width=0.7\columnwidth]{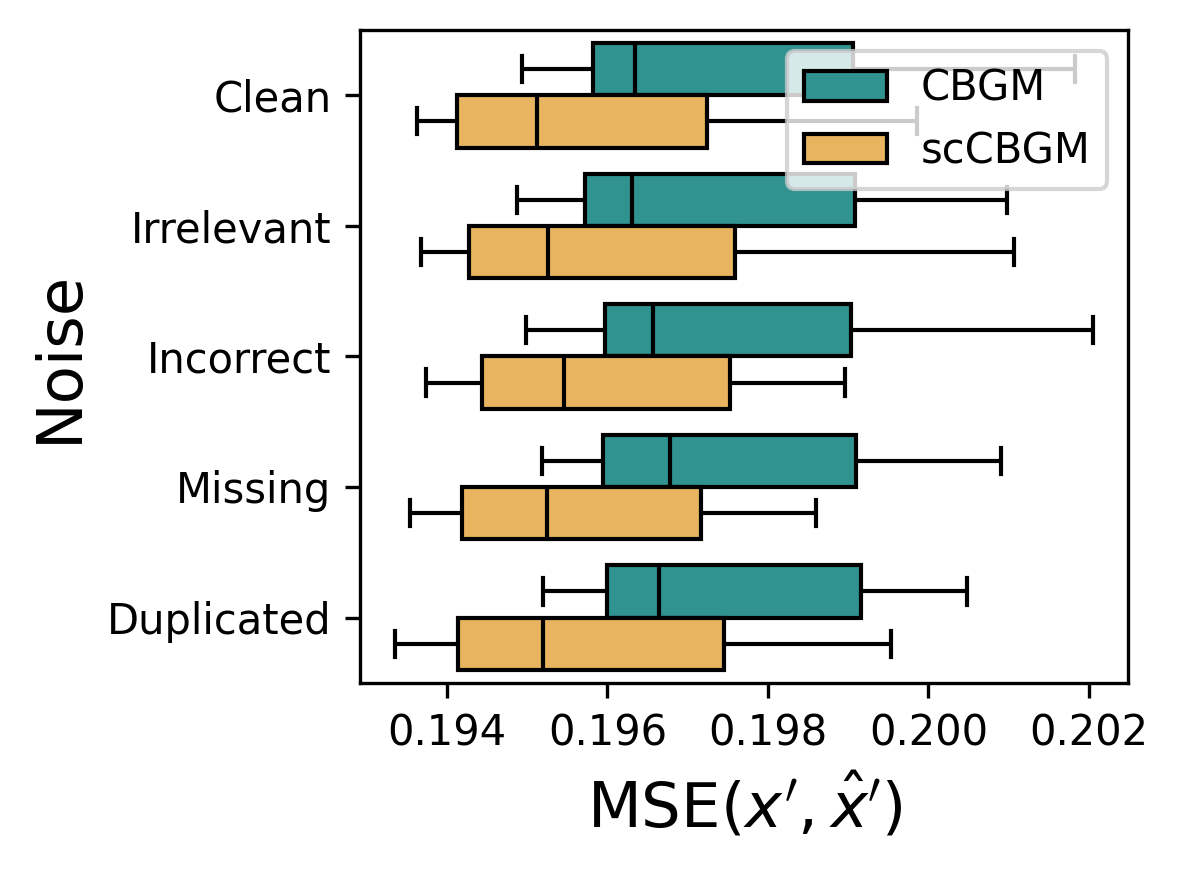}
    \caption{ \looseness=-2 Standard CBGMs vs. \method~under different concept annotations noise. 
    MSE between true and predicted counterfactuals across datasets (3), interventions (5), noise levels (3), and seeds (2).}
    \label{fig:synth_comparison}
\end{figure}

\textbf{Results} We compare standard CBGMs with \method~to demonstrate the importance of our architectural modifications for single-cell data. Experiments use three synthetic datasets (20,000 cells, 5,000 genes) varying in technical noise and concept effect size (Appendix~\ref{app:synth:synt_datasets}). For each dataset, we hold out five concept–cell type pairs, meaning the model never observes these combinations during training, thus emulating a compositional generalization task (Appendix~\ref{app:synth_split}). We evaluate robustness under four types of concept noise at two corruption levels with four random seeds. Intervention performance is measured by MSE between predicted and true counterfactuals. To ensure fairness, we conducted a hyperparameter sweep (432 configurations) on one noiseless dataset and selected the best model by intervention performance (Appendix~\ref{app:syn:hparam_sweep}). Results are presented in Figure~\ref{fig:synth_comparison}. \method~outperforms standard CBGM across all scenarios, both noisy and clean. scCBGM demonstrates greater robustness to noise, with only minimal performance degradation under different noise conditions.  Furthermore, as shown in Appendix~\ref{app:performance_by_concept}, \method~remains stable as the number of concepts increases.


%

\textbf{Ablation Study} 
To evaluate the impact of our main architectural components, we conducted an ablation study varying three factors: the decoder type (skip vs.\ direct), the concept head (concept bottleneck vs.\ concept embedding), and the orthogonality loss (cross-covariance vs.\ cosine). Experiments were performed on three synthetic datasets, each with five intervention types, two random seeds, and multiple noise settings. The results, summarized in Table~\ref{tab:ablation}, indicate that the concept bottleneck (CBM) variant is preferable to the concept embedding (CEM) counterpart on this class of data. Within the CBM model family, incorporating the cross-covariance loss further improves performance. The skip-connection decoder enhances performance only when paired with both the cross-covariance loss and the concept-bottleneck architecture. This behavior is supported by ablations on real datasets, see Appendix~\ref{app:real_ablations}.

\begin{figure*}[!ht]
  \centering
  \includegraphics[width=0.8\textwidth]{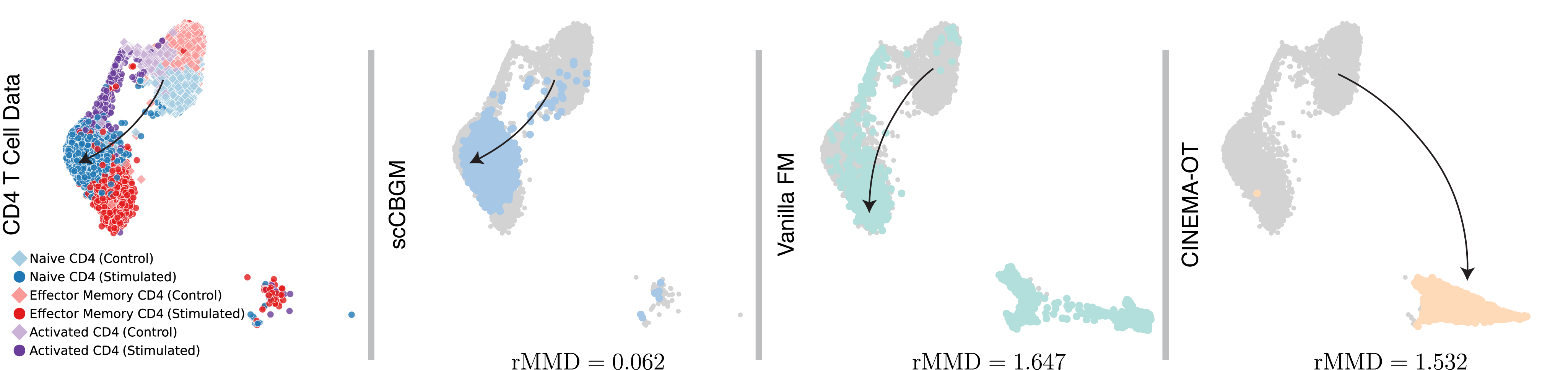}
  \caption{Counterfactual modeling predicts cellular response to perturbation. \textbf{Left:} UMAP of CD4 T-cell data from the Kang et al. data, showing Control and Stimulated cells, colored by subtype. Stimulated Naive CD4 T cells were held out during training. \textbf{Right:} Control Naive CD4 T cells are edited in silico to predict their stimulated state. scCBGM (second panel) accurately predicts the held-out stimulated Naive CD4 T cells, demonstrating superior zero-shot generalization compared to  \cellflow~ and CINEMA OT.}
    \label{fig:kang_stimulation}
\end{figure*}

\section{Related Works}
\label{sec:scCBGM_vs_exiting}
scCBGM delivers accurate control over single-cell editing by uniquely integrating decoupled and interpretable representations with counterfactual generation. Existing methods typically address these aspects in isolation: most approaches for disentangled representations of single cells do not directly support interventions, while most intervention-based models capture conditional probability distributions without enabling fine-grained editing or counterfactual generation. Below, we review key methods from both domains.

\textbf{Disentangled and interpretable representations for single cell data.} Many methods have recently proposed interpretable latent decompositions for single cell data~\citep{lopez2018deep,LDVAE,PmVAE,vega,sivae,lopez2023learning,pouyabahar2025interpretable}. Among these, a few combine interpretability with intervention prediction. For example, Celcomen~\citep{celcomen} leverages disentangled representations to model gene–gene interactions in spatial data and MichiGAN~\citep{yu2021michigan} conditions a generative model on the latent space of a $\beta-$VAE which provides limited biological interpretability. Closer to our work, scDisInFac~\citep{zhang2024scdisinfact} learns a disentangled representation of biological conditions and batch effects but assumes mutually independent factors and additive treatment mechanism in the latent space, which are unrealistic in our intended applications. biolord~\citep{biolord} incorporates interventions and use a noise injection mechanism for disentanglement. CinemaOT~\citep{CinemaOT} first factors the data into causal and spurious factors before predicting treatment effects, but is limited to mapping cells to observations in the training data, restricting its generalizability.

\looseness=-1
\textbf{Single-cell intervention prediction.}
 scRNA-seq measurements cannot be repeated on the same cell, models typically predict intervention effects at the population level. Early approaches modeled temporal dynamics between cell populations \citep{Yeo2021,WOT}, later extended to perturbation-response prediction \citep{rohbeck2025modeling,CellOT,cellflow}. Conditional generative models without explicit dynamics  have also been adapted~\citep{CFGen,scDiffusion,scGAN,ATAC-Diff}. However, these models primarily learn \emph{conditional distributions} and do not support counterfactual inference or cell-level editing. Moreover, many focus on specific perturbation classes such as gene knockouts~\citep{GEARS,wang2024modeling,littman2025gene} or chemical treatments~\citep{qi2024predicting}, whereas our approach supports diverse biological interventions.  Methods that attempt cell-level editing, such as scVIDR and scGen~\citep{scVIDR,scgen2019}, depend on strong assumptions about perturbation mechanisms (e.g., additive latent effects) and do not provide interpretable representations.

\section{Experiments}
\label{sec:experiments}

\subsection{Experiment setup and metrics}
We evaluate \method~on its ability to perform precise, fine-grained, and identity-preserving edits across three diverse real-world datasets~\citep{kang, cui2024dictionary, nault2023single}. Since exogenous noise in real-world data cannot be controlled, exact counterfactual evaluation is infeasible. To approximate this, we treat broad cell-type categories as observable concepts (\emph{e.g.}, CD4 T cells) and consider granular subtypes (\emph{e.g.}, \emph{activated} CD4 T cells) as ground-truth counterfactual distributions.

\begin{table*}[!ht]
\centering
\begin{adjustbox}{width=\tabwidth}
\begin{tabular}{lccccccc}
\toprule
\textbf{Model} & \textbf{B cells} & \makecell{\textbf{T cells} \\ \textbf{(CD4)}} & \makecell{\textbf{T cells} \\ \textbf{(CD8)}} & \makecell{\textbf{Monocytes} \\ \textbf{(FCGR3A)}} & \makecell{\textbf{Monocytes} \\ \textbf{(CD14)}} & \textbf{Dendritic} & \textbf{NK cells} \\
\midrule
scCBGM & \thirdbest{0.112 $\pm$ 0.028} & \thirdbest{0.169 $\pm$ 0.086} & \thirdbest{0.171 $\pm$ 0.012} & \thirdbest{1.845 $\pm$ 1.776} & 1.309 $\pm$ 0.231 & \thirdbest{0.375 $\pm$ 0.041} & 1.167 $\pm$ 0.271 \\
scCBGM-FM (decode) & \secondbest{0.106 $\pm$ 0.032} & \secondbest{0.162 $\pm$ 0.073} & \secondbest{0.138 $\pm$ 0.040} & \bestres{\textbf{1.141 $\pm$ 1.078}} & 1.334 $\pm$ 0.296 & \secondbest{0.288 $\pm$ 0.061} & \thirdbest{0.093 $\pm$ 0.016} \\
scCBGM-FM (edit)  & \bestres{\textbf{0.093 $\pm$ 0.031}} & \bestres{\textbf{0.156 $\pm$ 0.066}} & \bestres{\textbf{0.119 $\pm$ 0.019}} & \secondbest{1.206 $\pm$ 1.167} & \thirdbest{1.171 $\pm$ 0.390} & \bestres{\textbf{0.231 $\pm$ 0.042}} & \secondbest{0.084 $\pm$ 0.003} \\
\midrule
CBGM & 0.902 $\pm$ 0.058 & 2.228 $\pm$ 0.079 & 1.914 $\pm$ 0.049 & 9.514 $\pm$ 0.290 & 7.206 $\pm$ 0.196 & 1.503 $\pm$ 0.039 & 1.270 $\pm$ 0.140 \\
\midrule
\cellflow (decode) & 0.926 $\pm$ 0.067 & 1.037 $\pm$ 0.197 & 0.912 $\pm$ 0.079 & 4.389 $\pm$ 4.029 & \secondbest{0.549 $\pm$ 0.240} & 2.250 $\pm$ 1.019 & 0.099 $\pm$ 0.015 \\
\cellflow (edit) & 0.492 $\pm$ 0.174 & 0.487 $\pm$ 0.098 & 0.364 $\pm$ 0.078 & 2.188 $\pm$ 1.848 & \bestres{\textbf{0.394 $\pm$ 0.095}} & 1.307 $\pm$ 0.145 & \bestres{\textbf{0.082 $\pm$ 0.021}} \\
\midrule
biolord & 2.622 $\pm$ 0.128 & 5.514 $\pm$ 1.856 & 4.829 $\pm$ 0.680 & 12.123 $\pm$ 11.118 & 2.350 $\pm$ 0.385 & 3.904 $\pm$ 1.480 & 2.355 $\pm$ 0.006 \\
biolord-FM & 1.163 $\pm$ 0.003 & 2.309 $\pm$ 0.003 & 1.762 $\pm$ 0.008 & 11.749 $\pm$ 0.054 & 10.091 $\pm$ 0.071 & 2.191 $\pm$ 0.010 & 1.051 $\pm$ 0.000 \\
CINEMA-OT & 2.259 $\pm$ 0.276 & 7.042 $\pm$ 2.520 & 5.362 $\pm$ 0.562 & 11.193 $\pm$ 9.807 & 10.185 $\pm$ 5.194 & 1.367 $\pm$ 0.333 & 3.707 $\pm$ 0.008 \\
scGen & 1.830 $\pm$ 0.397 & 5.117 $\pm$ 2.398 & 4.748 $\pm$ 0.072 & 7.227 $\pm$ 6.673 & 5.748 $\pm$ 0.349 & 1.133 $\pm$ 0.587 & 2.436 $\pm$ 0.072 \\
CVAE & 0.620 $\pm$ 0.008 & 1.374 $\pm$ 0.008 & 1.043 $\pm$ 0.027 & 9.055 $\pm$ 0.040 & 6.446 $\pm$ 0.067 & 1.134 $\pm$ 0.021 & 0.706 $\pm$ 0.010 \\
CVAE-FM (decode) & 0.525 $\pm$ 0.005 & 1.115 $\pm$ 0.010 & 0.901 $\pm$ 0.007 & 8.379 $\pm$ 0.045 & 6.051 $\pm$ 0.076 & 0.982 $\pm$ 0.012 & 0.604 $\pm$ 0.008 \\
CVAE-FM (edit) & 0.512 $\pm$ 0.003 & 1.080 $\pm$ 0.007 & 0.847 $\pm$ 0.005 & 8.442 $\pm$ 0.066 & 5.930 $\pm$ 0.067 & 0.987 $\pm$ 0.013 & 0.567 $\pm$ 0.004 \\
\bottomrule
\end{tabular}
\end{adjustbox}
\caption{rMMD  per cell group for different models (\bestres{\textbf{best}}, \secondbest{$2^{nd}$ best}, and \thirdbest{$3^{rd}$ best} bolded) in the~\citet{kang} dataset.}
\label{tab:mmd_stats_kang}
\vspace{-7pt}
\end{table*}

Let $\mathbf{s}$ denote a subtype. We define $\hat{p}_{\mathbf{s},\mathbf{c}}$ as the empirical distribution of cells with subtype $\mathbf{s}$ under concept $\mathbf{c}$. The corresponding ground-truth counterfactual distribution is $\hat{p}_{\mathbf{s},\mathbf{c}'}$, i.e., the distribution of cells with the same subtype but different observable concept values. Since the subtype information is hidden from the model, it can be considered part of the exogenous noise $U$, making this construction a principled proxy for counterfactual evaluation.

Given a counterfactual predictor $f(\mathbf{x},\mathbf{c}')$, we require the distribution induced by applying $f$ to samples from $\hat{p}_{\mathbf{s},\mathbf{c}}$ to align with $\hat{p}_{\mathbf{s},\mathbf{c}'}$. We assess this alignment using the Maximum Mean Discrepancy ratio (rMMD):
\begin{equation}
\textrm{rMMD} = \frac{\textrm{MMD}((f_{\mathbf{c}'})_{\#}\hat{p}_{s,\mathbf{c}},\hat{p}_{s,\mathbf{c}'})}{\min_{\mathbf{c}}\textrm{MMD}(\hat{p}_{\mathbf{c}},\hat{p}_{s,\mathbf{c}'})}.
\end{equation}
Here, $(f_{\mathbf{c}'})\#\hat{p}_{\mathbf{s},\mathbf{c}}$ denotes the distribution induced by transforming $\hat{p}_{\mathbf{s},\mathbf{c}}$ with $f$, while $\hat{p}_{\mathbf{c}}$ is the empirical distribution of cells under concept $\mathbf{c}$. The denominator minimizes MMD over all observed concepts in the training data, normalizing for baseline misspecification common in single-cell perturbation modeling~\citep{wenteler2024perteval}. Intuitively, the numerator measures the discrepancy between predicted and observed counterfactuals at the subtype level, while the denominator reflects the similarity between the target population and its closest match in the training data. An rMMD below 1 indicates success, i.e., the model outperforms the trivial baseline of mapping to the most similar existing population. Analogously, we also evaluate the Frechet Inception Distance ratio and the Sinkhorn Divergence ratio, which are presented in Appendices~\ref{app:metrics} and~\ref{app:additional_benchmark}.

\textbf{Baselines.} \method~is compared against CBGM, a standard concept bottleneck model~\citep{cbgm}, as well as single-cell editing baselines: CINEMA-OT~\citep{CinemaOT}, biolord~\citep{biolord}, scGen~\citep{scgen2019}, and a standard Conditional Flow Matching model, representing state-of-the-art generative modeling in single-cell analysis (akin to the framework used in CellFlow~\citet{cellflow}). We refer to this baseline as \emph{Vanilla-FM}. A standard Conditional VAE (CVAE), a common approach in batch correction tasks which achieve disentanglement by implicitly regressing out conditioning variables via the latent, is also included. Benchmarks are further extended to \method~combined with flow matching models (\method-FM). To verify that \method-FM's performance stems from the structured \method~latent space rather than solely the generative capabilities of Flow Matching, CVAE-FM and biolord-FM are also evaluated. Finally, for Vanilla-FM, CVAE-FM and \method-FM, two counterfactual strategies—\emph{edit} and \emph{decode}—are evaluated as detailed in Section~\ref{sub_section:scCBMFM}. While all baselines successfully regenerated the \citet{kang} dataset, scCBGM and scCBGM-FM demonstrated superior accuracy in reconstructing cell identity (Figure \ref{fig:decode_accuracy}). Complete data processing and model implementation details are provided in Appendix~\ref{app:experiment_details}. 

\subsection{Benchmarking single-cell editing}

\textbf{Counter-factual modeling predicts cellular response to perturbation}
We first benchmarked \method~on predicting treatment responses across 14 immune cell subtypes using the~\citet{kang} dataset, which contains Peripheral Blood Mononuclear Cells (PBMCs) with and without IFN-$\beta$ stimulation. For each subtype, the stimulation response was treated as an independent benchmark across all models. The model’s concepts included 7 high-level cell types and a binary indicator for IFN-$\beta$ stimulation.
Table~\ref{tab:mmd_stats_kang} shows that \method-based edits accurately predict the stimulation response while preserving subtype identity,  significantly outperforming existing conditional generation methods in counterfactual accuracy on 5 out of the 7 experiments. Additional metrics can be found in Appendix~\ref{app:additional_benchmark}. While Table~\ref{tab:mmd_stats_kang} reports average rMMD across all CD4 T cell subtypes, Figure \ref{fig:kang_stimulation} illustrates this zero-shot generalization capability for a specific subtype. Here, we held out stimulated Naive CD4 T cells during training and evaluated our model's ability to predict their response from control Naive CD4 T cells. We find that scCBGM accurately predicts the held-out stimulated Naive CD4 T cells, demonstrating superior zero-shot generalization compared to baseline methods. 

\begin{table*}[!ht]
\centering
\begin{adjustbox}{width=\tabwidth}
\begin{tabular}{lcccc||ccc}
\toprule
 & \multicolumn{4}{c}{\citet{cui2024dictionary} dataset} & \multicolumn{3}{c}{\citet{nault2023single} dataset}\\
 \hline
\textbf{Model} & \makecell{\textbf{T cells}\\\textbf{(Gamma-delta)}} & \makecell{\textbf{T cells}\\\textbf{(CD4)}}  &  \makecell{\textbf{T cells}\\\textbf{(CD8)}} &\makecell{\textbf{Dendritic}\\\textbf{(Langerhans)}} & \makecell{\textbf{Hepatocytes} \\ \textbf{(Centrilobular)}} & \textbf{Stellate cell} &  \makecell{\textbf{Hepatocytes} \\ \textbf{(Periportal)}} \\
\midrule
scCBGM & 0.304 $\pm$ 0.024 & 0.062 $\pm$ 0.014  & 0.295 $\pm$ 0.038 & \thirdbest{0.110 $\pm$ 0.007} &  0.627 $\pm$ 0.011  & 0.895 $\pm$ 0.548  & \thirdbest{0.752 $\pm$ 0.046} \\
scCBGM-FM (decode)& \bestres{\textbf{0.238 $\pm$ 0.032}} & 0.046 $\pm$ 0.011 & 0.259 $\pm$ 0.042  & \secondbest{0.104 $\pm$ 0.005}  & \secondbest{0.608 $\pm$ 0.029}  & \thirdbest{0.861 $\pm$ 0.440}  & \secondbest{0.719 $\pm$ 0.056} \\
scCBGM-FM (edit) & \secondbest{0.278 $\pm$ 0.017} & \bestres{\textbf{0.034 $\pm$ 0.012}} & 0.254 $\pm$ 0.022 & \bestres{\textbf{0.097 $\pm$ 0.003}}  & \thirdbest{0.617 $\pm$ 0.001} & \bestres{\textbf{0.844 $\pm$ 0.451}}  & \bestres{\textbf{0.708 $\pm$ 0.056}} \\
\midrule
CBGM & 0.824 $\pm$ 0.151 & 0.401 $\pm$ 0.113 & 0.389 $\pm$ 0.054 & 0.139 $\pm$ 0.009 & 1.470 $\pm$ 0.695 & 0.979 $\pm$ 0.014 & 1.343 $\pm$ 0.466 \\
\midrule
\cellflow (decode) & 1.822 $\pm$ 0.264 & 1.647 $\pm$ 0.252 & 1.721 $\pm$ 0.232 & 1.609 $\pm$ 0.110  & 1.458 $\pm$ 0.690  & 10.481 $\pm$ 7.058  & 1.185 $\pm$ 0.484 \\
\cellflow (edit) & 0.540 $\pm$ 0.116 & 0.164 $\pm$ 0.037 & \bestres{\textbf{0.210 $\pm$ 0.062}}  & 0.368 $\pm$ 0.058 & \bestres{\textbf{0.442 $\pm$ 0.088}}  & 7.235 $\pm$ 3.239  & 1.030 $\pm$ 0.317 \\
\midrule
biolord & 2.848 $\pm$ 0.005 & 2.308 $\pm$ 0.006 & 1.934 $\pm$ 0.004 & 2.751 $\pm$ 0.003 & /  & 44.340 $\pm$ 32.531  & 4.707 $\pm$ 2.031 \\
CINEMA-OT & 2.033 $\pm$ 0.004 & 1.532 $\pm$ 0.003  & 1.244 $\pm$ 0.004 &  1.504 $\pm$ 0.002  & 4.667 $\pm$ 1.598  & 45.570 $\pm$ 33.289  & 5.295 $\pm$ 1.363 \\
scGen & 2.069 $\pm$ 0.074 & 2.849 $\pm$ 0.134 & 1.506 $\pm$ 0.083 &  0.388 $\pm$ 0.013 & 2.214 $\pm$ 0.642 & 12.067 $\pm$ 9.021 & 2.389 $\pm$ 0.798 \\
CVAE & 0.300 $\pm$ 0.007 & 0.049 $\pm$ 0.006 & 0.238 $\pm$ 0.012 &  0.127 $\pm$ 0.003 & 1.842 $\pm$ 0.809  & 1.349 $\pm$ 0.281 & 1.554 $\pm$ 0.433 \\
CVAE-FM (decode) & 0.284 $\pm$ 0.006 & \thirdbest{0.043 $\pm$ 0.006} & \secondbest{0.224 $\pm$ 0.006} &  0.121 $\pm$ 0.003 & 1.333 $\pm$ 0.585  & 0.882 $\pm$ 0.117  & 1.143 $\pm$ 0.263 \\
CVAE-FM (edit) & \thirdbest{0.292 $\pm$ 0.003} & \secondbest{0.036 $\pm$ 0.010} & \thirdbest{0.232 $\pm$ 0.011} & 0.117 $\pm$ 0.001 & 1.337 $\pm$ 0.569 & \secondbest{0.853 $\pm$ 0.171} & 1.131 $\pm$ 0.253 \\
\bottomrule
\end{tabular}
\end{adjustbox}
\caption{rMMD per cell group for different models (\bestres{\textbf{best}}, \secondbest{$2^{nd}$ best}, and \thirdbest{$3^{rd}$best} bolded)  in the~\citet{cui2024dictionary} and~\citet{nault2023single}.}
\label{tab:mmd_stats_cui}
\end{table*}

The~\citet{cui2024dictionary} immune dictionary dataset measures the response of 17 immune cell subtypes across an expansive panel of 86 cytokine-based stimulations. This setting presents a more challenging scenario, requiring the model to learn distinct responses for each perturbation applied to different cell types. Concepts here included both broad cell type and stimulation identity. The counterfactual task was to predict the response of control cells to a particular stimulation, which the model had not encountered in combination with that specific cell subtype during training. We restricted our evaluation to four cell-type–cytokine pairs previously identified by \citep{cui2024dictionary} to induce a significant transcriptional shifts relative to controls. As reported in Table~\ref{tab:mmd_stats_cui} (left), \method~successfully predicts accurate cellular responses for these unseen combinations. 

\textbf{Modeling continuous dose responses and hypothesis generation} To assess whether \method~can model continuous responses, we used the~\citet{nault2023single} dataset, which measures liver responses to varying dosages of TCDD, a toxic dioxin compound. In this setting, concepts consisted of broad cell type and treatment dosage. The counterfactual task was to predict gene expression profiles for intermediate dosages, given that high and low dosages of the test subtypes were withheld during training. As shown in Table~\ref{tab:mmd_stats_cui} (right), \method~generally outperforms other methods, thereby helping to fill experimental gaps while supporting mechanistic hypothesis generation. Results on all available cell types are given in Appendix~\ref{app:liver_results}.

\begin{figure}[!htbp]
\centering
\includegraphics[width=1\columnwidth]
{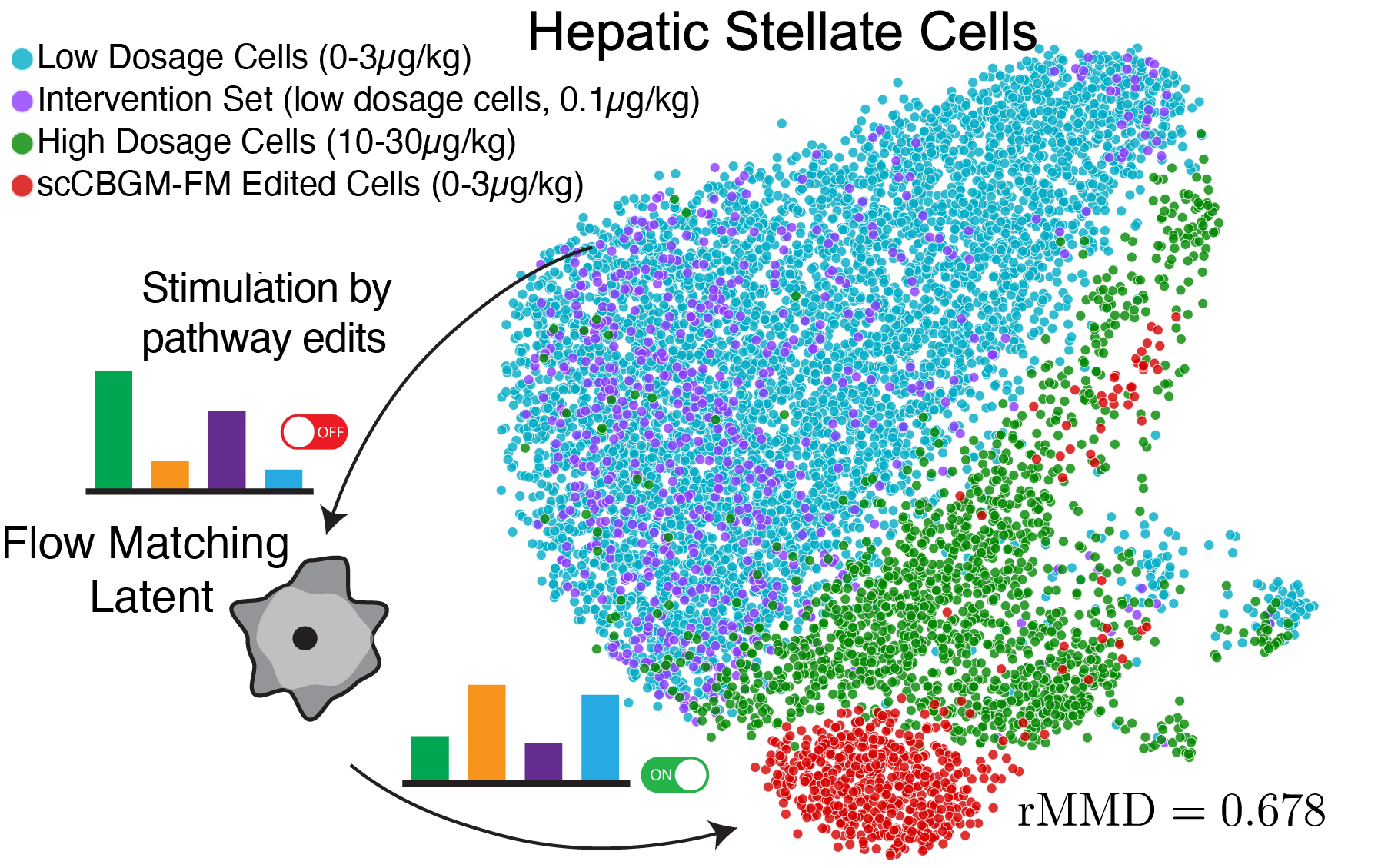}
\caption{\method~enables interpretable control of cells to enhance response to stimulation. Stellate cells with low dosage of TCDD showed limited treatment response, while cells with high dosage showed a clear response. By editing control cells' pathway activity profile, cells become more sensitive to TCDD, showing a treatment response similar to cells with higher TCDD dosage.}
\label{fig:pathway_pertubation}

\end{figure}

\noindent \textbf{Benchmarking on synthetic data} We benchmarked methods on the three synthetic datasets from Section~\ref{sec:scCBGM_vs_CBGM}, where access to true counterfactuals enables fine-grained cell-level rather than population-level evaluation. Flow matching models showed clear advantage, with scCBGM-FM outperforming all others across datasets  (Appendix, Table~\ref{tab:syn_benchmark_sota}).

\textbf{\method~boosts performance of flow matching models.} As discussed in Section~\ref{sub_section:scCBMFM}, \method~can be combined with generative models to harness their established generative capabilities. Tables~\ref{tab:mmd_stats_kang} and~\ref{tab:mmd_stats_cui} compare \method-FM directly with \cellflow~(\emph{i.e} our implementation of CellFlow). Whereas \cellflow~conditions the vector field only on observed concepts, allowing conditional but not fully counterfactual distributions, \method-FM conditions on the \method~embedding, enabling richer counterfactual inference. In our experiments, \method-FM improved over or matches vanilla-FM in 12/14 of the cell types across all datasets. 
We also note that the \emph{edit} procedure is generally much more effective than the \emph{decode} procedure. More broadly, the two components play complementary roles: base \method~provides the structured, disentangled latent space that enables interpretable concept-level editing, whereas \method-FM (edit) leverages this representation for higher-fidelity generation. The reported gains therefore reflect the framework as a whole rather than either component in isolation.

\subsection{Case study: controlled single-cell editing for enhanced drug response}


Beyond cell-editing, we show that \method~offers a framework for guiding cells toward a desired biological state in an interpretable manner. In the original~\citet{nault2023single} dataset, stellate cells exhibited only a moderate response to TCDD. With \method, we can identify biological concepts that appear to be crucial for mediating the treatment effect in this specific cell type. To this end, we trained the model with gene regulation pathways extracted from a reference signaling pathway database~\citep{badia2022decoupler, schubert2018perturbation}, represented as soft concepts (\emph{i.e.}, continuous values). By comparing control and treated stellate cells, we uncovered key pathways that were differentially regulated. We then applied \method~to edit control cells by jointly activating these pathways while simulating TCDD treatment. As illustrated in Figures ~\ref{fig:pathway_pertubation}, ~\ref{fig:pathway_marker_genes} \& ~\ref{fig:umap_gene_trends}, the resulting cell population after editing appeared similar to the populations which responded to the treatment, in both rMMD and gene-expression changes. That is despite the model not being trained on any treatment data, suggesting a potential mechanistic route for mediating treatment response. These results highlight the utility of \method~for testing mechanistic hypotheses and exploring strategies to overcome cellular resistance.



\section{Discussion and conclusion}

We introduced \method, a framework for interpretable and precise single-cell counterfactual editing. Our experiments demonstrate improved accuracy, flexibility, and robustness over existing methods, including under noisy annotations.

\textbf{Limitations.} Our evaluation approach involves assessing counterfactual methods on real single-cell data: true cell-specific counterfactuals cannot be observed, as this would require measuring the same cell under different conditions simultaneously. For real data, we therefore rely on population-level metrics as the best available approximation for validating cell-level edits. While this approach is standard in the field, it introduces a potential pitfall: population-level accuracy does not guarantee that individual cell predictions preserve biological realism or cell identity at the single-cell level. Our synthetic data evaluation addresses this limitation by providing access to ground-truth counterfactuals under controlled conditions. A further limitation concerns generalization scope: scCBGM generalizes to unseen combinations of known concepts but not to entirely new concepts absent from training, as each concept must be explicitly learned. However, this annotation bottleneck is partially mitigated by the fact that concepts can be learned in a semi-supervised manner, reducing the need for fully labeled data. Methods relying on a pretrained concept encoder (e.g., biolord) can in principle accept novel concept embeddings at inference, reflecting a design tradeoff rather than a strict advantage. Extending scCBGM to novel concepts, e.g. via Label-Free CBMs \cite{oikarinen2023label}, is an interesting direction for future work.




\section*{Impact Statement}
This work enables researchers to navigate combinatorial spaces of cellular responses that are experimentally infeasible to map, accelerating therapeutic design. However, high-fidelity generative models carry ethical risks, such as the potential for fabricating biological data. Furthermore, an over-reliance on simulated counterfactuals without rigorous experimental validation could lead to false therapeutic leads, wasting critical resources.

\bibliography{references}
\bibliographystyle{icml2026}

\newpage
\appendix
\onecolumn



\part*{Appendix}
\section{Additional details on Singe-Cell Concept Bottleneck Generative Models}
\subsection{A very short primer on counterfactuals}
\label{app:counterfactuals}
We posit a structural causal model (SCM) 
$M = (G, F, P(U))$ with a directed acyclic graph $G$, 
endogenous variables $V$, exogenous noise $U$, 
and assignments $F = \{ v_i \leftarrow f_i(\mathrm{pa}_i, u_i) \}$. 

Intuitively, a counterfactual query asks: 
\emph{What would an observed variable $X$ have been, 
had we set a parent variable to a different value, 
while keeping everything else (including latent background $U$) fixed?}

Formally, counterfactuals are generated by applying the Pearlian
$\mathrm{do}$-operator within an SCM, and their values 
depend on both the structural assignments and the realization of $U$.

\paragraph{Setup in this paper.}
Let $X$ denote a cell’s RNA-seq counts and $C$ a concept variable 
we wish to intervene on. Let $Z$ collect other endogenous covariates 
(e.g. other concepts). Then, we can abstract the data generating process as:
\begin{align}
Z &\leftarrow f_Z(\mathrm{pa}_Z, U_Z), \\
C &\leftarrow f_C(\mathrm{pa}_C, U_C), \\
X &\leftarrow f_X(Z, C, U_X).
\end{align}

\paragraph{Counterfactual of $X$ under an action on $C$.}
For an observed cell with concept $C=c$, $X_{C=c} = x$, we write the counterfactual value of $X$ if $C$ had been set to $c'$ as
\begin{align}
X^{\mathrm{cf}} \;=\; X_{\mathrm{do}(C:=c')}(U),
\end{align}
defined by the standard abduction--action--prediction steps:
\begin{enumerate}
    \item \textbf{Abduction:} infer the posterior 
    $p(u,z \mid X_C = x)$ over latent variables;
    \item \textbf{Action:} replace 
    $C \leftarrow f_C(\cdot)$ by $C := c'$;
    \item \textbf{Prediction:} compute 
    $X^{\mathrm{cf}} = f_X(z, c', u_X)$ 
    with $(z,u)$ drawn from the abducted posterior.
\end{enumerate}

Given $X_{C=c} = x$, the probability distribution of the counterfactual is computed as follows:
\begin{align}
p(x^{\mathrm{cf}} \mid X_{C=c}=x,\mathrm{do}(C:=c')) 
= \int \mathbb{I}[ f_X(z,c',u) = x^{\mathrm{cf}}], 
   p(z,u \mid X_{C=c}=x)\, dz\, du .
\end{align}

\paragraph{Ground-truth counterfactuals obtained from our synthetic data generating process}

When the data-generating process is fully known and simulatable, 
one can hold the exogenous noise $U$ fixed and generate two outcomes 
with different concept values $C=c$ and $C=c'$. This yields
\begin{align}
x &= f_X(z, c, u_X), \\
x^{\mathrm{cf}} &= f_X(z, c', u_X),
\end{align}
for the same realization of $(z,u_X)$. Such pairs $(x, x^{\mathrm{cf}})$ 
constitute \emph{true counterfactual examples} and can serve as ground 
truth for evaluating predictive methods.


\color{black}

\section{Consistency of our counterfactual predictor}
\label{app:consistency}
In this Section, we prove the consistency of our counterfactual estimator. That is, in the infinite data regime, our estimator converges to the expected counterfactual outcome defined in \Eqref{eq:target}, as the following results shows.

\begin{proposition}[Consistency of \method counterfactuals]

Let $\hat\theta_n \in \arg\min_{\theta} \widehat L_n(\theta)$ be any empirical minimizer of the training objective in \Eqref{eq:full_loss} on $n$ i.i.d.\ samples ($\mathbf{x},\mathbf{c}$). For $\mathbf{x}$ fixed and any counterfactual edit $\mathbf{c}'$ in the model's support, we define the learned counterfactual-mean predictor as
\begin{equation}
    g_{\theta}(\mathbf{x},\mathbf{c}') := \mathbb{E}_{\mathbf{z}\sim q_{\theta}(\mathbf{z}\mid \mathbf{x})} \left[\,D_{\theta} \big([f_{u,\theta}(\mathbf{z}),\mathbf{c}']\big)\right].
\end{equation}
Assuming:
\begin{enumerate}[label=(A\arabic*)]
\item (\emph{Data-generating SCM \& exogeneity}) $U \perp C$ with $X = f_X(C,U)$ and the target $\mu_{x,c'} := \mathbb{E}[f_X(\mathbf{c}',U)\mid X=\mathbf{x}]$ well-defined (exists and is finite).
\item (\emph{Realizability}) There exists $\theta^\star$ such that the pushforward $\big(D_{\theta^\star}([f_{u,\theta^\star}(\mathbf{z}),\mathbf{c}])\big)_{\#}q_{\theta^\star}(\mathbf{z}\mid \mathbf{x})  \equiv f_X(\mathbf{c},U)_{\#}P(U|X=\mathbf{x})$, for all $\mathbf{c}\sim P(C)$.
\item (\emph{Regularity}) The parameter space is compact, The neural networks are bounded and Lipschitz in parameters and inputs. $\widehat L_n$ admits near-global minimizers.
\item (\emph{Uniform convergence}) $\sup_{\theta}\big|\widehat L_n(\theta) - L(\theta)\big|\xrightarrow{p}0$, where $L(\theta) = \mathbb{E}_{X,C}[\widehat L_n(\theta)]$
\item (\emph{Support/positivity}) $\mathbf{c}'$ lies in the support of the training distribution of concepts (and in the support of the model).
\end{enumerate}

We have:
\begin{equation}
    g_{\hat\theta_n}(\mathbf{x},\mathbf{c}') \xrightarrow{p} \mu_{x,c'}.
\end{equation}
That is, the learned counterfactual \emph{mean} converges in probability to the Pearlian counterfactual mean for the same cell $\mathbf{x}$ under the edit $\mathbf{c}'$.
\end{proposition}

\begin{proof}[Proof Sketch] The proof follows an empirical minimization argument combined with the abduction-action-prediction formalism of counterfactual inference.
\begin{enumerate}[label=(S\arabic*)]
    \item From the consistency of M estimators~\citep{van2000asymptotic}, and assumption (A4), any sequence of empirical minimizers $\hat\theta_n$ converges in probability to $\theta^\star := \arg\min_\theta L(\theta)$ (or to a set whose $g_\theta$-image equals $g_{\theta^\star}$).
    \item At $\theta^\star$, (A2) leads $q_{\theta^\star}(\mathbf{z}\mid \mathbf{x})=p(\mathbf{z}\mid \mathbf{x})$ in the model’s latent space and ensures $D_{\theta^\star}([f_{u,\theta^\star}(\mathbf{z}),\mathbf{c}])$ implements the true structural map $f_X(\mathbf{c},U)$ when $\mathbf{z}$ is drawn from that posterior.
    \item Let $T(\theta) := g_{\theta}(\mathbf{x},\mathbf{c}')$. By (A3), we have that $T$ is continuous at $\theta^\star$, so $T(\hat\theta_n)\xrightarrow{p}T(\theta^\star)$.
    \item By (A2), the pushforward $\big(D_{\theta^\star}([f_{u,\theta^\star}(\mathbf{z}),\mathbf{c}])\big)_{\#}q_{\theta^\star}(\mathbf{z}\mid \mathbf{x})$ matches $f_X(\mathbf{c},U)_{\#}P(U|X=\mathbf{x})$. Hence,
\begin{equation}
T(\theta^\star) = \int D_{\theta^\star} \big([\,f_{u,\theta^\star}(\mathbf{z}),\mathbf{c}']\big)\,q_{\theta^\star}(\mathbf{z}\mid \mathbf{x})\,dz
\;=\; \int f_X(\mathbf{c}',\mathbf{u})\,p(\mathbf{u}\mid \mathbf{x})\,du
\;=\; \mu_{x,c'}.
\end{equation}
\end{enumerate}

Combining (S3) and (S4), we have

\begin{equation}
    g_{\hat\theta_n}(\mathbf{x},\mathbf{c}') \xrightarrow{p} \mu_{x,c'}.
\end{equation}

\end{proof}

The realizability assumption is the crux of the above result. The following proposition shows that minimizing our loss in~\Eqref{eq:full_loss} leads to the correct pushforward.

\begin{proposition}[Realizability of the pushforward]

Assume:
\begin{enumerate}[label=(A\arabic*)]
\item \textbf{SCM and exogeneity.} The data follow $X=f_X(C,U)$ with $U \perp C$, so that $f_X(\mathbf{c},U)_{\#}P(U|X=\mathbf{x})$ is well-defined for all $\mathbf{c}'$.
\item \textbf{Proper reconstruction and rich classes.} The reconstruction term in~\Eqref{eq:full_loss} is a proper conditional risk for the law of $X$ given the decoder input (e.g., a correctly specified exponential-family likelihood), and the encoder/decoder classes are rich enough to realize the data distribution. During training, the decoder is conditioned on the \emph{true} $c$.
\item \textbf{Concept supervision.} The concept head is consistent for predicting $\mathbf{c}$ from $\mathbf{x}$ under the chosen loss, so that at the population optimum we predict $\mathbf{c}$ correctly.
\end{enumerate}
Let $L(\theta)$ be the population objective corresponding to ~\Eqref{eq:full_loss}, and let $\theta^\star \in \arg\min_\theta L(\theta)$. Then for almost every $\mathbf{x}$ and every $\mathbf{c}'$,
\begin{equation}
\big(D_{\theta^\star}([f_{u,\theta^\star}(Z),\mathbf{c}'])\big)\# q_{\theta^\star}(Z\mid X=\mathbf{x})
=f_X(\mathbf{c}',U)_{\#}P(U|X=\mathbf{x}),
\end{equation}
i.e., the learned pushforward (decoder applied to the residual code with $\mathbf{c}'$) \emph{matches the true counterfactual kernel in law}. Equivalently, for any bounded measurable $\varphi$,
\begin{equation}
\mathbb{E}\left[\varphi\left(D_{\theta^\star}([f_{u,\theta^\star}(Z), \mathbf{c}'])\right)| X=x\right]
=
\mathbb{E}\left[\varphi\left(f_X(\mathbf{c}',U)\right)| X=\mathbf{x}\right].
\end{equation}
\end{proposition}

\begin{proof}[Proof sketch]

\begin{enumerate}[label=(S\arabic*)]

\item \emph{(Optimal conditional reconstruction)}
From A2, because $-\log p_\theta(x\mid \mathbf{z},\mathbf{c})$ is a strictly proper conditional scoring rule and the model class is rich, the population minimizer must realize the true conditional law of $X$ given $(Z,C)$:
\begin{equation}
p_{\theta^\star}(\mathbf{x}\mid Z,C) \;=\; p(\mathbf{x}\mid Z,C) \quad \text{a.s.}
\end{equation}
Equivalently, there exists a measurable decoder $D_{\theta^\star}$ such that $X = D_{\theta^\star}([f_{u,\theta^\star}(Z),C])$ almost surely under the joint law induced by $(X,C)$ and $Z\!\sim\!q_{\theta^\star}(\mathbf{z}\mid X)$.

\item \emph{(Coupling with the structural SCM)}
Combine with Step~1: for a.s.\ $(X,C)$ we can couple $(Z,U)$ given $X{=}x$ so that
\[
D_{\theta^\star}([f_{u,\theta^\star}(Z),C]) \;=\; f_X(C,U)\quad \text{a.s.}
\]
This is just matching two a.s.\ representations of the same $X$ under an appropriate coupling.

\item \emph{(Kernel equality in law)}
Condition on $X=\mathbf{x}$ and replace $C$ by any counterfactual value $\mathbf{c}'$. By the equality in Step~2 and the coupling of $(Z,U)\mid X=\mathbf{x}$, pushing $q_{\theta^\star}(Z\mid X=\mathbf{x})$ through $\mathbf{u}\mapsto D_{\theta^\star}([\mathbf{u},\mathbf{c}'])$ yields the same law as pushing $P(U\mid X=\mathbf{x})$ through $u\mapsto f_X(\mathbf{c}',\mathbf{u})$:
\begin{equation}
\big(D_{\theta^\star}([f_{u,\theta^\star}(Z),\mathbf{c}'])\big)\# q_{\theta^\star}(Z\mid X=\mathbf{x})
=f_X(\mathbf{c}',U)_{\#}P(U|X=\mathbf{x}), \quad \text{a.s.}
\end{equation}

\end{enumerate}

This is precisely the equality of the learned pushforward with the target counterfactual kernel. 

\end{proof}

Note that Step~3 matches \emph{measures on $X$} (the object we use for counterfactual prediction); it is invariant to internal reparameterizations of the residual code. The cross-cov penalty and independent training are used to \emph{select} residuals $\mathbf{u}$ that do not leak concept information, improving identifiability and stability.

\subsection{Decoupling of known and unknown concepts through cross-covariance penalty}
\label{app:decoupling_proof}

In Section~\ref{sub_section:losses} we introduced the cross-covariance loss $\mathcal{L}_{\text{cc}}$. Here we provide a theoretical proof for how the (predicted) known concepts $\hat{\mathbf{c}}$ and unknown factors $\mathbf{u}$ become decoupled through the proposed loss. 
Let: 
\begin{equation}
\mathbf{u} = \phi(A\mathbf{z} + \mathbf{a}^0), \qquad \hat{\mathbf{c}} = \psi(B\mathbf{z} + \mathbf{b}^0)
\end{equation}
where
\begin{equation}
A \in \mathbb{R}^{d_u \times d_z}, \qquad B \in \mathbb{R}^{K \times d_z}
\end{equation}
are weight matrices, and $\phi$, $\psi$ are element-wise, piecewise $C^1$ monotonically increasing activation functions (e.g., ReLU or sigmoid). The vectors $\mathbf{a}^0$, $\mathbf{b}^0$ are bias terms.
We aim to show that using a cross-covariance loss
\begin{equation}
\mathcal{L}_{\text{cc}} = \left\| \textrm{Cov}(\mathbf{u}, \hat{\mathbf{c}}) \right\|_F^2
\end{equation}
encourages decoupling between the representations $\mathbf{u}$ and $\hat{\mathbf{c}}$.
Define the Jacobians:
\begin{equation}
J_{\mathbf{u}}(\mathbf{z}) = \underbrace{\mathrm{diag}(\phi'(A\mathbf{z} + \mathbf{a}^0))}_{D_{\mathbf{u}}}A, \qquad J_{\hat{\mathbf{c}}}(\mathbf{z}) = \underbrace{\mathrm{diag}(\psi'(B\mathbf{z} + \mathbf{b}^0))}_{D_{\hat{\mathbf{c}}}}B
\end{equation}
We define the centered variables $\tilde{\mathbf{u}}$ and $\tilde{\hat{\mathbf{c}}}$ as:
\begin{equation}
\begin{array}{l}
    \tilde{\mathbf{u}}(\mathbf{z})  = \mathbf{u}(\mathbf{z}) - \mathbb{E}[\mathbf{u}(\mathbf{z})]\\
    \tilde{\hat{\mathbf{c}}}(\mathbf{z}) = \hat{\mathbf{c}}(\mathbf{z}) - \mathbb{E}[\hat{\mathbf{c}}(\mathbf{z})] \\
\end{array}
\end{equation}
If we assume that
\begin{equation}
    \mathbb{E}[\mathbf{u}(\mathbf{z})] \approx \mathbf{u}(\mathbb{E}[\mathbf{z}]) \quad , \quad    \mathbb{E}[\hat{\mathbf{c}}(\mathbf{z})] \approx \hat{\mathbf{c}}(\mathbb{E}[\mathbf{z}]) 
\end{equation}
\textit{Note: This assumption is true for all regions of the ReLU activation function where leakage is possible (positive domain).}

Next, using a first order Taylor approximation we can linearize around $\boldsymbol{\mu}_{\mathbf{z}}$. Here, we define $\boldsymbol{\mu}_{\mathbf{x}} := \mathbb{E}[\mathbf{x}]$.
\begin{equation}
        \tilde{\mathbf{u}}(\mathbf{z})  \approx \mathbf{u}(\mathbf{z}) - \underbrace{\mathbf{u}(\boldsymbol{\mu}_{\mathbf{z}})}_{\approx \boldsymbol{\mu}_{\mathbf{u}}} \approx J_{\mathbf{u}}(\boldsymbol{\mu}_{\mathbf{z}})(\mathbf{z} - \boldsymbol{\mu}_{\mathbf{z}}) \\
\end{equation}
\begin{equation}
        \tilde{\hat{\mathbf{c}}}(\mathbf{z}) \approx \hat{\mathbf{c}}(\mathbf{z}) - \underbrace{\hat{\mathbf{c}}(\boldsymbol{\mu}_{\mathbf{z}})}_{\approx \boldsymbol{\mu}_{\hat{\mathbf{c}}}} \approx J_{\hat{\mathbf{c}}}(\boldsymbol{\mu}_{\mathbf{z}})(\mathbf{z} - \boldsymbol{\mu}_{\mathbf{z}}) \\
\end{equation}
With this we can express the covariance $\textrm{Cov}(\mathbf{u},\hat{\mathbf{c}})$ as:
\begin{equation}
    \textrm{Cov}(\mathbf{u},\hat{\mathbf{c}}) = \mathbb{E}[\tilde{\mathbf{u}}\tilde{\hat{\mathbf{c}}}^T] \approx \mathbb{E}[J_{\mathbf{u}}(\boldsymbol{\mu}_{\mathbf{z}})(\mathbf{z} -\boldsymbol{\mu}_{\mathbf{z}})(\mathbf{z}-\boldsymbol{\mu}_{\mathbf{z}})^TJ_{\hat{\mathbf{c}}}(\boldsymbol{\mu}_{\mathbf{z}})^T] 
\end{equation}
Which can be written as:
\begin{equation}
        \textrm{Cov}(\mathbf{u},\hat{\mathbf{c}}) \approx J_{\mathbf{u}}(\boldsymbol{\mu}_{\mathbf{z}})\Sigma_{\mathbf{z}} J_{\hat{\mathbf{c}}}(\boldsymbol{\mu}_{\mathbf{z}})^T, \quad \Sigma_{\mathbf{z}} := \textrm{Cov}(\mathbf{z})
\end{equation}
And thus minimizing the covariance means that:
\begin{equation}
\textrm{Cov}(\mathbf{u}, \hat{\mathbf{c}}) \to 0 \quad \Leftrightarrow \quad J_{\mathbf{u}}(\boldsymbol{\mu}_{\mathbf{z}})\Sigma_{\mathbf{z}}J_{\hat{\mathbf{c}}}(\boldsymbol{\mu}_{\mathbf{z}})^T \to 0
\end{equation}
Let us now look at what happens when we add some infinitesimal change $\delta \mathbf{z} \sim N(\mathbf{0},\Sigma_{\mathbf{z}})$ around $\boldsymbol{\mu}_{\mathbf{z}}$:
\begin{equation}
    \begin{array}{l}
        \delta \mathbf{u} = J_{\mathbf{u}}(\boldsymbol{\mu}_{\mathbf{z}})\delta \mathbf{z}  \\
        \delta \hat{\mathbf{c}} = J_{\hat{\mathbf{c}}}(\boldsymbol{\mu}_{\mathbf{z}}) \delta \mathbf{z} \\ 
    \end{array}
\end{equation}
Then:
\begin{equation}
    \textrm{Cov}(\delta \mathbf{u}, \delta \hat{\mathbf{c}}) = \mathbb{E}[\delta \mathbf{u} \delta \hat{\mathbf{c}}^T] = \mathbb{E}[J_{\mathbf{u}}(\boldsymbol{\mu}_{\mathbf{z}})\delta \mathbf{z} \delta \mathbf{z}^TJ_{\hat{\mathbf{c}}}(\boldsymbol{\mu}_{\mathbf{z}})]
\end{equation} 
Which becomes:
\begin{equation}
 \mathbb{E}[J_{\mathbf{u}}(\boldsymbol{\mu}_{\mathbf{z}})\delta \mathbf{z} \delta \mathbf{z}^TJ_{\hat{\mathbf{c}}}(\boldsymbol{\mu}_{\mathbf{z}})] = J_{\mathbf{u}}(\boldsymbol{\mu}_{\mathbf{z}})\Sigma_{\mathbf{z}}J_{\hat{\mathbf{c}}}(\boldsymbol{\mu}_{\mathbf{z}}) =0
\end{equation}
Hence, there's no first order covariance in the shift of $\mathbf{u}$ and $\hat{\mathbf{c}}$ upon changes to $\mathbf{z}$ along the direction of variance in the data. This concludes the proof that minimizing $\mathcal{L}_{\text{cc}}$ enforces local decoupling between $\mathbf{u}$ and $\hat{\mathbf{c}}$.

\color{black}

\section{Synthetic data}
\label{app:synth}

\subsection{Generative Process}
\label{app:synth:process}

Our synthetic data generation proccess builds on an overdispersed hierarchical Poisson distribution. For each cell $i$ and gene $j$, the logarithm of the mean expression is modeled as the sum of several structured components. Because these factors are combined in log-space, their effects translate into multiplicative interactions on the original scale. Specifically, the log mean expression $(\lambda_{ij})$ is defined as:  

\begin{equation}
\log(\lambda_{ij}) = w_j + l_j + \gamma_{b_i j} + \omega_{u_i j} + \tau_{t_i j} + \sum_{k=1}^{K} v_{ij}c_{kj} + \varepsilon_{ij}
\end{equation}

In this formulation, gene expression is decomposed into baseline levels ($w_j$), cell-specific scaling ($l_i$), batch effects ($\gamma_{b_i j}$), cell-type variation ($\omega_{u_i j}$), and tissue effects ($\tau_{t_i j}$). In addition, latent concept influences are captured by $\sum_{k=1}^{K} v_{ik}c_{kj}$, where $v_{ik}$ denotes the activation of concept $k$ in cell $i$, and $c_{kj}$ its effect on gene $j$. Residual variation is represented by $\varepsilon_{ij}$.

Gene expression counts are then sampled via the inverse transform method:  

\begin{equation}
x_{ij} = Q(\delta_{ij}; \lambda_{ij}) = \inf \{ k \in \mathbb{N} : F_{\mathrm{Poi}}(k ; \lambda_{ij}) > \delta_{ij} \}, \quad \delta_{ij} \sim \mathcal{U}(0,1)
\end{equation}

which is equivalent to drawing directly from the Poisson distribution:  

\begin{equation}
x_{ij} \sim \mathrm{Poi}(\lambda_{ij})
\end{equation}

Here, $F_{\mathrm{Poi}}$ denotes the cumulative distribution function (CDF) of the Poisson distribution, and $Q$ its percent point function (PPF). 

By explicitly modeling the exogenous noise ($\delta_{ij}$), the framework enables the generation of true cell-level counterfactuals following concept-level interventions. To produce the counterfactual $x_{ij}'$, where we change concepts from $c$ to $c'$, we simply do:
\begin{equation}
\label{eq:synth_counterfactual}
\begin{array}{c}
\log(\lambda_{ij}') = w_j + l_i + \gamma_{b_i j} + \omega_{u_i j} + \tau_{t_i j} + \sum\limits_{k=1}^{K} v_{ik}\boxed{c'_{kj}} + \varepsilon_{ij} \\
 x_{ij}' = Q(\delta_{ij}; \lambda_{ij}')
\end{array}
\end{equation}
\\
\\
The factors in the model are distributed as follows:
\[
\begin{array}{rr} 
l_i \sim \mathit{Norm}(0,\sigma_{b_i l}) & (\textrm{library size})\\
\sigma_{b} \sim \mathit{Unif}(r_l,s_l) & (\textrm{library size std})\\
w_j \sim \mathit{Unif}(r_w,s_w) & (\textrm{baseline expression})\\
\varepsilon_{ij} \sim \mathit{Norm}(0,\sigma_\varepsilon) & (\textrm{noise})\\
\gamma_{bj} \sim \mathit{Norm}(0,\mathbf{\sigma_{b}}) & (\textrm{batch effect})\\
\tau_{tj} \sim \mathit{Norm}(0,\mathbf{\sigma_{t}}) & (\textrm{tissue effect})\\
\omega_{uj} \sim \mathit{Norm}(0,\mathbf{\sigma_{u}}) & (\textrm{cell type effect})\\
b_i \sim \mathit{Cat}(\mathbf{p}_i) & (\textrm{batch id})\\
p_i \sim \mathit{Dir}(\mathbf{\alpha}_b) & (\textrm{batch proportions})\\
t_i \sim \mathit{Cat}(\mathbf{p_q}) & (\textrm{tissue id})\\
p_q \sim \mathit{Dir}(\mathbf{\alpha_t}) & (\textrm{tissue proportions})\\
u_i \sim \mathit{Cat}(\mathbf{p}_{t_i}) & (\textrm{cell type id})\\
p_t \sim \mathit{Dir}(\mathbf{\alpha}_u) & (\textrm{cell type proportions, in batch})\\
\end{array}
\]

And the concept-related factors being:

\[
\begin{array}{rr}
c_{kj} \sim \mathit{ZINorm}(\pi_k,0,\sigma_c) & (\textrm{concept coefficient})\\
\pi_k \sim \mathit{Beta}(\alpha,\beta) & (\textrm{zero inflation})\\
p_u \sim \mathit{Dir}(\alpha_c) & (\textrm{concept probability, in cell type } u)\\
v_{ik} \sim \mathit{Ber}(p_{u_i k}) & (\textrm{concept indicator})\\
\end{array}
\]

Here we briefly describe the intuition behind the different terms:
\begin{itemize}
  \item \textbf{library size} ($l_i$) -- regulates the general expression level of cell $i$. This captures how some cells are more transcriptionally active than others.
  \item \textbf{baseline expression} ($w_j$) -- the general expression level of gene $j$. This captures how some genes (e.g., housekeeping genes) tend to have similar expression across multiple cells.
  \item \textbf{batch id} ($b_i$) -- the batch that cell $i$ belongs to.
  \item \textbf{batch effect} ($\gamma_{bj}$) -- how gene $j$ is impacted in cells that belong to batch $b$. Batches can be, for example, different assays or collection sites.
  \item \textbf{tissue id} ($t_i$) -- the tissue that cell $i$ belongs to. The tissue distribution is dependent on the batch.
  \item \textbf{tissue effect} ($\tau_{tj}$) -- determines how gene $j$ is impacted in cells that belong to tissue $t$. Tissue represents the anatomical region from which the sample was collected.
  \item \textbf{cell type id} ($u_i$) -- the cell type that cell $i$ belongs to. The tissue distribution is dependent on the batch. The cell type distribution is dependent on the tissue.
  \item \textbf{cell type effect} ($\omega_{u j}$) -- determines how gene $j$ is impacted in cells that belong to tissue $u$. Cell type represents the phenotypic state of the cell.
  \item \textbf{concept activation} ($v_{ik}$) -- the activation of concept $k$ in cell $i$. In the synthetic data, the concepts are binary, hence $v_{ik} \in \{0,1\}$. Concept activations are dependent on cell type, to mimic how the likelihood of a concept (e.g., biological pathways) being activated differs between cell types.
  \item \textbf{concept coefficients} ($c_{kj}$) -- determines how concept $k$ influences gene $j$. We model the concept coefficients with zero inflation to create a sparse concept coefficient matrix. This is similar to how most genes are not affected by a biological pathway being activated.
\end{itemize}

\subsection{Motivation}
\label{app:synth:motivation}

scRNA-seq data is derived from a sequencing protocol in which cells are dissociated from tissue, their transcripts are barcoded, and the number of transcripts from each gene is counted. This process is inherently noisy: only a fraction of transcripts in each cell is captured and sequenced. As a result, modeling scRNA-seq data requires accounting for both measurement noise and biological variation.

Let \( x_{ij} \) denote the observed number of transcripts (i.e., gene count) for gene \( j \) in cell \( i \), and let \( y_{ij} \) represent the true number of transcripts of gene \( j \) present in the cell prior to sequencing. 

This process can be viewed as a random sampling step: a fixed number of transcripts \( n_i \) are captured from the full pool of \( N_i = \sum_j y_{ij} \) transcripts present in cell \( i \). Each captured transcript is barcoded and sequenced, resulting in observed counts \( x_{ij} \).

The sampling is performed without replacement and can be modeled by a multivariate hypergeometric (MHG) distribution:

\begin{equation}
    \mathbf{x}_i \sim \mathrm{MHG}(\mathbf{t}_i, n_i)
\end{equation}

where \( \mathbf{t}_i = [t_{i1}, \dots, t_{iG}] \) denotes the true transcript counts per gene in cell \( i \), and \( n_i \) is the total number of transcripts captured from that cell. The resulting vector \( \mathbf{x}_i \) describes how many transcripts of each gene were observed among the \( n_i \) sampled molecules.

Since only a small fraction of all transcripts are captured, we have:

\begin{equation}
    \frac{\sum_j x_{ij}}{\sum_j y_{ij}} = \frac{n_i}{N_i} \to 0
\end{equation}

where \( N_i = \sum_j y_{ij} \) is the true total number of transcripts. Under this low-sampling regime, the MHG distribution can be approximated by a multinomial distribution:

\begin{equation}
    \mathbf{x}_i \sim \mathrm{Multinomial}(n_i, \boldsymbol{\pi}_i), \quad \pi_{ij} = \frac{y_{ij}}{N_i}
\end{equation}

Assuming that the total number of captured transcripts \( n_i \) follows a Poisson distribution with cell-specific rate \( \eta_i \), we obtain:

\begin{equation}
    P(\mathbf{x}_i = \mathbf{k}) = \mathrm{Mult}(\mathbf{k} \mid n_i, \boldsymbol{\pi}_i)\, \mathrm{Poi}(n_i \mid \eta_i) = \prod_j \mathrm{Poi}(k_j \mid \eta_i \pi_{ij})
\end{equation}

Hence, the observed count for each gene \( j \) in cell \( i \) follows a Poisson distribution:

\begin{equation}
    x_{ij} \sim \mathrm{Poisson}(\eta_i \pi_{ij})
\end{equation}

This model captures measurement noise but assumes fixed proportions \( \pi_{ij} \). In reality, gene expression varies across cells due to biological factors such as transcriptional bursting, regulation, or cell state heterogeneity. This leads to overdispersion in the observed counts:

\[
\mathrm{Var}(x_{ij}) > \mathbb{E}[x_{ij}]
\]

To capture this variability, we treat the Poisson rate \( \lambda_{ij} := \eta_i \pi_{ij} \) as a random variable, and model its logarithm as a sum of structured and stochastic components:

\begin{equation}
\begin{array}{c}
\log(\lambda_{ij}) = w_j + l_i + \gamma_{b_i j} + \omega_{u_i j} + \tau_{t_i j} + \sum_{k=1}^{|C|} v_{ik} c_{kj} + \varepsilon_{ij} \\
x_{ij} \sim \mathrm{Poisson}(\lambda_{ij})
\end{array}
\end{equation}

This formulation corresponds to a more general overdispersed Poisson model (e.g., a Poisson GLMM or latent variable model), which flexibly accounts for both technical and biological variability in scRNA-seq data.

\subsection{Example of generated synthetic data}
\label{app:synth:example}

Figure~\ref{fig:synthetic_data_example_labels} and Figure~\ref{fig:synthetic_data_example_concepts} show an example of synthetic data generated by our model, 
comprising 20,000 cells and 5,000 genes, with 2 batches, 2 tissue types, 4 cell types and 6 concepts. We let the effects from batch, tissue, cell type, and concepts follow the following impact order: batch $>$ tissue $>$ concept $>$ cell type.
\begin{figure*}[htpb!]
    \centering
    \includegraphics[width=1\textwidth]{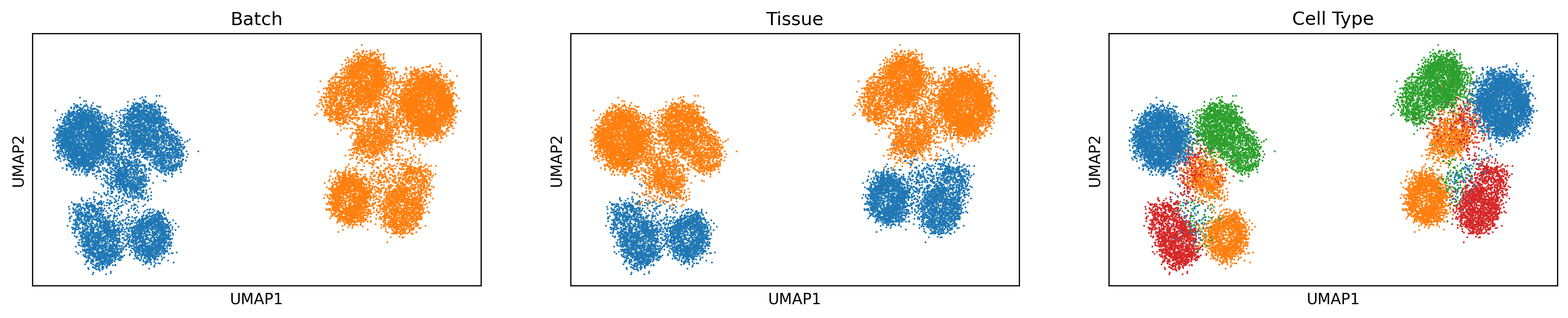}
\caption{Example of synthetic data with highlighted covariates (batch, tissue, cell type). Each dot is a cell, coloring indicate is cells belonging to the same class (e.g, same batch or cell type). Class colors are subplot-specific and not comparable across panels.}
    \label{fig:synthetic_data_example_labels}
\end{figure*}

\begin{figure*}[htpb!]
    \centering
    \includegraphics[width=1\textwidth]{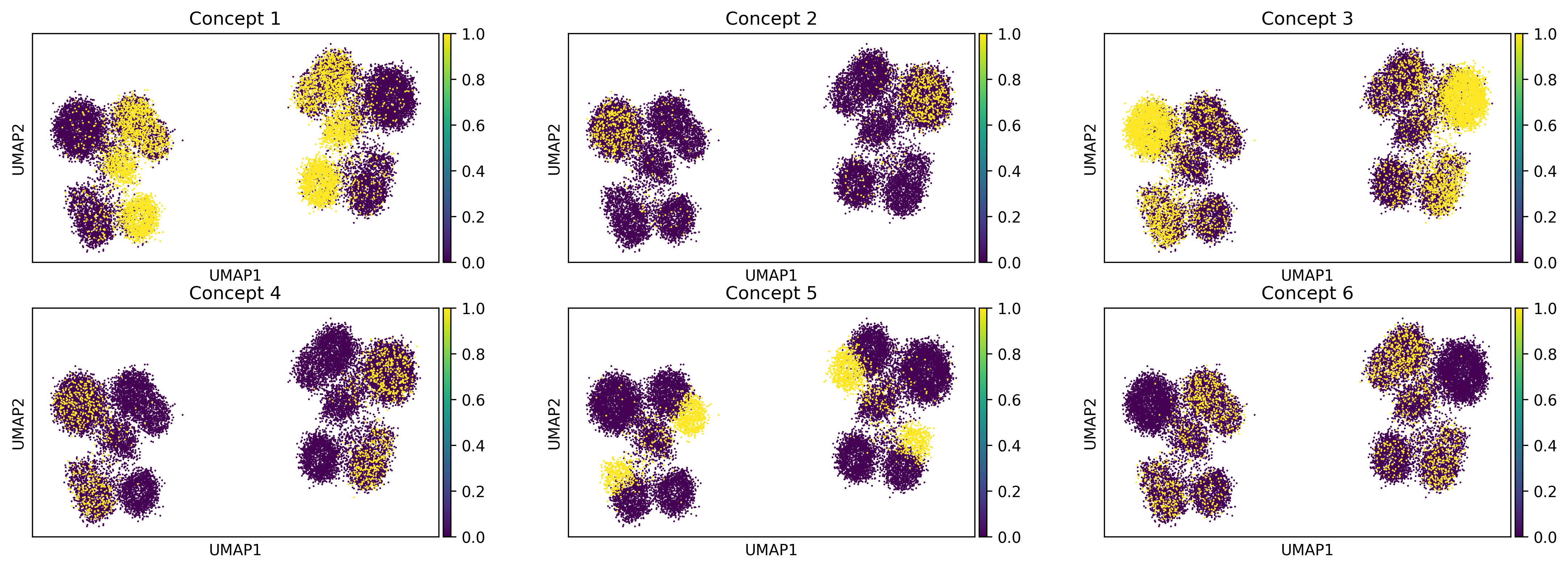}
\caption{Same synthetic data as in Figure \ref{fig:synthetic_data_example_labels} but with highlighted concept activations. Each subplot correspond to one concept.}
    \label{fig:synthetic_data_example_concepts}
\end{figure*}
\newpage

\subsection{Noisy Concept Annotations}
\label{app:synth:noisy_concepts}
To evaluate robustness, we introduce controlled perturbations into the concept annotations, which we refer to as \emph{noisy concepts}. These corruptions mimic realistic imperfections in annotation processes and take the following forms:

\begin{itemize} 
  \item \textbf{Missing:} Randomly removes a specified number or fraction of concept columns, simulating important concepts that are not annotated.  

  \item \textbf{Irrelevant:} Adds concepts with random on/off values, emulating annotations that capture spurious or uninformative signals.  

  \item \textbf{Incorrect:} Selects a subset of existing concepts and independently flips each entry (0$\leftrightarrow$1) with probability $p_{\text{noise}}$, mimicking systematic mislabeling at the concept level.  

  \item \textbf{Duplicate:} Copies selected concept columns one or more times and appends them, increasing redundancy and inducing stronger---or potentially spurious---correlations between concepts and genes.  
\end{itemize}

In all experiments, we evaluated two noise levels: corruption of 1 or 3 concepts. For the \emph{Missing case}, this meant dropping 1 or 3 concepts; for \emph{Irrelevant}, adding 1 or 3 artificial concepts; for \emph{Incorrect}, randomly flipping a certain percentage of the values in 1 or 3 concepts; and for \emph{Duplicate}, replicating 1 or 3 concepts. The intervened-upon concept was always left unaffected.

\subsection{Synthetic Datasets}
\label{app:synth:synt_datasets}
 We generated three synthetic datasets with controlled variations in noise and concept effect strength. \textbf{Synthetic 1} serves as the baseline with default parameters. \textbf{Synthetic 2} increases ''technical noise'' ($\texttt{std\_noise} = 0.08$), introducing more variability and reducing signal clarity. \textbf{Synthetic 3} reduces concept effect variance ($\texttt{std\_concept} = 0.02$), producing sparser and weaker concept-driven influences on gene expression.  

\begin{table}[h]
\centering
\begin{tabular}{lccc}
\toprule
\textbf{Dataset} & \textbf{std\_noise} & \textbf{std\_concept} & \textbf{Key difference vs. baseline} \\
\midrule
\textbf{Synthetic 1} & 0.01 & 0.08 & Baseline \\
\textbf{Synthetic 2} & 0.08 & 0.08 & More technical noise \\
\textbf{Synthetic 3} & 0.01 & 0.02 & Weaker concept effects \\
\bottomrule
\end{tabular}
\caption{Overview of synthetic datasets.}
\end{table}

\paragraph{Shared parameters.} 
All datasets were generated with \textbf{20,000 cells} and \textbf{5,000 genes}, across \textbf{4 cell types}, \textbf{3 tissues}, and \textbf{2 batches}, using \textbf{5 concepts}. Baseline expression values were sampled uniformly between 1 and 5, with batch, tissue, and cell type effects drawn from Normal distributions with standard deviations 0.06, 0.05, and 0.04, respectively. Library size variation was set between 0.01--0.03, residual noise at 0.01 (except where modified), and concept coefficients from a zero-inflated Normal with standard deviation 0.08 (except where modified). We let $\alpha = 1$, $\beta = 0.8$ in the Beta distribution used to sample $\pi_k$, impacting the sparsity of the genes that a concept influences.

\subsection{Intervention Splits for Evaluation}
\label{app:synth_split}
For each dataset, we define 5 intervention tasks by holding out specific concept–cell type pairs, i.e., cases where a given concept is active in a particular cell type. The model is then asked to predict what cells of type $u$ with the concept inactive would look like if it were switched on. This setup emulates a \emph{compositional generalization} scenario. Concretely, let $(u, k)$ denote a held-out pair, where $u$ is the cell type and $k$ the concept of interest:

\begin{itemize}
\item Remove all cells of type $u$ with $\mathbf{c}_k = 1$;
\item Split all cells of type $u$ with $\mathbf{c}_k = 0$ into two equally sized groups: \texttt{train} and \texttt{intervene}
\item Generate ground-truth counterfactuals for the \texttt{intervene} group using Eq.~\ref{eq:synth_counterfactual}, yielding the \texttt{test} set
\item Assign all remaining cells to the \texttt{train} group.
\end{itemize}

In this way, cells of type $u$ are present in the training data, and cells with $c_k = 1$ are also present, but no instance of type $u$ with $c_k = 1$ is ever observed during training. The model is trained on the \texttt{train} group, interventions are applied to the held-out \texttt{intervene} group, and predictions are evaluated against the corresponding counterfactuals in the \texttt{test} set. Because true counterfactuals are available for each intervened cell, evaluation can be performed at the cell level. The five intervention tasks are created by repeating this procedure across randomly chosen $(u, k)$ pairs to ensure robust assessment.

\subsection{Hyperparameter Sweep}
\label{app:syn:hparam_sweep}
For each model (scCBGM and CBGM) we performed a grid search over key hyperparameters. We use the dataset ''Synthetic 1'', with noiseless (clean) concept annotations for the sweep. Specifically, we varied the random seed $\{13, 69\}$, the latent dimension $\{64, 128\}$, hidden dimension $\{128, 256\}$, $\beta \in \{10^{-4}, 5 \times 10^{-5}, 10^{-5}\}$, learning rate $\{5 \times 10^{-4}, 10^{-5}\}$, orthogonality regularization $\{0.05, 0.2, 0.5\}$, and concept regularization $\{0.005, 0.05, 0.1\}$. Other parameters (e.g., bottleneck size $=128$, 2 layers) were held fixed. We picked the set of hyperparameters that performed best across all five (synthetic interventions). This resulted in a total of $2160$ different runs and a total of $432$ hyperparameter sets that were evaluated. The full set of results across all sweeps are displayed in Figure~\ref{fig:hyperparam_results}.
\begin{figure}[htbp!]
    \centering
    \includegraphics[width=0.7\linewidth]{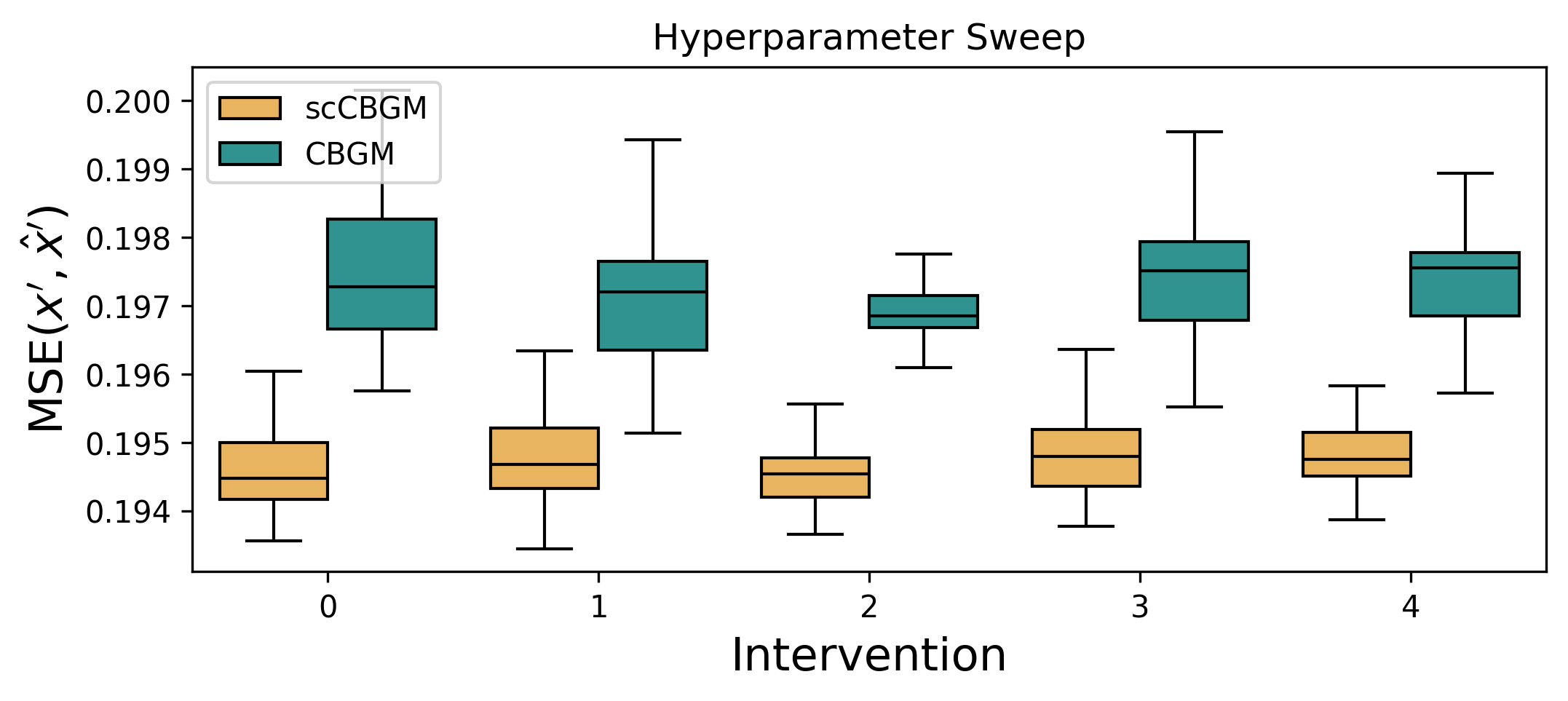}
    \caption{Results from hyperparameter sweep for the scCBGM and CBGM model. Error are \textrm{cell-level} MSEs, comparing the predicted counterfactual and the true counterfactual observation.}
    \label{fig:hyperparam_results}
\end{figure}


\section{Experimental details}
\label{app:experiment_details}

\subsection{Data Processing}
We benchmarked scCBGM and its Flow Matching variants against contemporary approaches for single-cell perturbations prediction on 3 datasets with annotated cell-types across varying perturbation conditions. Each dataset was minimally processed to fit within our counter-factual framework and to allow for systematic comparisons. The list of intervention variables for each dataset is given in Table~\ref{tab:intervention_variables}.

\subsubsection{\texorpdfstring{Kang et al. PBMC IFN-$\beta$ stimulation data}{Kang et al. PBMC IFN-beta stimulation data}}
\label{app:kang_data}

The \cite{kang} dataset comprises 24,264 cells across 8  broad-cell types, observed under two conditions: with and without IFN-$\beta$ stimulation. We excluded megakaryocytes due to their low cell count (210 cells). Data was preprocessed using \textit{scanpy} \cite{wolf2018scanpy}, involving median library size normalization, log-transformation of all counts, and filtering to the top 3000 most highly expressed genes.

\begin{table}[h]
    \centering
    \caption{Overview of cell types and subtypes with instance counts per condition. Note the sparsity in CD14+ Monocytes 1 (ctrl).}
    \label{tab:subtype_counts}
\begin{tabular}{lllr}
\toprule
 &  & Condition & Instances \\
Broad & Subtype &  &  \\
\midrule
\multirow[t]{4}{*}{B cells} & B cells 0 & ctrl & 1105 \\
 & B cells 0 & stim & 1044 \\
 & B cells 1 & ctrl & 201 \\
 & B cells 1 & stim & 187 \\
\cline{1-4}
\multirow[t]{4}{*}{CD14+ Monocytes} & CD14+ Monocytes 0 & ctrl & 2764 \\
 & CD14+ Monocytes 0 & stim & 2023 \\
 & CD14+ Monocytes 1 & ctrl & 9 \\
 & CD14+ Monocytes 1 & stim & 497 \\
\cline{1-4}
\multirow[t]{6}{*}{CD4 T cells} & CD4 T cells 0 & ctrl & 2743 \\
 & CD4 T cells 0 & stim & 2743 \\
 & CD4 T cells 1 & ctrl & 1873 \\
 & CD4 T cells 1 & stim & 1830 \\
 & CD4 T cells 2 & ctrl & 538 \\
 & CD4 T cells 2 & stim & 518 \\
\cline{1-4}
\multirow[t]{4}{*}{CD8 T cells} & CD8 T cells 0 & ctrl & 650 \\
 & CD8 T cells 0 & stim & 607 \\
 & CD8 T cells 1 & ctrl & 432 \\
 & CD8 T cells 1 & stim & 333 \\
\cline{1-4}
\multirow[t]{4}{*}{Dendritic cells} & Dendritic cells 0 & ctrl & 159 \\
 & Dendritic cells 0 & stim & 126 \\
 & Dendritic cells 1 & ctrl & 44 \\
 & Dendritic cells 1 & stim & 98 \\
\cline{1-4}
\multirow[t]{4}{*}{FCGR3A+ Monocytes} & FCGR3A+ Monocytes 0 & ctrl & 431 \\
 & FCGR3A+ Monocytes 0 & stim & 556 \\
 & FCGR3A+ Monocytes 1 & ctrl & 305 \\
 & FCGR3A+ Monocytes 1 & stim & 276 \\
\cline{1-4}
\multirow[t]{2}{*}{NK cells} & NK cells 0 & ctrl & 915 \\
 & NK cells 0 & stim & 1047 \\
\cline{1-4}
\bottomrule
\end{tabular}
    \caption{Overview of the cell types and subtypes in the \cite{kang} PBMC dataset.}
    \label{tab:kang_subtypes}
\end{table}



To benchmark our model on high-fidelity edits that preserve cell phenotype while changing experimental conditions, we focused on identifying granular phenotypes consistent across stimulated and unstimulated cells. We first integrated the two conditions into a unified latent space using Harmony \cite{korsunsky2019fast}. Within this unified latent space, we applied Leiden clustering \cite{traag2019louvain} independently to each broad-cell type to discover direct subtypes. This process yielded a total of 14 subtypes, each approximately evenly distributed across the two experimental conditions, see Table~\ref{tab:kang_subtypes}. For the \citet{kang} experiments, concepts were defined using these original broad-cell types, while stimulation predictions and held-out validations were conducted at the more granular subtype level. We ran our experiments over the same 4 random seeds for all models.
Table~\ref{tab:kang_sweep_hparams} presents the hyperparameters used for the~\citet{kang} experiments.

\begin{table}[ht]
\centering
\caption{Hyperparameters and seeds used in the \citet{kang} experiments. All neural-network-based
methods share \texttt{hidden\_dim}=1024, \texttt{n\_layers}=4, \texttt{batch\_size}=128.
For two-stage FM methods (biolord-FM, CVAE-FM, scCBM-FM), \emph{Epochs} and \emph{LR}
refer to the FM training stage; the base model uses the same settings as its non-FM
counterpart. $\beta$: KL divergence weight (\texttt{kl\_hp} in ConceptFlow);
$z_{\text{dim}}$: latent dimensionality; $\lambda_{\text{orth}}$: orthogonality
regularisation weight. Dashes indicate not applicable.
${}^{\star}$~Cinema-OT: \texttt{thresh}=0.5, \texttt{eps}=0.001.
${}^{\dagger}$~biolord LRs: \texttt{latent\_lr}=\texttt{decoder\_lr}=$10^{-4}$,
\texttt{attribute\_nn\_lr}=$10^{-2}$; penalty weights:
\texttt{reconstruction\_penalty}=$10^{2}$, \texttt{unknown\_attribute\_penalty}=$10^{1}$.
${}^{\ddagger}$~biolord-FM FM component has \texttt{latent\_dim}=128 (independent of
biolord's \texttt{n\_latent}=32).
ConceptFlow additionally has \texttt{flow\_hp}=1.0.
${}^{\mathsection}$~Ablation sweeps vary \texttt{use\_cosine\_loss}$\in\{\text{T,F}\}$
and \texttt{decoder\_type}$\in\{$skip, no-residual, residual$\}$ (residual not included
for CBGM). ${}^{\P}$~scCBM (no-var.) disables the variational bottleneck
(\texttt{variational}=False), so there is no KL term and $\beta$ does not apply.}
\label{tab:kang_sweep_hparams}
\resizebox{\linewidth}{!}{%
\begin{tabular}{llcccccccc}
\toprule
\textbf{Method} & \textbf{Data} & \textbf{Seeds} & \textbf{Epochs}
  & \textbf{LR} & $z_{\text{dim}}$ & $\beta$
  & \texttt{concepts\_hp} & $\lambda_{\text{orth}}$ & \texttt{model.edit} \\
\midrule
scGen          & kang & 13, 42, 1337, 69 & 100  & $10^{-3}$               & ---  & ---       & ---  & ---  & ---              \\
CVAE           & kang & 13, 42, 1337, 69 & 200  & $3\!\times\!10^{-4}$    & 128  & $10^{-5}$ & ---  & ---  & ---              \\
scCBM          & kang & 13, 42, 1337, 69 & 200  & $3\!\times\!10^{-4}$    & 128  & $10^{-5}$ & 0.1  & 0.5  & ---              \\
CBGM           & kang & 13, 42, 1337, 69 & 200  & $3\!\times\!10^{-4}$    & 128  & $10^{-5}$ & 0.1  & 0.5  & ---              \\
biolord$^{\dagger}$ & kang & 13, 42, 1337, 69 & 100 & $10^{-4}$           & 32   & ---       & ---  & ---  & ---              \\
Cinema-OT$^{\star}$ & kang & 13, 42, 1337, 69 & --- & ---                 & ---  & ---       & ---  & ---  & ---              \\
ConceptFlow    & kang & 13, 42, 1337, 69 & 200  & $3\!\times\!10^{-4}$    & 128  & 0.1       & 0.1  & 0.1  & ---              \\
\midrule
biolord-FM$^{\dagger\ddagger}$ & kang & 13, 42, 1337, 69 & 200 & $3\!\times\!10^{-4}$ & 128 & --- & --- & --- & ---         \\
CVAE-FM        & kang & 13, 42, 1337, 69 & 200  & $3\!\times\!10^{-4}$    & 128  & $10^{-5}$ & ---  & ---  & \{edit, decode\} \\
Vanilla-FM     & kang & 13, 42, 1337, 69 & 200  & $3\!\times\!10^{-4}$    & 128  & ---       & ---  & ---  & \{edit, decode\} \\
scCBM-FM       & kang & 13, 42, 1337, 69 & 200  & $3\!\times\!10^{-4}$    & 128  & $10^{-5}$ & 0.1  & 0.5  & \{edit, decode\} \\
\midrule
scCBM (abl.)$^{\mathsection}$ & kang & 13, 42, 1337, 69 & 200 & $3\!\times\!10^{-4}$ & 128 & $10^{-5}$ & 0.1 & \{0.0, 0.5\} & --- \\
CBGM (abl.)$^{\mathsection}$  & kang & 13, 42, 1337, 69 & 200 & $3\!\times\!10^{-4}$ & 128 & $10^{-5}$ & 0.1 & \{0.0, 0.5\} & --- \\
scCBM (no-var.)$^{\P}$ & kang & 13, 42, 1337, 69 & 200 & $3\!\times\!10^{-4}$ & 128 & ---    & 0.1  & 0.5  & ---              \\
\bottomrule
\end{tabular}%
}
\end{table}

\subsubsection{Cui et al. immune dictionary data}

We further evaluated \method~using the Immune Dictionary dataset by \cite{cui2024dictionary}. This dataset comprises 110,378 cells, encompassing over 17 distinct immune cell types stimulated with 86 different cytokines in vivo, providing a rich map of cellular responses to immune perturbations. For preprocessing, similar to the Kang et al. dataset, we applied median library size normalization, log-transformation of counts, and filtered to the top 3000 most highly expressed genes.

Unlike the Kang et al. data where we derived subtypes, the Cui et al. dataset already provides 17 granular immune cell type labels. For this dataset, we created broader concept groupings by pooling these original labels into seven aggregate cell types: Stromal cells, Granulocytes, Monocyte-Macrophages (referred to as myeloid in the tables), Dendritic cells, Innate lymphoid cells (referred to as lymphoid in the tables), B cells, and T cells. These broader categories served as our primary cell type concepts. Additionally, concepts were created for each cytokine stimulation. Consistent with our approach for the Kang et al. data, stimulation predictions and held-out validations were performed at the original, granular 17-subtype level.

Given the vast number of possible cytokine-cell type combinations (86 cytokines across 17 cell types), we focused our evaluation on a select set of seven cell-type - cytokine pairs. These pairs were chosen based on the Immune Dictionary paper's findings, representing combinations shown to induce a significant transcriptional shift compared to control conditions. The specific intervention pairs evaluated were (cell type - cytokine): 

\begin{itemize}
\item $\gamma$$\delta$ T cells - IL17E
\item CD4 T cells - TGF$\beta$
\item CD8 T cells - TNF$\alpha$
\item Conventional dendritic cells (cDC) 2 - INF$\alpha$1
\item Langerhans - IFN$\gamma$
\item Macrophages - M-CSF
\item NK cells - IL15
\end{itemize}
We ran our experiments over the same 4 random seeds for all models. Table~\ref{tab:cui_sweep_hparams} presents the hyperparameters used for the~\citet{cui2024dictionary} experiments.

\begin{table}[ht]
\centering
\caption{Hyperparameters and seeds used in the \citet{cui2024dictionary} experiments. All neural-network-based
methods share \texttt{hidden\_dim}=1024, \texttt{n\_layers}=4, \texttt{batch\_size}=128.
For two-stage FM methods (biolord-FM, CVAE-FM, scCBM-FM), \emph{Epochs} and \emph{LR}
refer to the FM training stage; the base model uses the same settings as its non-FM
counterpart. $\beta$: KL divergence weight (\texttt{kl\_hp} in ConceptFlow);
$z_{\text{dim}}$: latent dimensionality of the learned representation or FM flow space;
$\lambda_{\text{orth}}$: orthogonality regularization weight.
Dashes indicate not applicable.
${}^{\star}$~Cinema-OT is a non-parametric optimal-transport method; its only
hyperparameters are \texttt{thresh}=0.5 and \texttt{eps}=0.001.
${}^{\dagger}$~biolord uses component-wise LRs: \texttt{latent\_lr}=\texttt{decoder\_lr}=$10^{-4}$,
\texttt{attribute\_nn\_lr}=$10^{-2}$; additional loss weights:
\texttt{reconstruction\_penalty}=$10^{2}$, \texttt{unknown\_attribute\_penalty}=$10^{1}$.
${}^{\ddagger}$~biolord-FM's FM component has \texttt{latent\_dim}=128 (independent of
biolord's \texttt{n\_latent}=32).
ConceptFlow additionally has \texttt{flow\_hp}=1.0 (FM loss weight).}
\label{tab:cui_sweep_hparams}
\resizebox{\linewidth}{!}{%
\begin{tabular}{llcccccccc}
\toprule
\textbf{Method} & \textbf{Data} & \textbf{Seeds} & \textbf{Epochs}
  & \textbf{LR} & $z_{\text{dim}}$ & $\beta$
  & \texttt{concepts\_hp} & $\lambda_{\text{orth}}$ & \texttt{model.edit} \\
\midrule
scGen          & cui      & 13, 42, 1337, 69 & 100  & $10^{-5}$               & ---  & ---       & ---  & ---  & ---              \\
CVAE           & cui      & 13, 42, 1337, 69 & 200  & $3\!\times\!10^{-4}$    & 128  & $10^{-5}$ & ---  & ---  & ---              \\
scCBM          & cui      & 13, 42, 1337, 69 & 200  & $3\!\times\!10^{-4}$    & 128  & $10^{-5}$ & 0.1  & 0.5  & ---              \\
CBGM           & cui      & 13, 42, 1337, 69 & 200  & $3\!\times\!10^{-4}$    & 128  & $10^{-5}$ & 0.1  & 0.5  & ---              \\
biolord$^{\dagger}$ & cui & 13, 42, 1337, 69 & 100  & $10^{-4}$               & 32   & ---       & ---  & ---  & ---              \\
Cinema-OT$^{\star}$ & cui & 13, 42, 1337, 69 & ---  & ---                     & ---  & ---       & ---  & ---  & ---              \\
ConceptFlow    & cui      & 13, 42, 1337, 69 & 200  & $3\!\times\!10^{-4}$    & 128  & 0.1       & 0.1  & 0.1  & ---              \\
\midrule
biolord-FM$^{\dagger\ddagger}$ & cui & 13, 42, 1337, 69 & 200 & $3\!\times\!10^{-4}$ & 128 & --- & --- & --- & ---              \\
CVAE-FM        & cui      & 13, 42, 1337, 69 & 200  & $3\!\times\!10^{-4}$    & 128  & $10^{-5}$ & ---  & ---  & \{edit, decode\} \\
Vanilla-FM     & cui      & 13, 42, 1337, 69 & 200  & $3\!\times\!10^{-4}$    & 128  & ---       & ---  & ---  & \{edit, decode\} \\
scCBM-FM       & cui      & 13, 42, 1337, 69 & 200  & $3\!\times\!10^{-4}$    & 128  & $10^{-5}$ & 0.1  & 0.5  & \{edit, decode\} \\
\bottomrule
\end{tabular}%
}
\end{table}

\subsubsection{Nault et al. liver dosages response dataset}

To assess \method's capacity for modeling continuous responses and its utility in cellular control, we utilized the Nault et al. dataset \cite{nault2023single}. This dataset comprises 131,613 cells across 11 granular cell types of liver responses to varying dosages of TCDD, a toxic dioxin compound, ranging from 0$\mu$g/kg up to 30$\mu$g/kg. Preprocessing followed the same procedure as the other datasets: median library size normalization, log-transformation of counts, and filtering to the top 3000 most highly expressed genes.

We conducted two primary experiments using this dataset:

\paragraph{Continuous Dosage Prediction}
For this experiment, concepts were defined to include broad cell types and TCDD dosage. The 11 original cell types were grouped into four broad categories: Hepatocyte, Immune cells, Cholangiocyte, and Stromal, which served as hard  concepts. TCDD dosage was treated as a soft (continuous) concept and its true values were log-transformed and then scaled to be between 0 and 1. The counterfactual task involved predicting gene expression profiles for intermediate TCDD dosages, specifically holding out the $3$ and $10\mu$g/kg dosed test cell subtypes during training. This setup enabled benchmarking \method's ability to fill experimental gaps by inferring responses at unobserved dosage levels. For this experiment, we used a single seed and aggregated the results over the different treatment dosages.

\paragraph{Interpretable Cellular Control via Pathway Activation}
This experiment demonstrated \method's application in controlling cells towards a desired biological state, particularly enhancing treatment effects in specific cell types. We focused on hepatic stellate cells, which showed only a moderate response to TCDD in the original dataset (i.e., minimal effect at lower dosages, unlike hepatocytes, but responding at 10 and 30$\mu$g/kg).

First, we computed gene-set loadings for PROGENy pathways \cite{schubert2018perturbation} using the \textit{decoupler} package \cite{badia2022decoupler}. By comparing responder hepatic stellate cells (dosed at 10 and 30$\mu$g/kg) with non-responders (all other dosages), we identified differentially regulated pathways. Specifically, responder stellate cells exhibited high activity in TGF$\beta$, PI3K, and p53 pathways, and low activity in Trail and VEGF pathways. All $15$ PROGENy pathways were mapped to continuous concepts, with their values min-max scaled between $0$ and $1$ across the entire dataset.

For the editing task, all responder hepatic stellate cells were explicitly removed from the training data, meaning the model was unaware that these cells could respond to TCDD. We then performed in silico stimulation on control stellate cells by jointly modifying the soft concepts corresponding to the identified pathways (increasing TGF$\beta$, PI3K, p53, and decreasing Trail, VEGF), while simultaneously activating the TCDD dosage concept. Specifically, we intervened on hepatic stellate cells with a dosage of 0.01$\mu$g/kg. We encoded these cells, edited their dose concept to 30$\mu$g/kg, and shifted their \textit{TGF$\beta$, PI3K, p53, Trail}, and \textit{VEGF} pathway concepts to $0.598, 0.521, 0.327, 0.697$, and $0.333$ (respectively). These specific concept values represent the mean activations within the responder population.

The resulting edited cell population displayed markedly increased sensitivity to the compound, validating \method's ability to identify and modulate mechanistic routes for treatment response. 

We used a single seed for all models and the standard deviations were obtained using the variability over different treatment dosages. Table~\ref{tab:liver_sweep_hparams} presents the hyperparameters used for the~\citet{nault2023single} experiments

\begin{table}[ht]
\centering
\caption{Hyperparameters used in the \citet{nault2023single} experiments. All experiments use
a single random seed (0). All neural-network-based methods share
\texttt{hidden\_dim}=1024, \texttt{n\_layers}=4, \texttt{batch\_size}=128.
$\beta$: KL divergence weight; $z_{\text{dim}}$: latent dimensionality;
$\lambda_{\text{orth}}$: orthogonality regularization weight.
For two-stage FM methods (CB-FM, cVAE-FM), \emph{Epochs} and \emph{LR}
refer to the FM training stage; the base model (scCBGM or CVAE) uses the same settings
as its non-FM counterpart. The \emph{inference mode} column distinguishes
edit-mode ($\nabla$-guided flow editing) from guided-mode (conditional decoding),
corresponding to \texttt{edit=True/False} in the scripts;
these are denoted \textit{edit} and \textit{guided} respectively
(cf.\ \textit{decode} in the CUI/Kang tables).
Dashes indicate not applicable.
${}^{\star}$~Cinema-OT: \texttt{thresh}=0.5, \texttt{eps}=0.001.
${}^{\dagger}$~biolord LRs: \texttt{latent\_lr}=\texttt{decoder\_lr}=$10^{-4}$,
\texttt{attribute\_nn\_lr}=$10^{-2}$; penalty weights:
\texttt{reconstruction\_penalty}=$10^{2}$,
\texttt{unknown\_attribute\_penalty}=$10^{1}$.
${}^{\ddagger}$~CBGM in the baselines script uses \texttt{CEM\_MetaTrainer} with
\texttt{max\_epochs}=200; scCBGM in the main script uses \texttt{CB\_VAE\_MIXED} with
\texttt{num\_epochs}=100. Both share the same architectural hyperparameters.}
\label{tab:liver_sweep_hparams}
\resizebox{\linewidth}{!}{%
\begin{tabular}{lcccccccc}
\toprule
\textbf{Method} & \textbf{Seed} & \textbf{Epochs} & \textbf{LR}
  & $z_{\text{dim}}$ & $\beta$ & \texttt{concepts\_hp}
  & $\lambda_{\text{orth}}$ & \textbf{Inference mode} \\
\midrule
Cinema-OT$^{\star}$ & 0 & ---  & ---                     & ---  & ---       & ---  & ---  & ---              \\
scGen               & 0 & 100  & $10^{-5}$               & ---  & ---       & ---  & ---  & ---              \\
CBGM$^{\ddagger}$   & 0 & 200  & $3\!\times\!10^{-4}$    & 128  & $10^{-5}$ & 0.1  & 0.5  & ---              \\
biolord$^{\dagger}$ & 0 & 100  & $10^{-4}$               & 32   & ---       & ---  & ---  & ---              \\
scCBGM$^{\ddagger}$ & 0 & 100  & $3\!\times\!10^{-4}$    & 128  & $10^{-5}$ & 0.1  & 0.5  & ---              \\
CVAE                & 0 & 100  & $3\!\times\!10^{-4}$    & 128  & $10^{-5}$ & ---  & ---  & ---              \\
CB-FM               & 0 & 100  & $3\!\times\!10^{-4}$    & 128  & ---       & ---  & ---  & \{edit, guided\} \\
Raw-FM              & 0 & 100  & $3\!\times\!10^{-4}$    & 128  & ---       & ---  & ---  & \{edit, guided\} \\
cVAE-FM             & 0 & 100  & $3\!\times\!10^{-4}$    & 128  & ---       & ---  & ---  & \{edit, guided\} \\
\bottomrule
\end{tabular}%
}
\end{table}

\begin{table}[!h]
\centering
\caption{Intervention variables used across the three datasets.
\emph{\# tasks} counts the number of held-out (cell type, perturbation) conditions
evaluated per dataset: for \citet{kang} the same IFN-$\beta$ stimulus is applied to
multiple donor-cluster splits of each cell type, giving 14 tasks over 7 cell-type groups; for \citet{nault2023single} each cell type is evaluated
at two held-out dose levels, giving 6 tasks over 3 cell types.
All concepts are binary (0/1) or continuous (dose) scalars embedded in the concept
bottleneck.}
\label{tab:intervention_variables}
\begin{tabular}{llllrc}
\toprule
\textbf{Dataset} & \textbf{Cell type} & \textbf{Perturbation} & \textbf{Reference state} & \textbf{Target state} & \textbf{\# tasks} \\
\midrule
\multirow{7}{*}{\citet{cui2024dictionary}}
  & CD4 T cells           & TGF-$\beta$  & PBS & TGF-$\beta$ & 1 \\
  & CD8 T cells           & TNF$\alpha$   & PBS & TNF$\alpha$  & 1 \\
  & $\gamma\delta$ T cells & IL-17E       & PBS & IL-17E       & 1 \\
  & NK cells              & IL-15         & PBS & IL-15        & 1 \\
  & Macrophages           & M-CSF         & PBS & M-CSF        & 1 \\
  & cDC2                  & IFN$\alpha$1  & PBS & IFN$\alpha$1 & 1 \\
  & Langerhans cells      & IFN$\gamma$   & PBS & IFN$\gamma$  & 1 \\
\midrule
\multirow{7}{*}{\citet{kang}}
  & B cells               & IFN-$\beta$   & ctrl & stim & 2 \\
  & CD4 T cells           & IFN-$\beta$   & ctrl & stim & 3 \\
  & CD8 T cells           & IFN-$\beta$   & ctrl & stim & 2 \\
  & CD14$^+$ Monocytes    & IFN-$\beta$   & ctrl & stim & 2 \\
  & Dendritic cells       & IFN-$\beta$   & ctrl & stim & 2 \\
  & FCGR3A$^+$ Monocytes  & IFN-$\beta$   & ctrl & stim & 2 \\
  & NK cells              & IFN-$\beta$   & ctrl & stim & 1 \\
\midrule
\multirow{3}{*}{\citet{nault2023single}}
  & B cells               & Drug dose     & 1.0  & \{3.0, 10.0\} & 2 \\
  & Macrophages           & Drug dose     & 1.0  & \{3.0, 10.0\} & 2 \\
  & T cells               & Drug dose     & 1.0  & \{3.0, 10.0\} & 2 \\
\midrule
\multicolumn{5}{r}{\textbf{Total}} & \textbf{27} \\
\bottomrule
\end{tabular}
\end{table}

\clearpage
\subsection{Evaluation metrics}
\label{app:metrics}

We compare performance of different models using the following metrics: the Maximum Mean Discrepancy ratio (rMMD), the Frechet Inception Distance ratio (rFID), and the Sinkhorn Divergence ratio (rSD). We give a definition of these metrics below. Complete benchmarking results with these metrics are provided in Appendix~\ref{app:additional_benchmark}.

\paragraph{Notations}

We define $\hat{p}_{\mathbf{c}}^*$ as the true empirical distribution of an initial cell population of interest with concept $\mathbf{c}$, $\hat{p}_{\mathbf{c}}$ for the empirical distribution of the population of \emph{all} cells with with concept $\mathbf{c}$, $\hat{p}_{\mathbf{c}\rightarrow\mathbf{c}'}^*$ the distribution of the target counterfactual cell population of interest with concept $\mathbf{c}'$. We write $\hat{p}_{\mathbf{c}\rightarrow\mathbf{c}'}^*$ for the \emph{predicted} counterfactual distribution.

The ratio formulation normalizes the performance of the model’s prediction against the best achievable (lowest) metric among all cell populations, providing a scale-independent measure of relative quality.

\paragraph{rMMD}

The Maximum Mean Discrepancy ratio is defined as

\begin{equation}
\textrm{rMMD} = \frac{\textrm{MMD}(\hat{p}_{\mathbf{c}\rightarrow\mathbf{c}'}^*,\hat{p}_{\mathbf{c}\rightarrow\mathbf{c}'}^*)}{\min_{\mathbf{\gamma}}\textrm{MMD}(\hat{p}_{\mathbf{\gamma}},\hat{p}_{\mathbf{c}\rightarrow\mathbf{c}'}^*)}.
\end{equation}

where $\textrm{rMMD}$ is defined as the \textit{squared} maximum mean discrepancy with kernel $k$:

\begin{equation}
\textrm{MMD}(p,q) = \mathbb{E}_p[k(X,X)] - 2\mathbb{E}_{p,q}[k(X,Y)] + \mathbb{E}_q[k(Y,Y)].
\end{equation}

In our experiments, we used the RBF kernel.

\paragraph{rFID}

The Frechet Inception Distance ratio (rFID) is defined as

\begin{equation}
\textrm{rFID} = \frac{\textrm{FID}(\hat{p}_{\mathbf{c}\rightarrow\mathbf{c}'}^*, \hat{p}_{\mathbf{c}\rightarrow\mathbf{c}'}^*)}{\min_{\mathbf{\gamma}}\textrm{FID}(\hat{p}_{\mathbf{\gamma}}, \hat{p}_{\mathbf{c}\rightarrow\mathbf{c}'}^*)}.
\end{equation}

The Fréchet Inception Distance (FID) measures the difference between two multivariate Gaussian distributions fitted to the input empirical distributions. For two distributions $p$ and $q$ with means $\mu_p, \mu_q$ and covariances $\Sigma_p, \Sigma_q$, it is given by

\begin{equation}
\textrm{FID}(p, q) = \|\mu_p - \mu_q\|_2^2 + \mathrm{Tr}\left(\Sigma_p + \Sigma_q - 2(\Sigma_p \Sigma_q)^{1/2}\right).
\end{equation}

\paragraph{rSD}

The Sinkhorn Divergence ratio (rSD) is defined as

\begin{equation}
\textrm{rSD} = \frac{\textrm{SD}_\varepsilon(\hat{p}_{\mathbf{c}\rightarrow\mathbf{c}'}^*, \hat{p}_{\mathbf{c}\rightarrow\mathbf{c}'}^*)}{\min_{\mathbf{\gamma}}\textrm{SD}_\varepsilon(\hat{p}_{\mathbf{\gamma}}, \hat{p}_{\mathbf{c}\rightarrow\mathbf{c}'}^*)}.
\end{equation}

The Sinkhorn Divergence $\textrm{SD}_\varepsilon(p, q)$ is an entropically regularized optimal transport divergence between two distributions $p$ and $q$:

\begin{equation}
\textrm{SD}_\varepsilon(p, q) = W_{2,\varepsilon}^2(p, q) - \tfrac{1}{2}\left(W_{2,\varepsilon}^2(p, p) + W_{2,\varepsilon}^2(q, q)\right),
\end{equation}

where $W_{2,\varepsilon}$ denotes the entropically regularized Wasserstein-2 distance with regularization strength $\varepsilon$.  
The regularization smooths the optimal transport problem by adding an entropy term, which makes computation more tractable. As $\varepsilon \to 0$, $\textrm{SD}_\varepsilon$ converges to the squared Wasserstein distance $W_2^2$, while for large $\varepsilon$, it approaches the squared Maximum Mean Discrepancy (MMD)~\citep{feydy2019interpolating}.

\subsection{Challenges of evaluation on real data}
\label{app:challenges_of_real}

The use of real scRNA-seq data is essential to assess how our model performs under realistic conditions --- i.e., on data of the kind it is ultimately intended to be applied to. However, we emphasize that such datasets lack definitive ground truth labels. In our experiments, we use \textit{cell type} and \textit{stimulation} annotations both as concepts for intervention and as targets for evaluation in our ratio-based metrics. These labels, however, are not perfectly accurate: cell type annotations are inherently noisy and error-prone, and even for stimulation conditions, cells may have been exposed to the treatment but failed to exhibit a transcriptional response.  To illustrate this, we analyze the Kang \textit{et al.} dataset.

\begin{figure}[ht!]
    \centering
    \includegraphics[width=0.8\linewidth]{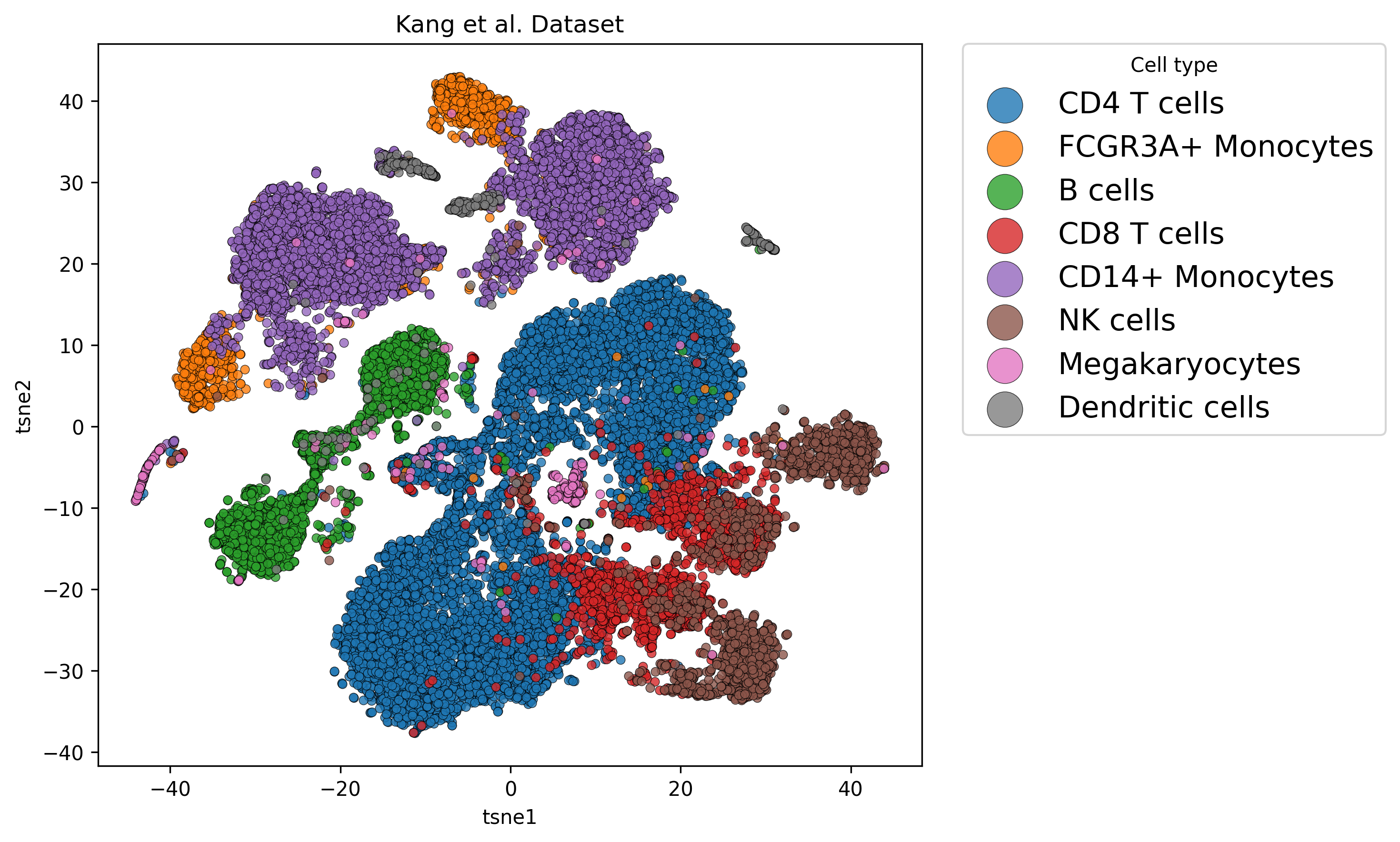}
    \caption{t-SNE of the Kang \textit{et al.} dataset using the original cell type labels. Each point corresponds to a single cell, colored by its annotated cell type.}
    \label{fig:kang_overlap}
\end{figure}

As shown in Figure~\ref{fig:kang_overlap}, several annotated cell types exhibit strong overlap in the embedding space, most notably NK cells and CD8 T cells. To quantify this overlap, we performed a pairwise 5-fold stratified cross-validation analysis. Specifically, we trained a logistic regression model to discriminate between each pair of cell types using their PCA representations and evaluated the performance on held-out sets using the misclassification error rate. As detailed in Table~\ref{tab:kang_separability}, several pairs yield high error rates, indicating that these classes are not linearly separable—likely due to biological similarity in underlying cell states or noise in the original annotations. These results provide context for instances where ratio-based metrics exceed one, as the lack of clear separation complicates distinct cluster definition. Notably, CD8 T cells and NK cells, FCGR3A+ Monocytes and CD14+ Monocytes, and FCGR3A+ Monocytes and Dendritic cells represent the most difficult pairs to resolve.


\begin{table}[ht!]
\centering
\resizebox{\textwidth}{!}{%
\begin{tabular}{lcccccccc}
\toprule
 & B cells & CD4 T cells & CD8 T cells & CD14+ Monocytes & Dendritic cells & FCGR3A+ Monocytes & Megakaryocytes & NK cells \\
\midrule
B cells & NaN & 0.40\% & 0.50\% & 0.11\% & 2.09\% & 1.09\% & 0.97\% & 0.36\% \\
CD4 T cells & 0.40\% & NaN & 2.43\% & 0.17\% & 0.13\% & 0.25\% & 0.80\% & 0.29\% \\
CD8 T cells & 0.50\% & 2.43\% & NaN & 0.10\% & 0.29\% & 0.42\% & 1.23\% & \textbf{8.90\%} \\
CD14+ Monocytes & 0.11\% & 0.17\% & 0.10\% & NaN & 1.04\% & \textbf{7.07\%} & 1.01\% & 0.12\% \\
Dendritic cells & 2.09\% & 0.13\% & 0.29\% & 1.04\% & NaN & \textbf{4.00\%} & 2.94\% & 0.00\% \\
FCGR3A+ Monocytes & 1.09\% & 0.25\% & 0.42\% & \textbf{7.07\%} & \textbf{4.00\%} & NaN & 2.30\% & 0.45\% \\
Megakaryocytes & 0.97\% & 0.80\% & 1.23\% & 1.01\% & 2.94\% & 2.30\% & NaN & 0.77\% \\
NK cells & 0.36\% & 0.29\% & \textbf{8.90\%} & 0.12\% & 0.00\% & 0.45\% & 0.77\% & NaN \\
\bottomrule
\end{tabular}
}
\caption{Pairwise misclassification rates (\%). Values are means from 5-fold stratified cross-validation using logistic regression on the top 50 PCs of held-out test sets.}
\label{tab:kang_separability}
\end{table}

\subsection{Neural Architectures and training parameters}
\label{app:real_arch}

This section details the neural network architectures, training hyper-parameters, and general experimental settings for scCBGM and the Flow Matching (FM) models used in our experiments. When possible for each model, gene expression data was reduced to 128 principal components (PCs), which were computed exclusively from the training data to prevent data leakage.

\paragraph{scCBGM Model}
The scCBGM encoder consisted of $4$ hidden layers, each with $1024$ neurons. These layers incorporated residual connections, layer normalization, and dropout ($p=0.1$). The encoder mapped gene expression (represented by 128 PCs) to a latent space of $2\times128$ dimensions (128 for $\mu$ and for $\sigma$). From the sampled latent variable $z\sim \mathcal{N}(\mu,\sigma)$, the concept bottleneck layers produced the known concepts (whose dimensions varied by experiment) and a default of $128$ unknown concepts. The decoder mirrored the encoder's structure with 4 layers of $1024$ neurons, layer normalization, and dropout. It utilized residual connections for processing unknown concepts and direct skip connections for known concepts, as described in the main text.

scCBGM was trained for $200$ epochs with a batch size of 128. Optimization was performed using Adam, starting with a learning rate of $5\cdot10^{-4}$ and an exponential learning rate scheduler with a decay rate of $0.997$. The loss function coefficients were set as follows: concept loss at $0.1$, orthogonality loss (for cross-covariance) at $0.5$, and the KL divergence for $\mu$ and $\sigma$ at $1\cdot10^{-5}$.

\paragraph{Flow Matching (FM) Models}
Both scCBGM-FM and Vanilla FM models employed a 4-layer fully-connected network, each layer comprising 1024 neurons. Similar to scCBGM, these networks included residual connections, layer normalization, and dropout. Covariates, which were either scCBGM's known and unknown concepts (for scCBGM-FM) or the ground-truth known concepts (for Vanilla FM), were fed into the model via a learned embedding layer designed to match the hidden dimension of 1024. All FM models were trained for 200 epochs with a batch size of 128. Optimization used Adam, with an initial learning rate of $3\cdot10^{-4} 
$ and an exponential learning rate scheduler with a decay rate of $0.997$. Classifier-free guidance was not utilized, as no significant performance benefit was observed.

\paragraph{Baseline Models} We also implemented a conditional VAE (cVAE) for ablation purposes. To ensure a fair comparison, these models utilized the same backbone architecture and training hyperparameters as the CBGM model described above. All other external baseline models, including biolord, CINEMTA-OT, and scGen, were trained using their respective default parameters as specified in their original implementations.

\paragraph{Reconstruction Validation} We further validated the ability of these architectures to faithfully reconstruct single-cell data. As shown in Figure~\ref{fig:decode_accuracy}, scCBGM and scCBGM-FM reliably decode the dataset, preserving both global structure and local density. In contrast, Vanilla-FM exhibits poor reconstruction fidelity, demonstrating the limitations of conditioning solely on coarse labels compared to our granular latent approach.For Vanilla-FM, reconstruction was evaluated by generating a sample from random noise conditioned on the cell's observed concepts and comparing it to the original expression profile. As expected, since this model lacks access to cell-specific identity, it yields high MSE despite generating biologically plausible samples. This baseline highlights the necessity of the scCBGM latent representation for preserving individual cell identity during the encoding-decoding process.

\begin{figure}[htbp!]
    \centering
    \includegraphics[width=0.8\linewidth]{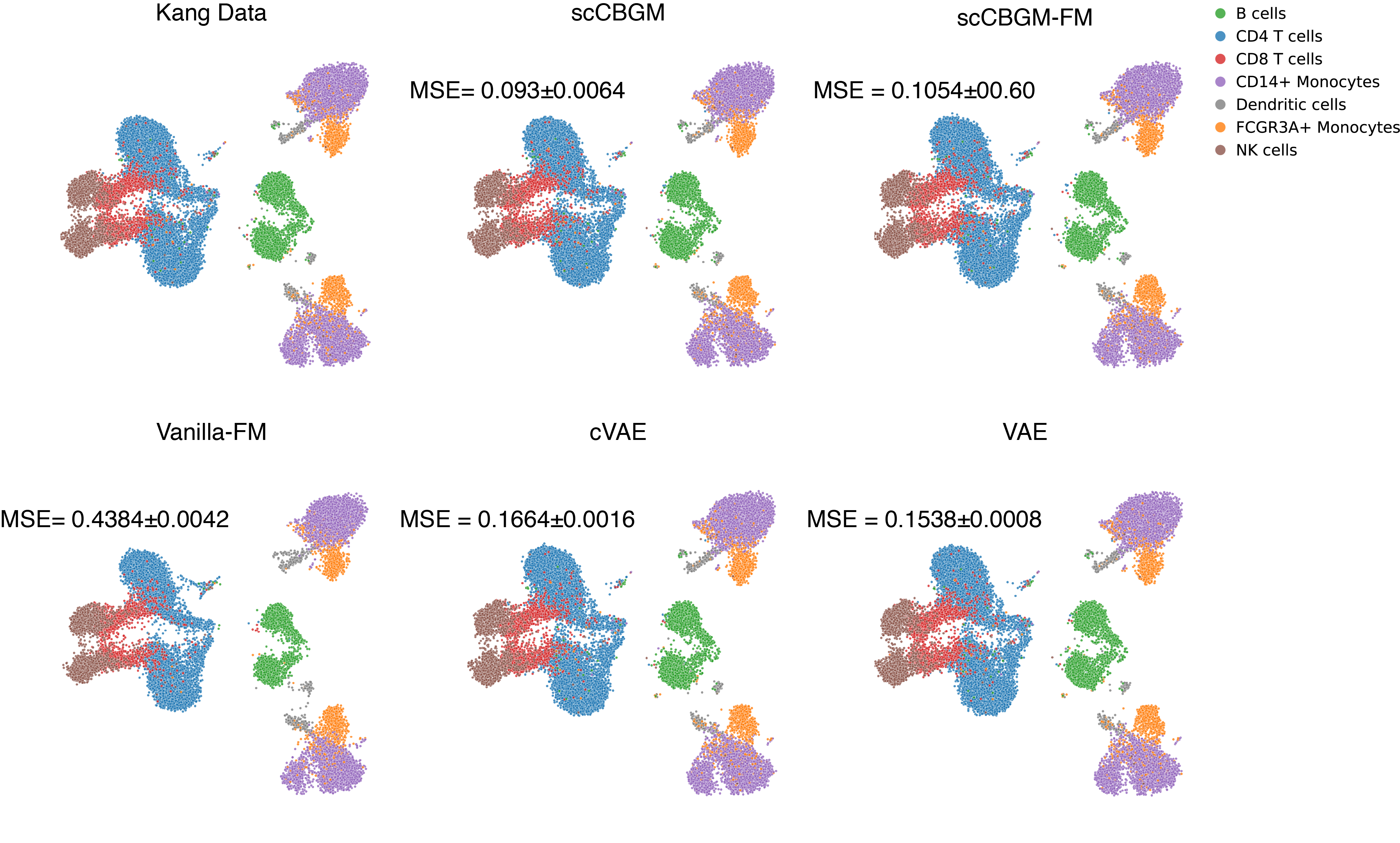}
    \caption{\textbf{Qualitative and quantitative comparison of data reconstruction.} 
    UMAP visualizations of the Kang dataset (top-left) and reconstructions from five generative models. Points are colored by cell type. scCBGM and scCBGM-FM successfully reproduce the data distribution with low test-set MSE. Vanilla-FM shows poor reconstruction fidelity due to a lack of granular conditioning, while cVAE and VAE achieve moderate accuracy. MSE values are means and standard deviations over 5 random initialization seeds for each model.}
    \label{fig:decode_accuracy}
\end{figure}

\subsection{Benchmark on synthetic data}
\label{app:exp:synth}
For the synthetic data benchmark, we used the same architectural design and hyperparameters as described in Appendix~\ref{app:real_arch}. We ran the experiments on the three datasets outlined in Appendix~\ref{app:synth:synt_datasets}, each with 5 interventions and 4 random seeds.

\begin{table*}[htbp!]
\centering
\begin{tabular}{lccc}
\toprule
 & Synthetic 1 & Synthetic 2 & Synthetic 3 \\
\midrule
scCBGM & 0.0388$\pm$0.0053 & 0.0318$\pm$0.0071 & 0.0407$\pm$0.0027 \\
scCBGM-FM (edit) & \textbf{0.0025$\pm$0.0016} & \textbf{0.0028$\pm$0.0014} & \textbf{0.0039$\pm$0.0023} \\
scCBGM-FM (decode) & 0.0539$\pm$0.0039 & 0.0444$\pm$0.0082 & 0.0594$\pm$0.0092 \\
Vanilla-FM (edit) & 0.0031$\pm$0.0014 & 0.0049$\pm$0.0016 & 0.0049$\pm$0.0027 \\
Vanilla-FM (decode) & 0.0756$\pm$0.0014 & 0.0769$\pm$0.0014 & 0.0773$\pm$0.0041 \\
CINEMA-OT & 0.0390$\pm$0.0025 & 0.0393$\pm$0.0012 & 0.0400$\pm$0.0034 \\
biolord & 0.0402$\pm$0.0014 & 0.0411$\pm$0.0021 & 0.0428$\pm$0.0021 \\
\bottomrule
\end{tabular}
\caption{Comparison of methods on synthetic data. We evaluate in PCA space to account for different normalization strategies, so scCBGM values differ from Table \ref{tab:ablation}. Values are MSE (mean ± std), averaged over interventions (5) and seeds (4).}
\label{tab:syn_benchmark_sota}
\end{table*}

\subsection{Performance and number of concepts}
\label{app:performance_by_concept}

While only the concepts on which one intends to \textit{intervene} need to be explicitly characterized in the \textit{known concept layer}, the total number of concepts can increase rapidly when modeling complex biological systems. To further assess how \method scales with the number of concepts, we generated three additional synthetic datasets using the same hyperparameters as \textit{Synthetic~1} (see Appendix~\ref{app:synth:synt_datasets}), varying only the number of concepts. This setup enables a controlled evaluation of how concept dimensionality influences model performance.  

As shown in Table~\ref{tab:performance_by_concept}, performance decreases only marginally---by less than 2\%---as the number of concepts increases. This trend is consistent with prior observations in the literature~\cite{abdelsalam2024concept}, indicating that the model remains stable and effective even as the concept space expands.

\begin{table}[ht!]
\centering
\begin{tabular}{cc}
\toprule
concepts & MSE \\
\midrule
5 & 0.19617 $\pm$ 0.00195 \\
20 & 0.19597 $\pm$ 0.00089 \\
100 & 0.19846 $\pm$ 0.00081 \\
250 & 0.19809 $\pm$ 0.00053 \\
\bottomrule
\end{tabular}
\caption{Performance under varying numbers of concepts. Values are reported as mean~$\pm$~standard deviation, computed across three random seeds and five interventions. No noise was introduced in the concept annotations.}
\label{tab:performance_by_concept}
\end{table}

\newpage
\subsection{Concept leakage}
\label{app:concept_leakage}

The \textbf{cross-covariance loss} between the known and unknown concepts is introduced to prevent information leakage between the two layers. This effect is supported by our results in Section~\ref{sec:experiments} and the theoretical justification in Appendix~\ref{app:decoupling_proof}. 
To examine this question more directly, we designed an experiment that mimics a realistic scenario in which relevant concepts are \textit{not} included among the known ones and must instead be captured by the unknown layer, followed by an examination of how intervening on the known concepts impacted the state of the unknown ones.

In this experiment, we used the \textbf{Kang et al.} dataset (Appendix~\ref{app:kang_data}). Only the \textit{binary stimulation concept} was included in the model --- i.e., cell-type information was \textit{not} encoded in the known layer. The data was split into a 60/40 train/test partition. 
We trained \textsc{scCBGM} with the single stimulation concept, along with two separate classifiers: 
(i) a one-vs-all classifier predicting \textbf{cell type} from the learned expression representation, and 
(ii) a binary classifier predicting whether a cell was \textbf{stimulated} or \textbf{control}.

\begin{wraptable}{r}{0.42\textwidth} 
\centering
\footnotesize 
\setlength{\tabcolsep}{2pt} 

\begin{tabular}{lcccc} 
\toprule
 & \multicolumn{2}{c}{f(cell type)} & \multicolumn{2}{c}{f(stim)} \\
 & Bef. & Aft. & Bef. & Aft. \\
\midrule
B cells & 0.983 & 0.977 & 0.004 & 1.000 \\
CD14+ Mono. & 0.971 & 0.938 & 0.002 & 1.000 \\
CD4 T cells & 0.985 & 0.988 & 0.009 & 1.000 \\
CD8 T cells & 0.718 & 0.691 & 0.014 & 1.000 \\
FCGR3A+ Mon. & 0.750 & 0.777 & 0.000 & 1.000 \\
NK cells & 0.902 & 0.945 & 0.022 & 1.000 \\
\bottomrule
\end{tabular}

\caption{Mean classifier scores before/after intervention (test data).}
\label{tab:concept_leakage}
\end{wraptable}

For evaluation, we took the test data and, within each cell type,
\textit{intervened} on all unstimulated (control) cells by setting their
stimulation state to ``on.'' We then applied the two classifiers---trained on
the original training data---to predict cell type and stimulation status for
each cell \textit{before} and \textit{after} the intervention.

If concept leakage has been effectively prevented, the \textbf{stimulation
predictions} should switch from 0 to 1, while the \textbf{cell-type predictions}
should remain stable. This behavior is clearly observed in
Figure~\ref{fig:concept_leakage}, where the stimulation scores shift from low to
high after intervention, whereas the cell-type scores remain largely unchanged.
The corresponding quantitative results are summarized in
Table~\ref{tab:concept_leakage}.

\begin{figure}[ht!]
    \centering
    \includegraphics[width=0.8\linewidth]{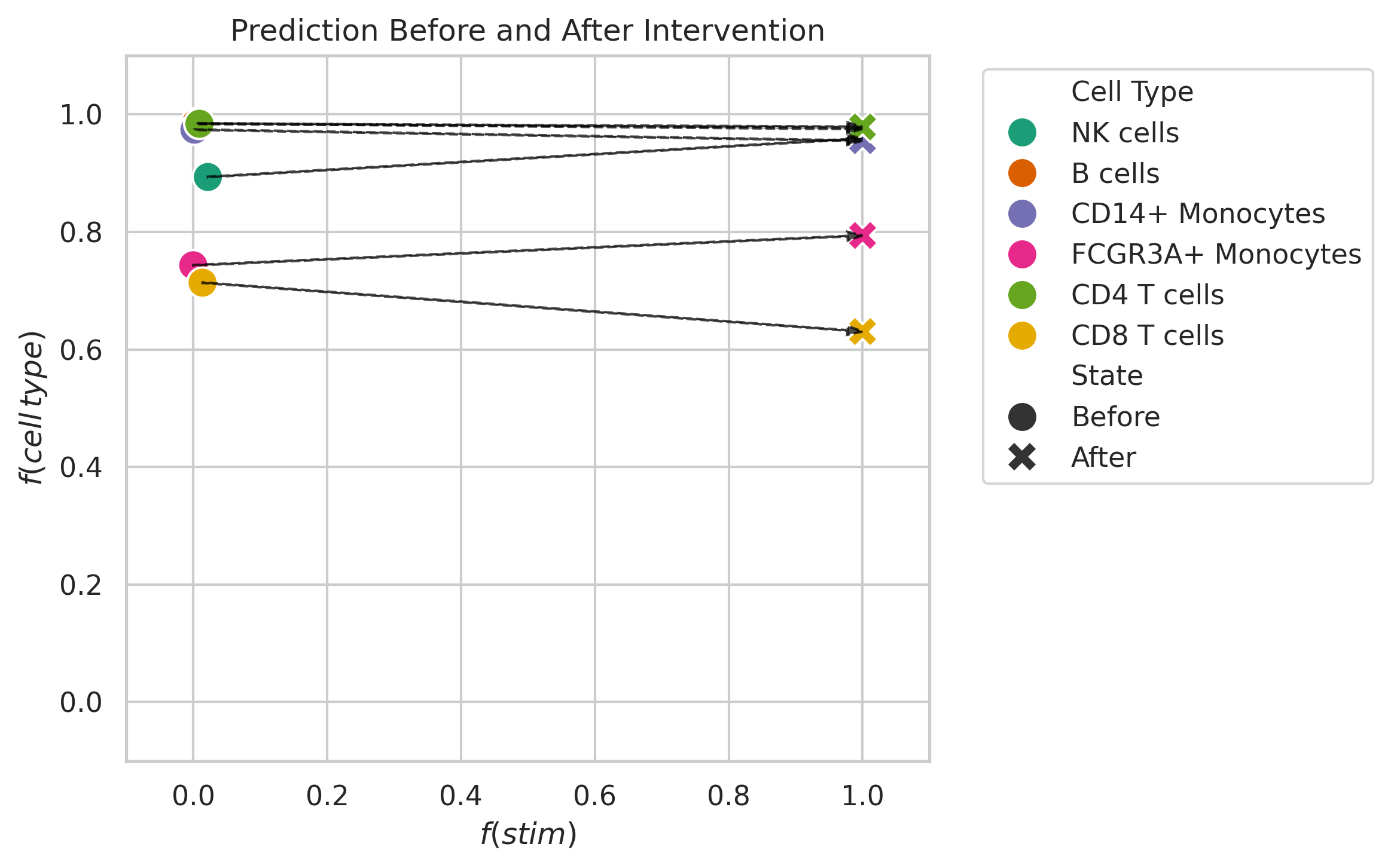}
    \caption{Results from classifiers trained to predict cell type and stimuli applied to the test data before and after intervention. Values are reported as mean across all cells in the test data. Each color represents a cell type. Good results are straight lines from left to right.}
    \label{fig:concept_leakage}
\end{figure}

Furthermore, we performed an isolation experiment on the Kang~et~al.\ dataset using the stimulation label. Specifically, we compared two \method~models: one where stimulation was included as a known concept, and another where it was excluded. A logistic regression classifier (via 5-fold cross-validation) was trained to predict the stimulation state from the unknown embeddings. When stimulation is explicitly modeled, the classifier's performance drops from an AUROC of 0.9988 to 0.5580, \ref{tab:isolation_experiment}. This stark reduction confirms that the known categorical variation is effectively isolated and decoupled from the residual unknown layer.

\begin{table}[htbp]
\centering
\footnotesize 
\setlength{\tabcolsep}{6pt} 

\begin{tabular}{lcccccc}
\toprule
 & \textbf{Fold 1} & \textbf{Fold 2} & \textbf{Fold 3} & \textbf{Fold 4} & \textbf{Fold 5} & \textbf{Mean $\pm$ Std} \\
\midrule
Stim in Known   & 0.572 & 0.577 & 0.555 & 0.545 & 0.541 & 0.5580 $\pm$ 0.0160 \\
Stim in Unknown & 0.998 & 1.000 & 0.998 & 0.999 & 0.999 & 0.9988 $\pm$ 0.0008 \\
\bottomrule
\end{tabular}
\caption{Classifier performance (AUROC) across a 5-fold cross-validation isolation experiment.}
\label{tab:isolation_experiment}
\end{table}

\subsection{Ablations}
\subsubsection{Architecture Ablations Synthetic Data}

\begin{table}[ht!]
        \centering
 \begin{tabular}{lccccc}
\toprule
Model & Bottleneck & Skip & $\mathcal{L}$ & MSE \\
\midrule
 & CBM & \cmark & $\mathcal{L}_{cosine}$ & 0.1989$\pm$0.00023 \\
 & CEM & \cmark & $\mathcal{L}_{cc}$ & 0.19865$\pm$0.00023 \\
 & CEM & \cmark & $\mathcal{L}_{cosine}$ & 0.19816$\pm$0.00022 \\
 & CEM & \xmark & $\mathcal{L}_{cc}$ & 0.1981$\pm$0.0002 \\
CBGM & CEM & \xmark & $\mathcal{L}_{cosine}$ & 0.19804$\pm$0.00021 \\
 & CBM & \xmark & $\mathcal{L}_{cosine}$ & 0.19791$\pm$0.00021 \\
 & CBM & \xmark & $\mathcal{L}_{cc}$ & 0.19736$\pm$0.0002 \\
\hline
\textbf{scCBGM} & \textbf{CBM} & \textbf{\cmark} & \textbf{$\mathcal{L}_{cc}$} & \textbf{0.19655$\pm$0.00019} \\
\bottomrule
\end{tabular}
        \caption{Ablation study over the skip connections and cross-covariance loss. MSE between true and predicted counterfactuals, averaged over datasets (3), interventions (3), noise levels (2), types of noise (5), and seeds (2). We report values as Mean $\pm$ SEM.}
        \label{tab:ablation}
    \end{table}

We compare the contribution of each individual component across a $2 \times 2 \times 2$ design space of Bottleneck (CEM vs. CBM), Skip Connection (yes or no), Loss (cross-covariance vs. cosine). On the synthetic data we are able to do cell-level evaluation, something that real data does not permit. We evaluate each model across: 3 different datasets, 5 different types of interventions (in each dataset), 2 random seeds, 5 different types of noise, and 2 different levels of noise. For comparable results, we used the exact same hyperparameters across all configurations except for the ablated components. However, we highlight that the models using a cosine loss are restricted to using the same number of known as unknown concepts, a limit not imposed on the cross-covariance loss.

\subsubsection{Architecture Ablations on Real Data}
\label{app:real_ablations}

In Table~\ref{tab:rwd_ablations}, we report the result of our model architecture ablations on the~\citet{kang} dataset. This experiment confirms the need for the disentanglement loss and the skip connection in \method.

\begin{table}[htpb!]
\centering
\caption{rMMD over all subtypes on the~\citet{kang} dataset, for different ablations of \method ordered mean in descending order. We evaluate different choices of bottleneck architecture, the presence of skip connection and different disentanglement losses (cosine similarity, cross-covariance, or no disentanglement loss (left blank)). We report values as Mean $\pm$ SEM, and report the average rMMD over all subtypes for each method (\textit{Mean} column). Best model for each subtype is bolded. The \method and CBGM configurations used in our benchmarks are highlighted in the \emph{Method} column.}
\label{tab:rwd_ablations}
\begin{adjustbox}{width=\textwidth}
\begin{tabular}{cccc|cc|cc|ccc|cc|cc|cc|c|c}
\toprule
 & & & & \multicolumn{2}{c|}{B Cells} & \multicolumn{2}{c|}{Cd14} & \multicolumn{3}{c|}{Cd4} & \multicolumn{2}{c|}{Cd8} & \multicolumn{2}{c|}{Dendritic} & \multicolumn{2}{c|}{Fcgr3A} & Nk & \\
Method & Bottleneck & Skip & $\mathcal{L}$ & subtype 0 & subtype 1 & subtype 0 & subtype 1 & subtype 0 & subtype 1 & subtype 2 & subtype 0 & subtype 1 & subtype 0 & subtype 1 & subtype 0 & subtype 1 & subtype 0 & Mean \\
\midrule
 & CEM & \cmark & $\mathcal{L}_{cc}$ & 2.859 $\pm$ 0.560 & 1.484 $\pm$ 0.348 & 11.505 $\pm$ 1.676 & 9.367 $\pm$ 0.917 & 4.294 $\pm$ 0.864 & 6.938 $\pm$ 2.221 & 2.264 $\pm$ 0.511 & 2.575 $\pm$ 0.174 & 5.936 $\pm$ 1.966 & 2.159 $\pm$ 0.244 & 1.779 $\pm$ 0.164 & 1.663 $\pm$ 0.073 & 24.590 $\pm$ 1.246 & 1.911 $\pm$ 0.094 & 5.666 \\
 & CEM & \cmark & $\mathcal{L}_{cosine}$ & 1.456 $\pm$ 0.266 & 1.154 $\pm$ 0.143 & 8.753 $\pm$ 0.244 & 8.449 $\pm$ 0.728 & 3.103 $\pm$ 0.031 & 4.401 $\pm$ 0.609 & 1.178 $\pm$ 0.125 & 2.604 $\pm$ 0.169 & 2.676 $\pm$ 0.350 & 1.133 $\pm$ 0.163 & 1.603 $\pm$ 0.164 & 1.777 $\pm$ 0.087 & 20.765 $\pm$ 1.889 & 1.917 $\pm$ 0.517 & 4.355 \\
 & CEM & \cmark &  & 1.257 $\pm$ 0.063 & 0.930 $\pm$ 0.126 & 9.026 $\pm$ 0.429 & 7.466 $\pm$ 0.548 & 2.674 $\pm$ 0.441 & 3.970 $\pm$ 0.246 & 1.133 $\pm$ 0.032 & 2.271 $\pm$ 0.085 & 2.434 $\pm$ 0.275 & 1.394 $\pm$ 0.120 & 1.762 $\pm$ 0.244 & 1.159 $\pm$ 0.061 & 22.574 $\pm$ 3.052 & 1.391 $\pm$ 0.187 & 4.246 \\
 & CBM & \xmark &  & 1.078 $\pm$ 0.011 & 0.758 $\pm$ 0.007 & 9.634 $\pm$ 0.078 & 7.544 $\pm$ 0.100 & 2.081 $\pm$ 0.010 & 3.321 $\pm$ 0.040 & 0.809 $\pm$ 0.022 & 1.758 $\pm$ 0.007 & 1.674 $\pm$ 0.012 & 1.436 $\pm$ 0.010 & 1.449 $\pm$ 0.012 & 1.684 $\pm$ 0.015 & 22.232 $\pm$ 0.202 & 1.067 $\pm$ 0.008 & 4.038 \\
 & CBM & \cmark & $\mathcal{L}_{cosine}$ & 1.111 $\pm$ 0.123 & 1.134 $\pm$ 0.247 & 4.007 $\pm$ 2.175 & 1.647 $\pm$ 0.134 & 1.664 $\pm$ 0.264 & 10.931 $\pm$ 3.377 & 1.276 $\pm$ 0.028 & 3.942 $\pm$ 1.281 & 2.650 $\pm$ 0.258 & 1.324 $\pm$ 0.221 & 1.527 $\pm$ 0.246 & 1.380 $\pm$ 0.296 & 14.633 $\pm$ 7.135 & 5.634 $\pm$ 0.890 & 3.776 \\
 & CBM & \cmark &  & 0.977 $\pm$ 0.012 & 0.639 $\pm$ 0.008 & 9.001 $\pm$ 0.045 & 6.776 $\pm$ 0.056 & 1.719 $\pm$ 0.072 & 2.797 $\pm$ 0.066 & 0.703 $\pm$ 0.003 & 1.578 $\pm$ 0.023 & 1.372 $\pm$ 0.022 & 1.336 $\pm$ 0.015 & 1.310 $\pm$ 0.022 & 1.467 $\pm$ 0.021 & 20.841 $\pm$ 0.071 & 0.953 $\pm$ 0.018 & 3.676 \\
 & CEM & \xmark &  & 1.060 $\pm$ 0.062 & 0.877 $\pm$ 0.048 & 7.820 $\pm$ 0.228 & 6.365 $\pm$ 0.218 & 2.036 $\pm$ 0.135 & 3.496 $\pm$ 0.164 & 1.058 $\pm$ 0.072 & 2.032 $\pm$ 0.136 & 1.896 $\pm$ 0.083 & 1.336 $\pm$ 0.028 & 1.741 $\pm$ 0.093 & 1.387 $\pm$ 0.119 & 17.643 $\pm$ 0.650 & 1.310 $\pm$ 0.065 & 3.575 \\
 & CEM & \xmark & $\mathcal{L}_{cc}$ & 1.137 $\pm$ 0.132 & 0.823 $\pm$ 0.066 & 6.488 $\pm$ 0.284 & 5.767 $\pm$ 0.103 & 1.703 $\pm$ 0.071 & 3.088 $\pm$ 0.214 & 0.937 $\pm$ 0.093 & 1.999 $\pm$ 0.061 & 1.723 $\pm$ 0.051 & 1.196 $\pm$ 0.032 & 1.560 $\pm$ 0.153 & 1.400 $\pm$ 0.043 & 16.781 $\pm$ 0.306 & 1.793 $\pm$ 0.092 & 3.314 \\
CBGM & CEM & \xmark & $\mathcal{L}_{cosine}$ & 1.057 $\pm$ 0.052 & 0.866 $\pm$ 0.056 & 7.270 $\pm$ 0.435 & 5.371 $\pm$ 0.412 & 1.761 $\pm$ 0.084 & 3.268 $\pm$ 0.295 & 1.152 $\pm$ 0.029 & 1.829 $\pm$ 0.062 & 2.024 $\pm$ 0.185 & 1.225 $\pm$ 0.085 & 1.621 $\pm$ 0.074 & 1.305 $\pm$ 0.034 & 14.960 $\pm$ 0.971 & 1.165 $\pm$ 0.222 & 3.205 \\
 & CBM & \xmark & $\mathcal{L}_{cosine}$ & 1.156 $\pm$ 0.339 & 1.148 $\pm$ 0.630 & 5.739 $\pm$ 2.045 & 1.599 $\pm$ 0.340 & 2.371 $\pm$ 1.264 & 1.647 $\pm$ 0.485 & 0.609 $\pm$ 0.030 & 0.928 $\pm$ 0.121 & 1.260 $\pm$ 0.196 & 1.118 $\pm$ 0.298 & 1.943 $\pm$ 0.534 & 1.437 $\pm$ 0.189 & 9.970 $\pm$ 2.096 & 0.791 $\pm$ 0.114 & 2.265 \\
scCBGM & CBM & \cmark & $\mathcal{L}_{cc}$ & \textbf{0.153 $\pm$ 0.004} & \textbf{0.103 $\pm$ 0.005} & 1.195 $\pm$ 0.058 & \textbf{1.069 $\pm$ 0.061} & \textbf{0.155 $\pm$ 0.008} & \textbf{0.279 $\pm$ 0.015} & \textbf{0.076 $\pm$ 0.004} & \textbf{0.182 $\pm$ 0.010} & \textbf{0.165 $\pm$ 0.014} & \textbf{0.390 $\pm$ 0.011} & 0.359 $\pm$ 0.039 & \textbf{0.166 $\pm$ 0.007} & \textbf{3.151 $\pm$ 0.318} & 0.921 $\pm$ 0.244 & 0.598 \\
 & CBM & \xmark & $\mathcal{L}_{cc}$ & 0.175 $\pm$ 0.004 & \textbf{0.103 $\pm$ 0.007} & \textbf{0.751 $\pm$ 0.082} & 1.142 $\pm$ 0.030 & 0.201 $\pm$ 0.012 & 0.297 $\pm$ 0.023 & 0.096 $\pm$ 0.011 & 0.221 $\pm$ 0.010 & 0.198 $\pm$ 0.011 & 0.430 $\pm$ 0.018 & \textbf{0.343 $\pm$ 0.059} & 0.202 $\pm$ 0.012 & 3.474 $\pm$ 0.129 & \textbf{0.166 $\pm$ 0.009} & \textbf{0.557} \\
\bottomrule
\end{tabular}
\end{adjustbox}
\end{table}

    \subsubsection{variational auto-encoder vs. auto-encoder}
In figure~\ref{fig:vae_vs_ae}, we present a performance comparison between~\method and an ablated version of~\method where the variational auto-encoder backbone has been replaced with a simple auto-encoder (\method (ae)). we found that the rMMD values on different cell populations of the~\citet{kang} dataset were typically lower (better) with a variational auto-encoder backbone than without.

\begin{figure}[!ht]
    \centering
    \includegraphics[width=\linewidth]{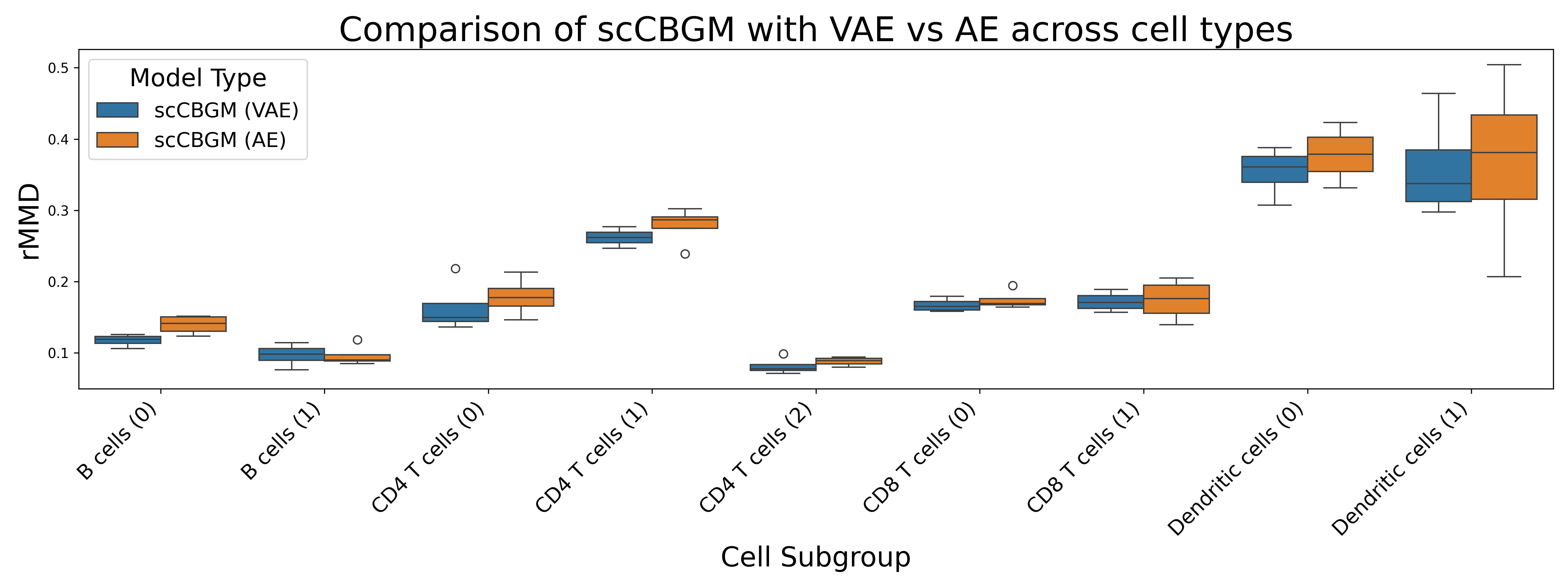}
    \caption{Comparison of rMMD performance between two versions of the auto-encoder backbone of \method: a variational auto-encoder (VAE) and a regular auto-encoder (AE). The boxplots are computed over 4 seeds for each cell population from the~\citet{kang} dataset. Lower rMMD is better.}
    \label{fig:vae_vs_ae}
\end{figure}

\subsection{Additional benchmarking results}
\label{app:additional_benchmark}

\subsubsection{Counter-factual modeling predicts cellular response to perturbation}

Tables~\ref{tab:kang_fid_ratio_aggregated_mean_sem} and~\ref{tab:kang_sinkhorn_div_w2_ratio_aggregated_mean_sem} show results on the~\citet{kang} dataset on the rFID and rSD metrics. Similiarly to our rMMD results in the main text, we find that our model outperforms other baselines in 5 out of 7 experiments.

\begin{table}[htbp!]
\caption{rFID (Mean ± SEM) per cell group for different models (\bestres{\textbf{best}}, \secondbest{second-best}, and \thirdbest{third-best} bolded) in the ~\citet{kang} dataset.}
\label{tab:kang_fid_ratio_aggregated_mean_sem}
\begin{adjustbox}{width=\textwidth}
\begin{tabular}{lccccccc}
\toprule
\textbf{Model} & B cells & \makecell{\textbf{T cells} \\ \textbf{(CD4)}} & \makecell{\textbf{T cells} \\ \textbf{(CD8)}} & \makecell{\textbf{Monocytes} \\ \textbf{(FCGR3A)}} & \makecell{\textbf{Monocytes} \\ \textbf{(CD14)}} & Dendritic Cells & NK cells \\
\midrule
scCBGM & \thirdbest{0.265 $\pm$ 0.003} & \thirdbest{0.277 $\pm$ 0.003} & \thirdbest{0.515 $\pm$ 0.004} & \thirdbest{1.135 $\pm$ 0.044} & \thirdbest{1.719 $\pm$ 0.032} & \thirdbest{0.702 $\pm$ 0.010} & 1.230 $\pm$ 0.285 \\
scCBGM-FM (decode) & \bestres{\textbf{0.247 $\pm$ 0.002}} & \secondbest{0.253 $\pm$ 0.001} & \bestres{\textbf{0.457 $\pm$ 0.003}} & \secondbest{0.996 $\pm$ 0.040} & 1.769 $\pm$ 0.051 & \bestres{\textbf{0.665 $\pm$ 0.009}} & 0.215 $\pm$ 0.006 \\
scCBGM-FM (edit) & \secondbest{0.248 $\pm$ 0.002} & \bestres{\textbf{0.241 $\pm$ 0.003}} & \secondbest{0.472 $\pm$ 0.009} & \bestres{\textbf{0.971 $\pm$ 0.018}} & 1.723 $\pm$ 0.038 & \secondbest{0.670 $\pm$ 0.005} & \secondbest{0.202 $\pm$ 0.001} \\
\midrule
CBGM & 1.086 $\pm$ 0.053 & 2.031 $\pm$ 0.059 & 2.142 $\pm$ 0.041 & 5.244 $\pm$ 0.112 & 5.353 $\pm$ 0.123 & 1.634 $\pm$ 0.031 & 1.664 $\pm$ 0.174 \\
\midrule
Vanilla-FM (decode) & 0.937 $\pm$ 0.013 & 0.953 $\pm$ 0.002 & 0.999 $\pm$ 0.013 & 1.995 $\pm$ 0.009 & \secondbest{1.589 $\pm$ 0.043} & 2.017 $\pm$ 0.026 & \thirdbest{0.215 $\pm$ 0.001} \\
Vanilla-FM (edit) & 0.551 $\pm$ 0.012 & 0.452 $\pm$ 0.007 & 0.584 $\pm$ 0.007 & 1.233 $\pm$ 0.039 & \bestres{\textbf{1.341 $\pm$ 0.011}} & 1.300 $\pm$ 0.033 & \bestres{\textbf{0.177 $\pm$ 0.003}} \\
\midrule
biolord & 1.730 $\pm$ 0.001 & 3.219 $\pm$ 0.001 & 3.039 $\pm$ 0.006 & 3.978 $\pm$ 0.009 & 1.855 $\pm$ 0.003 & 1.979 $\pm$ 0.002 & 1.949 $\pm$ 0.001 \\
biolord-FM & 1.165 $\pm$ 0.003 & 1.782 $\pm$ 0.002 & 1.613 $\pm$ 0.004 & 5.446 $\pm$ 0.011 & 7.439 $\pm$ 0.020 & 1.986 $\pm$ 0.004 & 1.049 $\pm$ 0.004 \\
Cinema-OT & 1.500 $\pm$ 0.001 & 3.563 $\pm$ 0.001 & 3.077 $\pm$ 0.003 & 3.550 $\pm$ 0.004 & 4.637 $\pm$ 0.002 & 0.928 $\pm$ 0.000 & 2.565 $\pm$ 0.003 \\
scGen & 1.275 $\pm$ 0.003 & 2.732 $\pm$ 0.010 & 2.833 $\pm$ 0.013 & 2.567 $\pm$ 0.014 & 2.801 $\pm$ 0.081 & 0.883 $\pm$ 0.010 & 1.966 $\pm$ 0.022 \\
CVAE & 0.653 $\pm$ 0.006 & 1.002 $\pm$ 0.007 & 1.068 $\pm$ 0.019 & 4.075 $\pm$ 0.016 & 4.553 $\pm$ 0.034 & 1.216 $\pm$ 0.013 & 0.732 $\pm$ 0.007 \\
CVAE-FM (decode) & 0.568 $\pm$ 0.005 & 0.812 $\pm$ 0.006 & 0.948 $\pm$ 0.006 & 3.608 $\pm$ 0.023 & 4.257 $\pm$ 0.029 & 1.109 $\pm$ 0.010 & 0.623 $\pm$ 0.006 \\
CVAE-FM (edit) & 0.563 $\pm$ 0.002 & 0.803 $\pm$ 0.004 & 0.933 $\pm$ 0.003 & 3.723 $\pm$ 0.014 & 4.213 $\pm$ 0.035 & 1.123 $\pm$ 0.007 & 0.600 $\pm$ 0.003 \\
\bottomrule
\end{tabular}
\end{adjustbox}
\end{table}
\begin{table}
\caption{rSD (Mean ± SEM) per cell group for different models (\bestres{\textbf{best}}, \secondbest{second-best}, and \thirdbest{third-best} bolded) in the ~\citet{kang} dataset.}
\label{tab:kang_sinkhorn_div_w2_ratio_aggregated_mean_sem}
\begin{adjustbox}{width=\textwidth}
\begin{tabular}{lccccccc}
\toprule
\textbf{Model} & B cells & \makecell{\textbf{T cells} \\ \textbf{(CD4)}} & \makecell{\textbf{T cells} \\ \textbf{(CD8)}} & \makecell{\textbf{Monocytes} \\ \textbf{(FCGR3A)}} & \makecell{\textbf{Monocytes} \\ \textbf{(CD14)}} & Dendritic Cells & NK cells \\
\midrule
scCBGM & \thirdbest{0.393 $\pm$ 0.004} & \secondbest{0.507 $\pm$ 0.004} & 0.971 $\pm$ 0.006 & \thirdbest{0.968 $\pm$ 0.028} & 1.246 $\pm$ 0.033 & 0.744 $\pm$ 0.010 & 1.135 $\pm$ 0.254 \\
scCBGM-FM (decode) & \secondbest{0.387 $\pm$ 0.014} & \thirdbest{0.515 $\pm$ 0.005} & \bestres{\textbf{0.853 $\pm$ 0.015}} & \secondbest{0.868 $\pm$ 0.023} & 1.312 $\pm$ 0.046 & \thirdbest{0.742 $\pm$ 0.013} & \bestres{\textbf{0.290 $\pm$ 0.023}} \\
scCBGM-FM (edit) & \bestres{\textbf{0.387 $\pm$ 0.002}} & \bestres{\textbf{0.500 $\pm$ 0.002}} & \thirdbest{0.960 $\pm$ 0.013} & \bestres{\textbf{0.838 $\pm$ 0.012}} & 1.261 $\pm$ 0.040 & \secondbest{0.726 $\pm$ 0.003} & \secondbest{0.293 $\pm$ 0.003} \\
\midrule
CBGM & 1.046 $\pm$ 0.057 & 1.824 $\pm$ 0.048 & 1.865 $\pm$ 0.010 & 3.432 $\pm$ 0.024 & 4.401 $\pm$ 0.124 & 1.677 $\pm$ 0.008 & 1.446 $\pm$ 0.130 \\
\midrule
Vanilla-FM (decode) & 1.089 $\pm$ 0.029 & 1.042 $\pm$ 0.003 & 1.336 $\pm$ 0.044 & 1.484 $\pm$ 0.020 & \thirdbest{1.165 $\pm$ 0.038} & 2.062 $\pm$ 0.038 & 0.436 $\pm$ 0.031 \\
Vanilla-FM (edit) & 0.662 $\pm$ 0.011 & 0.750 $\pm$ 0.011 & 1.124 $\pm$ 0.018 & 1.025 $\pm$ 0.019 & \secondbest{0.916 $\pm$ 0.010} & 1.431 $\pm$ 0.032 & \thirdbest{0.318 $\pm$ 0.012} \\
\midrule
biolord & 0.888 $\pm$ 0.002 & 0.897 $\pm$ 0.002 & \secondbest{0.878 $\pm$ 0.001} & 1.405 $\pm$ 0.004 & \bestres{\textbf{0.371 $\pm$ 0.001}} & 1.570 $\pm$ 0.001 & 0.444 $\pm$ 0.002 \\
biolord-FM & 1.283 $\pm$ 0.016 & 1.694 $\pm$ 0.005 & 1.785 $\pm$ 0.021 & 3.765 $\pm$ 0.012 & 6.414 $\pm$ 0.030 & 1.937 $\pm$ 0.009 & 1.114 $\pm$ 0.022 \\
Cinema-OT & 0.840 $\pm$ 0.001 & 1.593 $\pm$ 0.004 & 1.154 $\pm$ 0.003 & 1.591 $\pm$ 0.002 & 3.241 $\pm$ 0.007 & \bestres{\textbf{0.700 $\pm$ 0.001}} & 1.195 $\pm$ 0.007 \\
scGen & 0.967 $\pm$ 0.005 & 1.876 $\pm$ 0.012 & 1.872 $\pm$ 0.007 & 1.454 $\pm$ 0.019 & 1.979 $\pm$ 0.075 & 0.796 $\pm$ 0.010 & 1.366 $\pm$ 0.014 \\
CVAE & 0.730 $\pm$ 0.005 & 1.192 $\pm$ 0.018 & 1.422 $\pm$ 0.030 & 2.914 $\pm$ 0.030 & 3.846 $\pm$ 0.029 & 1.265 $\pm$ 0.010 & 0.767 $\pm$ 0.009 \\
CVAE-FM (decode) & 0.670 $\pm$ 0.005 & 0.959 $\pm$ 0.011 & 1.213 $\pm$ 0.036 & 2.527 $\pm$ 0.032 & 3.601 $\pm$ 0.022 & 1.206 $\pm$ 0.019 & 0.666 $\pm$ 0.018 \\
CVAE-FM (edit) & 0.643 $\pm$ 0.004 & 1.010 $\pm$ 0.008 & 1.309 $\pm$ 0.012 & 2.665 $\pm$ 0.015 & 3.544 $\pm$ 0.036 & 1.181 $\pm$ 0.009 & 0.646 $\pm$ 0.002 \\
\bottomrule
\end{tabular}
\end{adjustbox}
\end{table}

\subsubsection{Compositional generalization with many stimuli} Tables~\ref{tab:cui_fid_ratio_aggregated_mean_sem} and~\ref{tab:cui_sinkhorn_div_w2_ratio_aggregated_mean_sem} show results on rFID and rSD on the~\citet{cui2024dictionary} dataset. Similarly to our rMMD results in the main text, we find that our model outperforms other baselines in 4 out of 7 experiments.

The~\citet{cui2024dictionary} dataset contains 17 cell types with 86 cytokines perturbations, leading to 1479 possible configurations. To keep experiments concise, we focused on seven pairs that the original study identified as inducing significant transcriptional shifts, ensuring a robust signal for evaluation.

For completeness, we still ran an extensive comparison analysis between scCBGM-FM (edit) and CVAE-FM (edit) over all 1479 cell-types - interventions pairs in the dataset with a single seed. Running all baselines was prohibitive (both models require 850 GPU hours each). We thus chose the two most competitive models on the seven pairs from the main analyses. Aggregated results per cell-type (over cytokine perturbations) are presented in Figure~\ref{fig:complete_cui}. Using a paired t-test over all 1479 configurations, scCBGM-FM was found to significantly outperform CVAE-FM across the full intervention space (p-val$<1e^{-10}$) for all metrics.

\begin{table}
\caption{rFID (Mean ± SEM) per cell group for different models (\bestres{\textbf{best}}, \secondbest{second-best}, and \thirdbest{third-best} bolded) in the ~\citet{cui2024dictionary} dataset.}
\label{tab:cui_fid_ratio_aggregated_mean_sem}
\begin{adjustbox}{width=\textwidth}
\begin{tabular}{lccccccc}
\toprule
\textbf{Model} & \makecell{\textbf{T cells}\\\textbf{(Gamma-delta)}} & \makecell{\textbf{T cells}\\\textbf{(CD4)}} & \makecell{\textbf{T cells}\\\textbf{(CD8)}} & \makecell{\textbf{Dendritic}\\\textbf{(cDC2)}} & \makecell{\textbf{Dendritic}\\\textbf{(Langerhans)}} & \makecell{\textbf{Myeloid}\\\textbf{(Macrophages)}} & \makecell{\textbf{Lymphoid}\\\textbf{(NK cells)}} \\
\midrule
scCBGM & \bestres{\textbf{0.688 $\pm$ 0.006}} & \bestres{\textbf{0.523 $\pm$ 0.003}} & \secondbest{0.474 $\pm$ 0.007} & 0.851 $\pm$ 0.007 & \bestres{\textbf{0.162 $\pm$ 0.002}} & \bestres{\textbf{1.263 $\pm$ 0.019}} & \secondbest{1.941 $\pm$ 0.021} \\
scCBGM-FM (decode) & \thirdbest{0.731 $\pm$ 0.007} & \thirdbest{0.555 $\pm$ 0.003} & 0.509 $\pm$ 0.005 & 0.875 $\pm$ 0.008 & 0.170 $\pm$ 0.002 & \secondbest{1.276 $\pm$ 0.004} & \thirdbest{1.956 $\pm$ 0.008} \\
scCBGM-FM (edit) & 0.751 $\pm$ 0.005 & 0.570 $\pm$ 0.003 & 0.508 $\pm$ 0.003 & 0.869 $\pm$ 0.008 & 0.169 $\pm$ 0.001 & 1.319 $\pm$ 0.008 & 1.985 $\pm$ 0.015 \\
\midrule
CBGM & 0.884 $\pm$ 0.091 & 0.588 $\pm$ 0.025 & \thirdbest{0.488 $\pm$ 0.015} & \thirdbest{0.840 $\pm$ 0.056} & \secondbest{0.167 $\pm$ 0.004} & 1.555 $\pm$ 0.160 & 2.335 $\pm$ 0.070 \\\midrule
Vanilla-FM (decode) & 1.527 $\pm$ 0.040 & 1.474 $\pm$ 0.090 & 1.470 $\pm$ 0.057 & 1.223 $\pm$ 0.055 & 1.472 $\pm$ 0.044 & 4.001 $\pm$ 0.253 & 2.124 $\pm$ 0.110 \\
Vanilla-FM (edit) & 0.875 $\pm$ 0.020 & 0.644 $\pm$ 0.009 & 0.510 $\pm$ 0.016 & \bestres{\textbf{0.646 $\pm$ 0.049}} & 0.478 $\pm$ 0.045 & 3.518 $\pm$ 0.169 & \bestres{\textbf{1.014 $\pm$ 0.044}} \\\midrule
biolord & 2.164 $\pm$ 0.003 & 1.412 $\pm$ 0.001 & 1.265 $\pm$ 0.001 & 1.223 $\pm$ 0.000 & 0.999 $\pm$ 0.000 & 4.400 $\pm$ 0.013 & 2.612 $\pm$ 0.006 \\
biolord-FM & 1.357 $\pm$ 0.007 & 1.117 $\pm$ 0.005 & 1.056 $\pm$ 0.009 & 1.577 $\pm$ 0.003 & 1.125 $\pm$ 0.010 & 1.859 $\pm$ 0.009 & 2.634 $\pm$ 0.019 \\
Cinema-OT & 1.655 $\pm$ 0.002 & 1.056 $\pm$ 0.001 & 1.127 $\pm$ 0.001 & 0.976 $\pm$ 0.001 & 0.816 $\pm$ 0.000 & 4.187 $\pm$ 0.015 & 5.729 $\pm$ 0.007 \\
scGen & 1.601 $\pm$ 0.011 & 1.405 $\pm$ 0.009 & 1.070 $\pm$ 0.013 & \secondbest{0.830 $\pm$ 0.004} & 0.291 $\pm$ 0.002 & 2.053 $\pm$ 0.005 & 2.767 $\pm$ 0.008 \\
CVAE & \secondbest{0.730 $\pm$ 0.004} & \secondbest{0.525 $\pm$ 0.002} & \bestres{\textbf{0.470 $\pm$ 0.002}} & 0.930 $\pm$ 0.003 & \thirdbest{0.167 $\pm$ 0.001} & 1.312 $\pm$ 0.006 & 2.047 $\pm$ 0.016 \\
CVAE-FM (decode) & 0.754 $\pm$ 0.005 & 0.559 $\pm$ 0.002 & 0.511 $\pm$ 0.004 & 0.929 $\pm$ 0.003 & 0.176 $\pm$ 0.001 & \thirdbest{1.288 $\pm$ 0.008} & 2.048 $\pm$ 0.006 \\
CVAE-FM (edit) & 0.778 $\pm$ 0.002 & 0.571 $\pm$ 0.004 & 0.507 $\pm$ 0.005 & 0.938 $\pm$ 0.001 & 0.177 $\pm$ 0.000 & 1.351 $\pm$ 0.002 & 2.093 $\pm$ 0.005 \\
\bottomrule
\end{tabular}
\end{adjustbox}
\end{table}
\begin{table}
\caption{rSD (Mean ± SEM) per cell group for different models (\bestres{\textbf{best}}, \secondbest{second-best}, and \thirdbest{third-best} bolded) in the ~\citet{cui2024dictionary} dataset.}
\label{tab:cui_sinkhorn_div_w2_ratio_aggregated_mean_sem}
\begin{adjustbox}{width=\textwidth}
\begin{tabular}{lccccccc}
\toprule
\textbf{Model} & \makecell{\textbf{T cells}\\\textbf{(Gamma-delta)}} & \makecell{\textbf{T cells}\\\textbf{(CD4)}} & \makecell{\textbf{T cells}\\\textbf{(CD8)}} & \makecell{\textbf{Dendritic}\\\textbf{(cDC2)}} & \makecell{\textbf{Dendritic}\\\textbf{(Langerhans)}} & \makecell{\textbf{Myeloid}\\\textbf{(Macrophages)}} & \makecell{\textbf{Lymphoid}\\\textbf{(NK cells)}} \\
\midrule
scCBGM & \bestres{\textbf{0.781 $\pm$ 0.018}} & \bestres{\textbf{0.684 $\pm$ 0.004}} & \thirdbest{0.704 $\pm$ 0.011} & 0.852 $\pm$ 0.012 & \bestres{\textbf{0.194 $\pm$ 0.005}} & \bestres{\textbf{1.301 $\pm$ 0.024}} & \thirdbest{1.898 $\pm$ 0.015} \\
scCBGM-FM (decode) & \secondbest{0.822 $\pm$ 0.010} & 0.757 $\pm$ 0.011 & 0.808 $\pm$ 0.013 & 0.866 $\pm$ 0.010 & 0.204 $\pm$ 0.002 & \secondbest{1.326 $\pm$ 0.015} & 1.915 $\pm$ 0.014 \\
scCBGM-FM (edit) & 0.858 $\pm$ 0.003 & 0.756 $\pm$ 0.001 & 0.747 $\pm$ 0.003 & 0.872 $\pm$ 0.007 & 0.208 $\pm$ 0.002 & \thirdbest{1.343 $\pm$ 0.008} & 1.984 $\pm$ 0.013 \\\midrule
CBGM & 0.936 $\pm$ 0.122 & 0.767 $\pm$ 0.033 & \bestres{\textbf{0.680 $\pm$ 0.029}} & \thirdbest{0.836 $\pm$ 0.056} & \secondbest{0.196 $\pm$ 0.002} & 1.520 $\pm$ 0.101 & 2.272 $\pm$ 0.049 \\\midrule
Vanilla-FM (decode) & 1.503 $\pm$ 0.024 & 1.678 $\pm$ 0.038 & 2.127 $\pm$ 0.069 & 1.212 $\pm$ 0.049 & 1.454 $\pm$ 0.035 & 3.455 $\pm$ 0.155 & 2.040 $\pm$ 0.121 \\
Vanilla-FM (edit) & 1.005 $\pm$ 0.011 & 0.873 $\pm$ 0.020 & 0.819 $\pm$ 0.029 & \bestres{\textbf{0.636 $\pm$ 0.045}} & 0.511 $\pm$ 0.050 & 3.105 $\pm$ 0.108 & \bestres{\textbf{0.964 $\pm$ 0.038}} \\\midrule
biolord & 1.099 $\pm$ 0.001 & 0.778 $\pm$ 0.001 & 0.785 $\pm$ 0.002 & 0.897 $\pm$ 0.001 & 0.795 $\pm$ 0.000 & 2.287 $\pm$ 0.007 & \secondbest{1.537 $\pm$ 0.005} \\
biolord-FM & 1.431 $\pm$ 0.019 & 1.359 $\pm$ 0.022 & 1.526 $\pm$ 0.033 & 1.530 $\pm$ 0.004 & 1.132 $\pm$ 0.010 & 2.019 $\pm$ 0.016 & 2.525 $\pm$ 0.030 \\
Cinema-OT & 1.240 $\pm$ 0.003 & 0.869 $\pm$ 0.002 & 1.105 $\pm$ 0.002 & 0.918 $\pm$ 0.001 & 0.792 $\pm$ 0.001 & 3.138 $\pm$ 0.010 & 4.971 $\pm$ 0.028 \\
scGen & 0.999 $\pm$ 0.010 & 1.044 $\pm$ 0.009 & 0.821 $\pm$ 0.017 & \secondbest{0.689 $\pm$ 0.002} & 0.205 $\pm$ 0.002 & 1.496 $\pm$ 0.007 & 2.119 $\pm$ 0.009 \\
CVAE & \thirdbest{0.852 $\pm$ 0.005} & \secondbest{0.717 $\pm$ 0.001} & \secondbest{0.688 $\pm$ 0.005} & 0.928 $\pm$ 0.003 & \thirdbest{0.204 $\pm$ 0.003} & 1.356 $\pm$ 0.012 & 2.035 $\pm$ 0.015 \\
CVAE-FM (decode) & 0.878 $\pm$ 0.016 & \thirdbest{0.733 $\pm$ 0.008} & 0.750 $\pm$ 0.017 & 0.919 $\pm$ 0.004 & 0.214 $\pm$ 0.006 & 1.383 $\pm$ 0.019 & 2.016 $\pm$ 0.028 \\
CVAE-FM (edit) & 0.895 $\pm$ 0.004 & 0.752 $\pm$ 0.003 & 0.754 $\pm$ 0.009 & 0.937 $\pm$ 0.000 & 0.216 $\pm$ 0.001 & 1.385 $\pm$ 0.003 & 2.076 $\pm$ 0.011 \\
\bottomrule
\end{tabular}
\end{adjustbox}
\end{table}

\begin{figure}[h!]
    \centering
    \includegraphics[width=\linewidth]{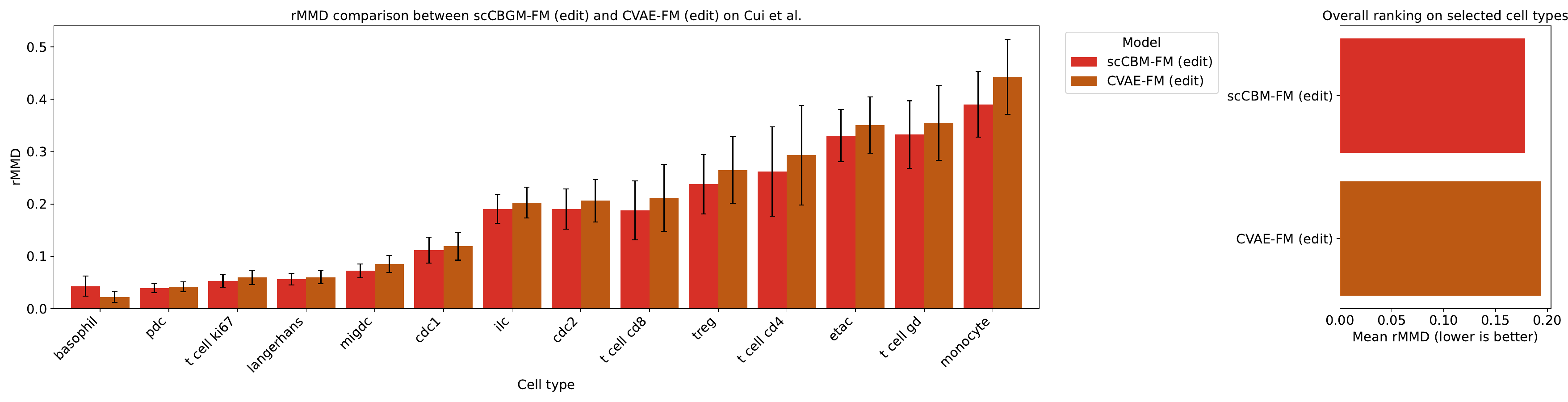}
    \caption{Comparison of \method-FM (edit) and CVAE-FM (edit) on the full \citet{cui2024dictionary} dataset. (Left) rMMD for both methods averaged per cell type over the different cytokine perturbations. (Right) Average rMMD on all cell-type-cytokine pairs.}
    \label{fig:complete_cui}
\end{figure}

\subsubsection{\texorpdfstring{Complete \citet{nault2023single} results}{Complete Nault et al. (2023) results}}
\label{app:liver_results}

In Tables~\ref{tab:complete_mmd_stats_liver},~\ref{tab:liver_frechet}, and~\ref{tab:liver_sinkhorn} we report the results of our experiment on the~\citet{nault2023single} dataset on all available cell types for the rMMD, rFID, and rSD metrics respectively.
\begin{table}[htbp!]
\centering
\caption{rMMD (Mean $\pm$ SEM) per model and target group (\bestres{\textbf{best}}, \secondbest{second-best}, and \thirdbest{third-best} bolded) in~\citet{nault2023single} dataset}
\begin{adjustbox}{width=\textwidth}
\begin{tabular}{lccccccc}
\toprule
\textbf{Model} & \textbf{B cells} & \textbf{T cells} & \makecell{\textbf{Hepatocytes} \\ \textbf{Centrilobular}} & \textbf{Endothelial cells} & \textbf{Stellate cell} & \textbf{Macrophages} & \makecell{\textbf{Hepatocytes} \\ \textbf{(Periportal)}} \\
\midrule
\method & \bestres{\textbf{0.9309 $\pm$ 0.3402}} & 0.8832 $\pm$ 0.3996 & 0.6236 $\pm$ 0.0114 & 1.1352 $\pm$ 0.3946 & 0.8952 $\pm$ 0.5483 & \secondbest{0.6171 $\pm$ 0.4101} & \thirdbest{0.7521 $\pm$ 0.0457} \\
\method -FM (decode) & \thirdbest{0.9497 $\pm$ 0.2785} & \bestres{\textbf{0.8368 $\pm$ 0.3466}} & \secondbest{0.6084 $\pm$ 0.0286} & \thirdbest{1.0465 $\pm$ 0.3216} & \thirdbest{0.8609 $\pm$ 0.4399} & \bestres{\textbf{0.5960 $\pm$ 0.3432}} & \secondbest{0.7186 $\pm$ 0.0564} \\
\method -FM (edit) & 0.9499 $\pm$ 0.3040 & \secondbest{0.8507 $\pm$ 0.3719} & \thirdbest{0.6172 $\pm$ 0.0012} & 1.0589 $\pm$ 0.3438 & \bestres{\textbf{0.8437 $\pm$ 0.4515}} & \thirdbest{0.6286 $\pm$ 0.4083} & \bestres{\textbf{0.7079 $\pm$ 0.0565}} \\
\midrule
Vanilla-FM (decode) & 3.2153 $\pm$ 1.6126 & 3.7338 $\pm$ 2.0041 & 1.4582 $\pm$ 0.6899 & 8.1729 $\pm$ 2.6690 & 10.4811 $\pm$ 7.0579 & 2.3361 $\pm$ 1.3756 & 1.1853 $\pm$ 0.4838 \\
Vanilla-FM (edit) & 1.2832 $\pm$ 0.4411 & 1.4381 $\pm$ 0.2921 & \bestres{\textbf{0.4424 $\pm$ 0.0877}} & 6.8274 $\pm$ 0.9141 & 7.2351 $\pm$ 3.2392 & 1.2286 $\pm$ 0.4893 & 1.0301 $\pm$ 0.3172 \\
\midrule
Biolord & 22.2091 $\pm$ 9.1562 & 27.6154 $\pm$ 12.7611 & / & 41.3361 $\pm$ 17.3826 & 44.3398 $\pm$ 32.5313 & 21.1511 $\pm$ 10.7040 & 4.7065 $\pm$ 2.0315 \\
Cinema-OT & 23.1503 $\pm$ 11.3323 & 29.3735 $\pm$ 15.0724 & 4.6668 $\pm$ 1.5976 & 40.3636 $\pm$ 20.2509 & 45.5700 $\pm$ 33.2892 & 11.7172 $\pm$ 4.8449 & 5.2949 $\pm$ 1.3632 \\
scGen & 11.5939 $\pm$ 5.2396 & 13.4637 $\pm$ 6.5632 & 2.2144 $\pm$ 0.6419 & 16.7187 $\pm$ 7.2767 & 12.0668 $\pm$ 9.0206 & 6.1012 $\pm$ 3.3398 & 2.3892 $\pm$ 0.7981 \\
CVAE & 1.263 $\pm$ 0.072 & 1.328 $\pm$ 0.060 & 1.842 $\pm$ 0.809 & 1.625 $\pm$ 0.121 & 1.349 $\pm$ 0.281 & 1.104 $\pm$ 0.229 & 1.554 $\pm$ 0.433 \\
CVAE-FM (decode) & 1.005 $\pm$ 0.102 & \thirdbest{0.870 $\pm$ 0.091} & 1.333 $\pm$ 0.585 & \secondbest{1.018 $\pm$ 0.110} & 0.882 $\pm$ 0.117 & 0.852 $\pm$ 0.015 & 1.143 $\pm$ 0.263 \\
CVAE-FM (edit) & \secondbest{0.938 $\pm$ 0.101} & 0.887 $\pm$ 0.161 & 1.337 $\pm$ 0.569 & \bestres{\textbf{0.963 $\pm$ 0.102}} & \secondbest{0.853 $\pm$ 0.171} & 0.909 $\pm$ 0.122 & 1.131 $\pm$ 0.253 \\
\bottomrule
\end{tabular}
\end{adjustbox}

\label{tab:complete_mmd_stats_liver}
\end{table}
\begin{table}[htbp!]
\caption{rFID (Mean ± SEM) per cell group for different models (\bestres{\textbf{best}}, \secondbest{second-best}, and \thirdbest{third-best} bolded) in the~\citet{nault2023single} dataset.}
\label{tab:liver_frechet}
\begin{adjustbox}{width=\textwidth}
\begin{tabular}{lccccccc}
\toprule
\textbf{Model} & \textbf{B cells} & \textbf{T cells} & \makecell{\textbf{Hepatocytes} \\ \textbf{(Centrilobular)}} & \textbf{Endothelial cells} & \textbf{Stellate cell} & \textbf{Macrophages} & \makecell{\textbf{Hepatocytes} \\ \textbf{(Periportal)}} \\
\midrule
scCBGM & 1.028 $\pm$ 0.158 & 0.964 $\pm$ 0.241 & 0.755 $\pm$ 0.034 & \thirdbest{1.121 $\pm$ 0.255} & 1.042 $\pm$ 0.353 & \thirdbest{0.993 $\pm$ 0.756} & \thirdbest{0.842 $\pm$ 0.050} \\
scCBGM-FM (decode) & \thirdbest{1.015 $\pm$ 0.140} & \secondbest{0.918 $\pm$ 0.210} & \secondbest{0.717 $\pm$ 0.007} & 1.132 $\pm$ 0.230 & \thirdbest{0.964 $\pm$ 0.324} & \bestres{\textbf{0.911 $\pm$ 0.674}} & \bestres{\textbf{0.807 $\pm$ 0.026}} \\
scCBGM-FM (edit) & 1.022 $\pm$ 0.151 & 0.942 $\pm$ 0.224 & \thirdbest{0.727 $\pm$ 0.006} & 1.156 $\pm$ 0.250 & \thirdbest{0.964 $\pm$ 0.332} & \secondbest{0.957 $\pm$ 0.723} & \secondbest{0.810 $\pm$ 0.036} \\\midrule
Vanilla-FM (decode) & 3.425 $\pm$ 0.001 & 3.022 $\pm$ 0.147 & 1.552 $\pm$ 0.793 & 21.293 $\pm$ 6.744 & 21.950 $\pm$ 7.596 & 3.885 $\pm$ 1.869 & 1.232 $\pm$ 0.482 \\
Vanilla-FM (edit) & 1.499 $\pm$ 0.185 & 1.333 $\pm$ 0.198 & \bestres{\textbf{0.436 $\pm$ 0.151}} & 17.998 $\pm$ 3.897 & 14.030 $\pm$ 2.147 & 1.766 $\pm$ 0.489 & 1.028 $\pm$ 0.283 \\\midrule
biolord & 5.103 $\pm$ 0.560 & 5.510 $\pm$ 0.702 & / & 26.211 $\pm$ 9.435 & 23.090 $\pm$ 9.283 & 7.530 $\pm$ 4.283 & 4.353 $\pm$ 1.775 \\
Cinema-OT & 5.414 $\pm$ 0.918 & 5.841 $\pm$ 0.928 & 4.035 $\pm$ 1.900 & 18.180 $\pm$ 8.224 & 24.920 $\pm$ 9.252 & 4.808 $\pm$ 2.321 & 4.335 $\pm$ 1.276 \\
scGen & 6.018 $\pm$ 0.665 & 6.498 $\pm$ 1.000 & 3.246 $\pm$ 1.460 & 9.912 $\pm$ 3.031 & 6.016 $\pm$ 2.582 & 6.495 $\pm$ 5.042 & 3.457 $\pm$ 1.269 \\
CVAE & 1.064 $\pm$ 0.020 & 1.069 $\pm$ 0.020 & 1.778 $\pm$ 0.314 & 1.276 $\pm$ 0.134 & 1.126 $\pm$ 0.126 & 1.215 $\pm$ 0.303 & 1.623 $\pm$ 0.308 \\
CVAE-FM (decode) & \secondbest{0.985 $\pm$ 0.087} & \bestres{\textbf{0.909 $\pm$ 0.108}} & 1.235 $\pm$ 0.417 & \secondbest{1.014 $\pm$ 0.123} & \secondbest{0.924 $\pm$ 0.106} & 1.032 $\pm$ 0.294 & 1.190 $\pm$ 0.324 \\
CVAE-FM (edit) & \bestres{\textbf{0.961 $\pm$ 0.070}} & \thirdbest{0.926 $\pm$ 0.113} & 1.250 $\pm$ 0.421 & \bestres{\textbf{0.974 $\pm$ 0.082}} & \bestres{\textbf{0.898 $\pm$ 0.136}} & 1.056 $\pm$ 0.326 & 1.192 $\pm$ 0.321 \\
\bottomrule
\end{tabular}
\end{adjustbox}
\end{table}

\begin{table}[htbp!]
\caption{rSD  (Mean ± SEM) per cell group for different models (\bestres{\textbf{best}}, \secondbest{second-best}, and \thirdbest{third-best} bolded) in the~\citet{nault2023single} dataset.}
\label{tab:liver_sinkhorn}
\begin{adjustbox}{width=\textwidth}
\begin{tabular}{lccccccc}
\toprule
\textbf{Model} & \textbf{B cells} & \textbf{T cells} & \makecell{\textbf{Hepatocytes} \\ \textbf{(Centrilobular)}} & \textbf{Endothelial cells} & \textbf{Stellate cell} & \textbf{Macrophages} & \makecell{\textbf{Hepatocytes} \\ \textbf{(Periportal)}} \\
\midrule
scCBGM & 0.986 $\pm$ 0.103 & 0.974 $\pm$ 0.182 & 0.705 $\pm$ 0.090 & 1.073 $\pm$ 0.232 & 1.014 $\pm$ 0.358 & \thirdbest{0.961 $\pm$ 0.794} & \thirdbest{0.742 $\pm$ 0.001} \\
scCBGM-FM (decode) & 0.999 $\pm$ 0.089 & \secondbest{0.902 $\pm$ 0.165} & \secondbest{0.683 $\pm$ 0.061} & 1.108 $\pm$ 0.209 & 0.943 $\pm$ 0.326 & \bestres{\textbf{0.927 $\pm$ 0.722}} & \secondbest{0.713 $\pm$ 0.027} \\ 
scCBGM-FM (edit) & 1.021 $\pm$ 0.084 & 0.968 $\pm$ 0.170 & \thirdbest{0.704 $\pm$ 0.089} & 1.154 $\pm$ 0.230 & \secondbest{0.914 $\pm$ 0.342} & \secondbest{0.959 $\pm$ 0.763} & \bestres{\textbf{0.712 $\pm$ 0.016}} \\\midrule
Vanilla-FM (decode) & 3.051 $\pm$ 0.276 & 2.686 $\pm$ 0.012 & 1.394 $\pm$ 0.572 & 20.301 $\pm$ 5.140 & 25.426 $\pm$ 10.854 & 3.173 $\pm$ 1.247 & 1.330 $\pm$ 0.635 \\
Vanilla-FM (edit) & 1.703 $\pm$ 0.123 & 1.290 $\pm$ 0.287 & \bestres{\textbf{0.504 $\pm$ 0.218}} & 17.380 $\pm$ 2.853 & 16.764 $\pm$ 4.256 & 1.416 $\pm$ 0.136 & 1.037 $\pm$ 0.319 \\\midrule
biolord & 2.261 $\pm$ 0.000 & 2.277 $\pm$ 0.137 & / & 20.420 $\pm$ 6.117 & 22.210 $\pm$ 10.198 & 3.418 $\pm$ 1.693 & 1.444 $\pm$ 0.641 \\
Cinema-OT & 2.784 $\pm$ 0.189 & 2.839 $\pm$ 0.320 & 1.414 $\pm$ 0.290 & 13.078 $\pm$ 5.379 & 24.800 $\pm$ 10.770 & 1.440 $\pm$ 0.128 & 1.813 $\pm$ 0.307 \\
scGen & 3.544 $\pm$ 0.109 & 3.778 $\pm$ 0.372 & 0.981 $\pm$ 0.118 & 5.742 $\pm$ 1.091 & 3.615 $\pm$ 1.817 & 3.859 $\pm$ 3.268 & 1.215 $\pm$ 0.312 \\
CVAE & \bestres{\textbf{0.922 $\pm$ 0.028}} & \thirdbest{0.911 $\pm$ 0.015} & 1.398 $\pm$ 0.616 & \thirdbest{1.046 $\pm$ 0.054} & 0.926 $\pm$ 0.082 & 1.102 $\pm$ 0.266 & 1.303 $\pm$ 0.429 \\
CVAE-FM (decode) & \secondbest{0.944 $\pm$ 0.031} & \bestres{\textbf{0.816 $\pm$ 0.121}} & 1.235 $\pm$ 0.449 & \secondbest{1.015 $\pm$ 0.039} & \thirdbest{0.921 $\pm$ 0.104} & 0.989 $\pm$ 0.248 & 1.154 $\pm$ 0.307 \\
CVAE-FM (edit) & \thirdbest{0.942 $\pm$ 0.070} & 0.933 $\pm$ 0.100 & 1.268 $\pm$ 0.435 & \bestres{\textbf{0.990 $\pm$ 0.076}} & \bestres{\textbf{0.892 $\pm$ 0.140}} & 1.032 $\pm$ 0.311 & 1.166 $\pm$ 0.294 \\
\bottomrule
\end{tabular}
\end{adjustbox}
\end{table}

\clearpage
\subsection{Case study: controlled single-cell editing for enhanced drug response}
Figures  ~\ref{fig:pathway_marker_genes} \& ~\ref{fig:umap_gene_trends} show the cell population after editing appeared similar to the populations which responded to the treatment, in both rMMD and gene-expression changes

\begin{figure} [htpb!]
    \centering
    \includegraphics[width=0.8\linewidth]{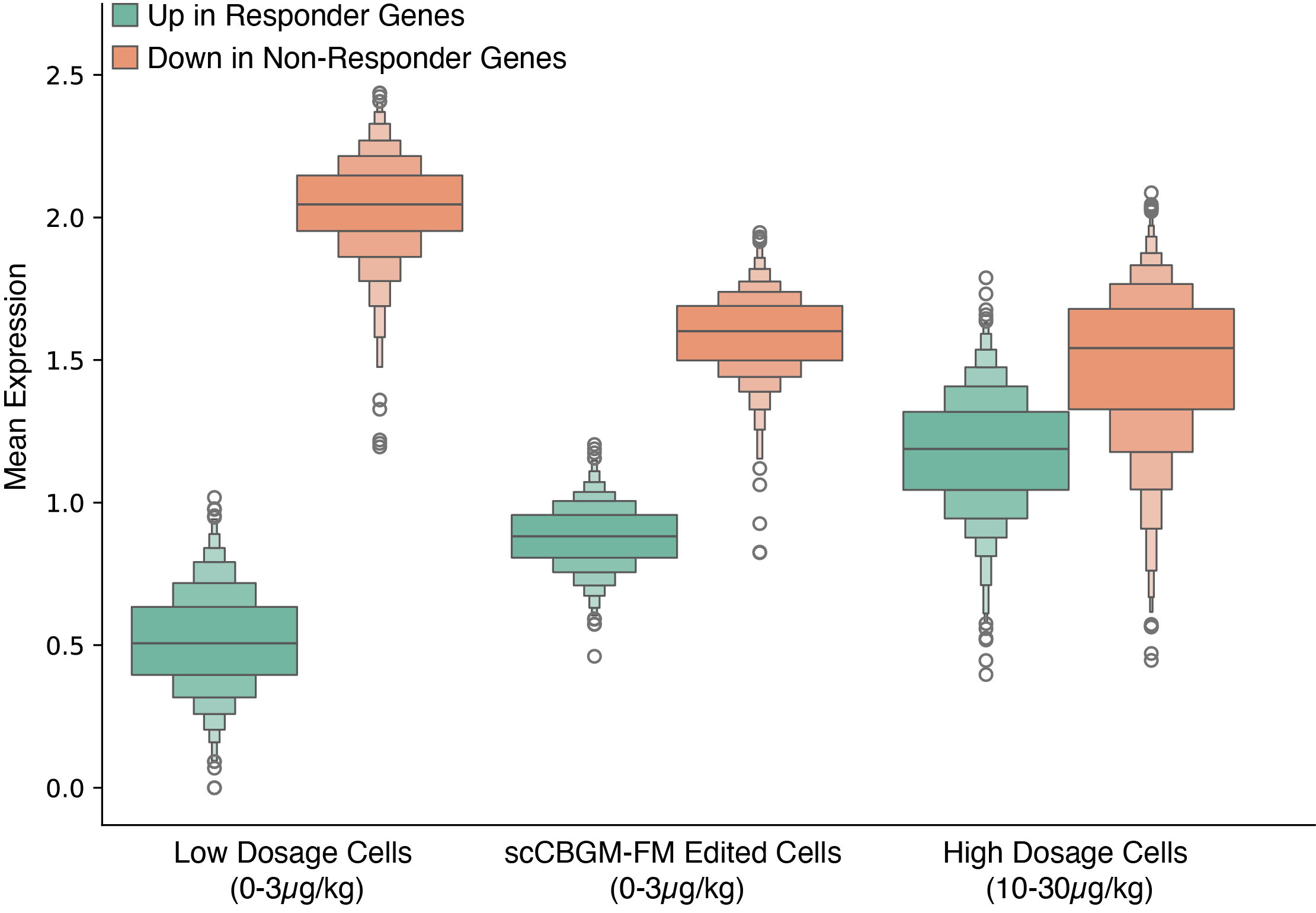}
    \caption{\textbf{Gene expression trends under in-silico perturbation in the~\citet{nault2023single} dataset.} 
    The distribution of mean expression values (averaged across cells in each group) for the top 100 upregulated (green) and downregulated (orange) marker genes. The edited cells (center) successfully reproduce the target gene signatures, shifting the expression of responder and non-responder genes towards the levels observed in the true high-dosage cells, complementing the global results in Figure \ref{fig:pathway_pertubation}. We note that only 40 of these marker genes overlap with the total set of top 100 genes defining the manipulated pathway concepts (500 total). The successful recovery of the remaining 80\% suggests that the model captures downstream regulatory effects beyond the direct inputs.}
    \label{fig:pathway_marker_genes}
\end{figure}

\begin{figure} [htpb!]
    \centering
    \includegraphics[width=0.8\linewidth]{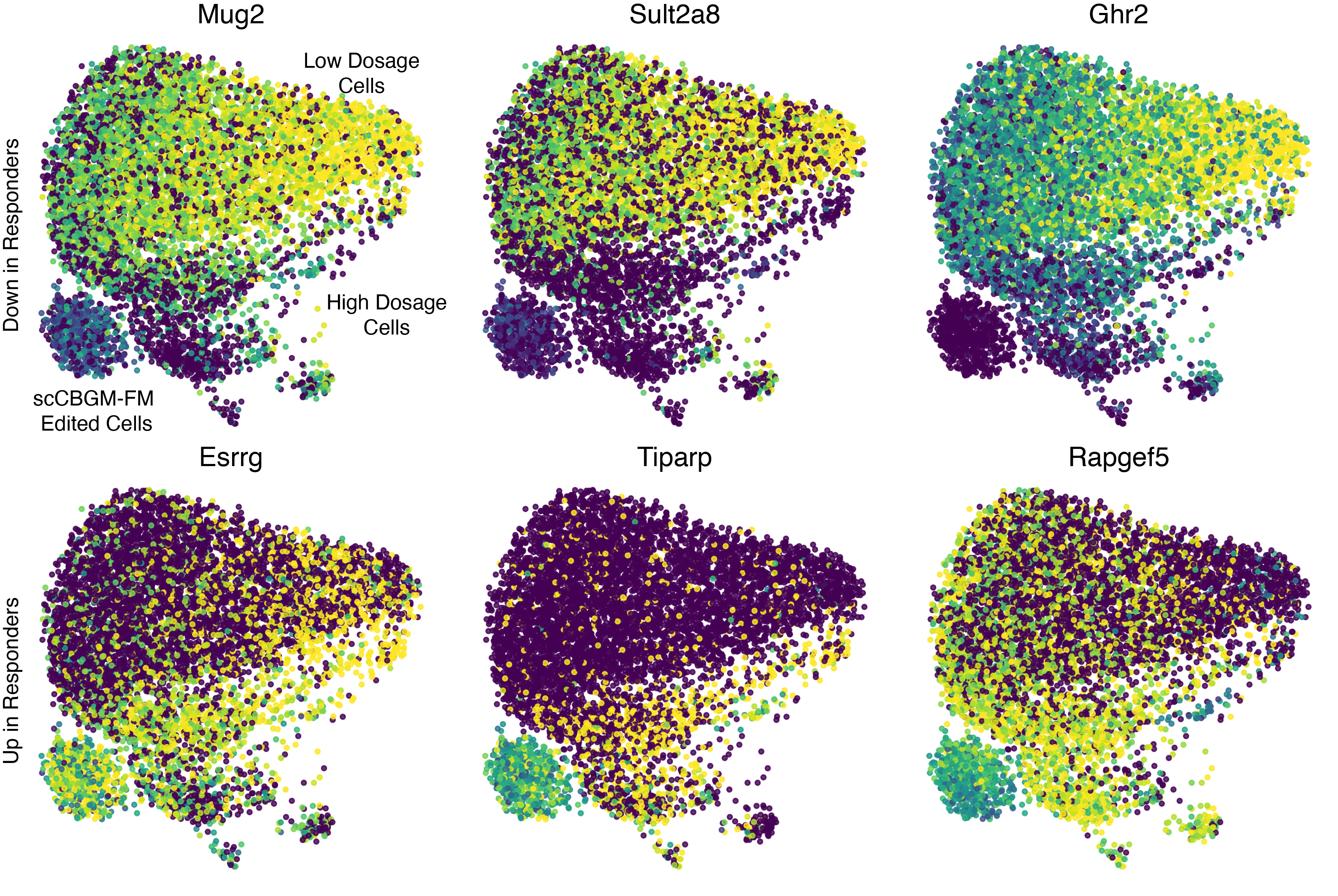}
    \caption{\textbf{Predicted gene expression trends match experimental ground truth.} 
    UMAP visualizations of representative genes from the top 100 differentially expressed set analyzed in Figure~\ref{fig:pathway_marker_genes}. 
    \textbf{Top row:} Genes downregulated in responders exhibit high expression (yellow) in low-dosage cells but are correctly suppressed in the edited population. 
    \textbf{Bottom row:} Genes upregulated in responders are activated in the edited cells, shifting from the low-dosage state (purple) to the high-dosage phenotype (yellow). 
    These patterns confirm that scCBGM-FM drives granular gene-specific shifts consistent with the aggregate trends.}
    \label{fig:umap_gene_trends}
\end{figure}

\subsection{Cell subtype accuracy}

To complement our benchmark, we evaluate whether edited cells from a given model preserve their subtype $\mathbf{s}$. For each edited cell in the~\citet{kang} dataset, we predict its subtype using a kNN classifier with $k=15$ neighbors. In Table~\ref{tab:accuracy_by_subtype}, we report the predicted subtype accuracy (mean $\pm$ SEM) for different models used to edit the cells. \method(-FM) achieves again the best performance over most cell types. As expected, models that achieved an rMMD $\leq 1$ show high accuracy (rMMD $\leq 1$ 1 means that the predicted counterfactual distribution is closer to the target than any other cell-type distribution in the data). However, we note that high subtype accuracy does not necessarily lead to successful counterfactual predictions. Indeed, the predicted distribution may be closest to the correct subtype distribution but not necessarily overlap.

\begin{table}[htbp!
]
\caption{Subtype Prediction Accuracy (Mean ± SEM) per subtype for different models on the~\citet{kang} dataset (\bestres{\textbf{best}}, \secondbest{second-best}, and \thirdbest{third-best} models bolded).}
\label{tab:accuracy_by_subtype}
\begin{adjustbox}{width=\textwidth}
\begin{tabular}{lcc|cc|ccc|cc|cc|cc|c}
\toprule
 & \multicolumn{2}{c}{\textbf{B-cells}} & \multicolumn{2}{c}{\textbf{CD14}} & \multicolumn{3}{c}{\textbf{CD4 T cells}} & \multicolumn{2}{c}{\textbf{CD8 T cells}} & \multicolumn{2}{c}{\textbf{Dendritic}} & \multicolumn{2}{c}{\textbf{FCGR3A+}} & \multicolumn{1}{c}{\textbf{NK cells}} \\
\textbf{Model} & \textbf{Type 0} & \textbf{Type 1} & \textbf{Type 0} & \textbf{Type 1} & \textbf{Type 0} & \textbf{Type 1} & \textbf{Type 2} & \textbf{Type 0} & \textbf{Type 1} & \textbf{Type 0} & \textbf{Type 1} & \textbf{Type 0} & \textbf{Type 1} & \textbf{Type 0} \\
\midrule
scCBGM & \thirdbest{0.996 $\pm$ 0.001} & \thirdbest{0.822 $\pm$ 0.008} & \bestres{\textbf{0.990 $\pm$ 0.000}} & \thirdbest{0.222 $\pm$ 0.045} & \thirdbest{0.983 $\pm$ 0.002} & \secondbest{0.863 $\pm$ 0.007} & \thirdbest{0.734 $\pm$ 0.011} & \bestres{\textbf{0.980 $\pm$ 0.001}} & \thirdbest{0.546 $\pm$ 0.007} & \bestres{\textbf{0.973 $\pm$ 0.003}} & \secondbest{0.710 $\pm$ 0.006} & 0.769 $\pm$ 0.043 & \thirdbest{0.106 $\pm$ 0.009} & \bestres{\textbf{0.992 $\pm$ 0.001}} \\
scCBGM-FM (decode) & \bestres{\textbf{0.996 $\pm$ 0.000}} & 0.817 $\pm$ 0.014 & \thirdbest{0.988 $\pm$ 0.003} & 0.139 $\pm$ 0.028 & \bestres{\textbf{0.985 $\pm$ 0.002}} & 0.837 $\pm$ 0.011 & 0.730 $\pm$ 0.007 & \thirdbest{0.972 $\pm$ 0.004} & \bestres{\textbf{0.552 $\pm$ 0.005}} & \secondbest{0.958 $\pm$ 0.004} & \bestres{\textbf{0.722 $\pm$ 0.006}} & \thirdbest{0.786 $\pm$ 0.037} & \secondbest{0.121 $\pm$ 0.005} & \thirdbest{0.980 $\pm$ 0.002} \\
scCBGM-FM (edit) & \secondbest{0.996 $\pm$ 0.001} & \secondbest{0.833 $\pm$ 0.006} & 0.988 $\pm$ 0.001 & \bestres{\textbf{0.306 $\pm$ 0.053}} & \secondbest{0.983 $\pm$ 0.001} & \bestres{\textbf{0.881 $\pm$ 0.005}} & \bestres{\textbf{0.745 $\pm$ 0.006}} & \secondbest{0.975 $\pm$ 0.001} & \secondbest{0.551 $\pm$ 0.008} & \thirdbest{0.948 $\pm$ 0.004} & \thirdbest{0.682 $\pm$ 0.000} & \secondbest{0.876 $\pm$ 0.018} & 0.093 $\pm$ 0.007 & 0.960 $\pm$ 0.002 \\
\midrule
Vanilla-FM (decode) & 0.338 $\pm$ 0.014 & 0.020 $\pm$ 0.002 & 0.802 $\pm$ 0.014 & 0.028 $\pm$ 0.028 & 0.474 $\pm$ 0.009 & 0.125 $\pm$ 0.007 & 0.008 $\pm$ 0.002 & 0.412 $\pm$ 0.011 & 0.065 $\pm$ 0.005 & 0.033 $\pm$ 0.005 & 0.028 $\pm$ 0.006 & 0.108 $\pm$ 0.004 & 0.011 $\pm$ 0.003 & \secondbest{0.990 $\pm$ 0.001} \\
Vanilla-FM (edit) & 0.425 $\pm$ 0.026 & 0.450 $\pm$ 0.014 & 0.812 $\pm$ 0.015 & \secondbest{0.306 $\pm$ 0.053} & 0.858 $\pm$ 0.005 & 0.468 $\pm$ 0.006 & 0.281 $\pm$ 0.022 & 0.825 $\pm$ 0.012 & 0.362 $\pm$ 0.009 & 0.044 $\pm$ 0.021 & 0.131 $\pm$ 0.025 & 0.164 $\pm$ 0.025 & 0.056 $\pm$ 0.004 & 0.967 $\pm$ 0.004 \\
\midrule
CVAE & 0.988 $\pm$ 0.001 & \bestres{\textbf{0.833 $\pm$ 0.006}} & \secondbest{0.990 $\pm$ 0.001} & 0.111 $\pm$ 0.000 & 0.970 $\pm$ 0.002 & \thirdbest{0.844 $\pm$ 0.003} & \secondbest{0.738 $\pm$ 0.008} & 0.945 $\pm$ 0.006 & 0.444 $\pm$ 0.010 & 0.904 $\pm$ 0.008 & 0.648 $\pm$ 0.007 & \bestres{\textbf{0.983 $\pm$ 0.001}} & \bestres{\textbf{0.142 $\pm$ 0.005}} & 0.922 $\pm$ 0.001 \\
\bottomrule
\end{tabular}
\end{adjustbox}
\end{table}

\newpage

\section{Reproducibility Statement}
We provide code and notebooks to fully reproduce all results, figures, and tables in a public GitHub repository: \url{https://github.com/almaan/conceptlab/tree/mlgenx} All data used in this work are publicly available; for convenience, we also share the processed files through an anonymous OSF repository: \url{https://osf.io/kfqj8/?view_only=02cfaddc86da47d5b8fca0577628ddf7}.

\end{document}